\documentclass[twoside]{article}
\usepackage[accepted]{aistats2023}
\usepackage[round]{natbib}
\usepackage[utf8]{inputenc} %
\usepackage[T1]{fontenc}    %
\usepackage{hyperref}       %
\usepackage{url}            %
\usepackage{booktabs}       %
\usepackage{amsfonts}       %
\usepackage{nicefrac}       %
\usepackage{microtype}      %
\usepackage{xcolor}         %
\usepackage{amsmath,amssymb, amsthm}
\usepackage{dsfont}
\usepackage{mathabx}
\usepackage{soul}
\usepackage{pgfplots}
\pgfplotsset{width=10cm,compat=1.9}
\usepgfplotslibrary{external}
\hypersetup{
	colorlinks=true,
    linkcolor=red,
    citecolor=blue,
    filecolor=black,
    urlcolor=black,
}

\usepackage{amsopn}
\DeclareMathOperator*{\argmin}{argmin}

\newtheorem{theorem}{Theorem}[section]
\newtheorem{lemma}{Lemma}[section]
\newtheorem{definition}{Definition}
\newtheorem{proposition}{Proposition}[section]

\newtheorem{assumption}{Assumption}
\newtheorem{remark}{Remark}[section]

\newcommand{\ZZ}{\mathcal{Z}}
\usepackage{cleveref}
\crefname{algorithm}{Algorithm}{Algorithms}
\crefname{assumption}{Assumption}{Assumptions}
\crefname{equation}{}{}
\crefname{figure}{Fig.}{Figs.}
\crefname{table}{Table}{Tables}
\crefname{section}{Section}{Sections}
\crefname{subsection}{Section}{Sections}
\crefname{theorem}{Theorem}{Theorems}
\crefname{lemma}{Lemma}{Lemmmas}
\crefname{proposition}{Proposition}{Propositions}
\crefname{definition}{Definition}{Definitions}
\crefname{corollary}{Corollary}{Corollaries}
\crefname{remark}{Remark}{Remarks}
\crefname{example}{Example}{Examples}
\crefname{appendix}{Appendix}{Appendices}

\usepackage{hyperref}       %
\usepackage{amsmath}
\usepackage{amsfonts}
\usepackage{bbm}
\usepackage{color}
\usepackage{enumitem}
\usepackage{subfigure}
\usepackage{wrapfig}
\usepackage{svg}
\usepackage{algorithm}
\usepackage{algorithmic}
\newcommand{\XX}{\mathcal{X}}

\newcommand{\WW}{\mathcal{W}}

\newcommand{\ws}{w^{*}}

\newcommand{\prox}{\texttt{\textup{prox}}}

\newcommand{\GG}{\mathcal{G}}

\newcommand{\hf}{\widehat{F}}
\newcommand{\DD}{\mathcal{D}}

\newcommand{\hw}{\widehat{w}}
\newcommand{\expec}{\mathbb{E}}

\newcommand{\EPL}{\mathbb{E}F(\widehat{w}_R) - F(\ws)}

\usepackage{mathtools}

\newcommand{\edit}[1]{{\color{black} #1}}

\newcommand{\epsor}{\epsilon_0^r}
\newcommand{\delor}{\delta_0^r}

\newcommand{\epso}{\epsilon_0}
\newcommand{\delo}{\delta_0}
\newcommand{\rand}{\mathcal{R}^{(i)}}

\newcommand{\bz}{\mathbf{Z}}
\newcommand{\bx}{\mathbf{X}}

\newcommand{\edsim}{\underset{(\epsilon, \delta)}{\simeq}}

\newcommand{\Al}{\mathcal{A}}
\newcommand{\EG}{\expec \| \nabla \hf_{\bx}(\wpr)\|^2}

\newcommand{\wpr}{w_{\text{priv}}}

\newcommand{\wb}{\widebar{w}}
\newcommand{\EGMN}{\mathbb{E}\|\widehat{\mathcal{G}}_{\eta}(\wpr, \bx)\|^2}

\newcommand{\hwrt}{\hw_{r+1}^{t+1}}
\newcommand{\wrt}{w_{r+1}^t}
\newcommand{\wrtp}{w_{r+1}^{t+1}}
\newcommand{\ES}{\mathbb{E}\widehat{F}_{\small \bx \normalsize}(w_S) - \widehat{F}_{\small \bx \normalsize}^*}
\newcommand{\ESR}{\mathbb{E}\widehat{F}_{\small \bx \normalsize}(\wpr) - \widehat{F}_{\small \bx \normalsize}^*}
\newcommand{\EPLL}{\expec F(w_R) - F^*}
\newcommand{\Renyi}{R\'enyi }
\newcommand{\hfx}{\hf_\bx^0}
\newcommand{\one}{\mathds{1}_{\{M < N\}}}
\newcommand{\gapx}{\hat{\Delta}_{\bx}}
\newcommand{\GGh}{\widehat{\GG}_{\eta}}
\newcommand{\pvec}{\mathcal{P}_{\text{vec}}}

\begin{document}

\twocolumn[

\aistatstitle{Private Non-Convex Federated Learning Without a Trusted Server}

\aistatsauthor{Andrew Lowy \And Ali Ghafelebashi \And  Meisam Razaviyayn}

\aistatsaddress{
University of Southern California 
\And
University of Southern California \And
University of Southern California
} ]

\begin{abstract}
\vspace{-0.4cm}
We study federated learning (FL)--especially \textit{cross-silo} FL--with non-convex loss functions  and data from people who do not trust the server or other silos. In this setting, each silo (e.g. hospital) must protect the privacy of each person's data (e.g. patient's medical record), even if the server or other silos act as adversarial eavesdroppers. To that end, we consider \textit{inter-silo record-level} (ISRL) \textit{differential privacy} (DP), which requires silo~$i$'s \textit{communications} to satisfy record/item-level DP. 
We propose novel ISRL-DP algorithms 
for FL with heterogeneous (non-i.i.d.) silo data and two classes of Lipschitz continuous loss functions: First, we consider losses satisfying the  {\it Proximal Polyak-\L ojasiewicz (PL)} inequality, which is an extension of the classical PL condition to the constrained setting. In contrast to our result, prior works only considered  \textit{unconstrained} private  optimization with \textit{Lipschitz} PL loss, which rules out most interesting PL losses such as strongly convex problems and linear/logistic regression. 
Our algorithms nearly attain the optimal 
\textit{strongly convex}, \textit{homogeneous} (i.i.d.)
rate for ISRL-DP FL \textit{without assuming convexity or i.i.d. data}. 
Second, we give the first private algorithms for \textit{non-convex non-smooth} loss functions. Our utility bounds even improve on the state-of-the-art bounds for \textit{smooth} losses.
We complement our upper bounds with lower bounds. 
Additionally, we provide \textit{shuffle DP} (SDP) algorithms that improve over the state-of-the-art \textit{central DP} algorithms 
under more practical trust assumptions.
Numerical experiments show that our 
algorithm
has better
accuracy 
than
baselines for most 
privacy levels.  All the codes are publicly available at: \href{https://github.com/ghafeleb/Private-NonConvex-Federated-Learning-Without-a-Trusted-Server}{https://github.com/ghafeleb/Private-NonConvex-Federated-Learning-Without-a-Trusted-Server}.
\end{abstract}

\section{INTRODUCTION}
\label{sec: intro}
Federated learning (FL) is a machine learning paradigm in which many ``silos'' (a.k.a. ``clients''), such as  hospitals, banks, or schools, collaborate to train a model by exchanging local updates, while storing their training data locally \citep{kairouz2019advances}. Privacy has been an important motivation for FL due to decentralized data storage \citep{mcmahan2017originalFL}. However, silo data can still be leaked in FL without additional safeguards (e.g. via membership or model inversion attacks) 
\citep{inversionfred, inversionhe, inversionsong, zhu2020deep}. Such leaks can occur when silos send updates to the central server---which an adversarial eavesdropper may access---or (in peer-to-peer FL) directly to other silos.

\textit{Differential privacy} (DP)~\citep{dwork2006calibrating} ensures that data cannot be leaked to an adversarial eavesdropper. Several variations of DP have been considered for FL. 
Numerous works~\citep{jayaraman2018distributed, truex2019hybrid, wang2019efficient, kang2021weighted, noble2022differentially} 
studied FL with \textit{central DP} (CDP).\footnote{Central differential privacy (CDP) is often simply referred to as differential privacy (DP) \citep{dwork2014}, but we use CDP here for emphasis. This notion should not be confused with concentrated DP \citep{dwork2016concentrated, bun16}, which is sometimes also abbreviated as ``CDP'' in other works.} Central DP provides protection for silos' \textit{aggregated} data against an adversary who only sees the \textit{final trained model}. Central DP FL has two drawbacks: 1) the aggregate guarantee does not protect the privacy of \textit{each individual silo}'s local data; and 2) it does not defend against privacy attacks from other silos or against an adversary with access to the server during training.

User-level DP (a.k.a. client-level DP) has been proposed as an alternative to central DP~\citep{mcmahan17, geyer17, jayaraman2018distributed, gade2018privacy, wei2020user, zhou2020differentially, levy2021learning, ghazi2021user}. User-level DP remedies the first drawback of CDP by preserving the privacy of every silo's \textit{full local data set}. Such a privacy guarantee is useful for \textit{cross-device FL}, where each silo/client is associated with data from a single person (e.g. cell phone user) possessing many records (e.g. text messages). However, it is ill-suited for \textit{cross-silo FL}, where silos (e.g. hospitals, banks, or schools) typically have data from many different people (e.g. patients, customers, or students). 
In cross-silo FL, each person's (health, financial, or academic) record (or ``item'') may contain sensitive information. Thus, it is desirable to ensure DP for \textit{each individual record} (``item-level DP'') of silo $i$, instead of silo $i$'s full data set. 
Another crucial shortcoming of user-level DP is that, like central DP, it only guarantees the privacy of the \textit{final output} of the FL algorithm against \textit{external} adversaries: it \textit{does not protect against an adversary with access to the server, other silos, or the communications among silos} during training.

While central DP and user-level DP implicitly assume that people (e.g. patients) \textit{trust all parties} (e.g. their own hospital, other hospitals, and the server) with their private data, \textit{local DP} (LDP)~\citep{whatcanwelearnprivately, duchi13} makes an extremely different assumption. In the LDP model, each person (e.g. patient) who contributes data does not trust \textit{anyone}: not even their own silo (e.g. hospital) is considered trustworthy. In cross-silo FL, this assumption is unrealistic: e.g., patients typically want to share their accurate medical test results with \textit{their own} doctor/hospital to get the best care possible. Therefore, LDP is often unnecessary and \textit{may be too stringent to learn useful/accurate models}. 

\begin{figure}[ht]
 \vspace{-.1in}
 \centering
 \includegraphics[width=.5\textwidth]
{ 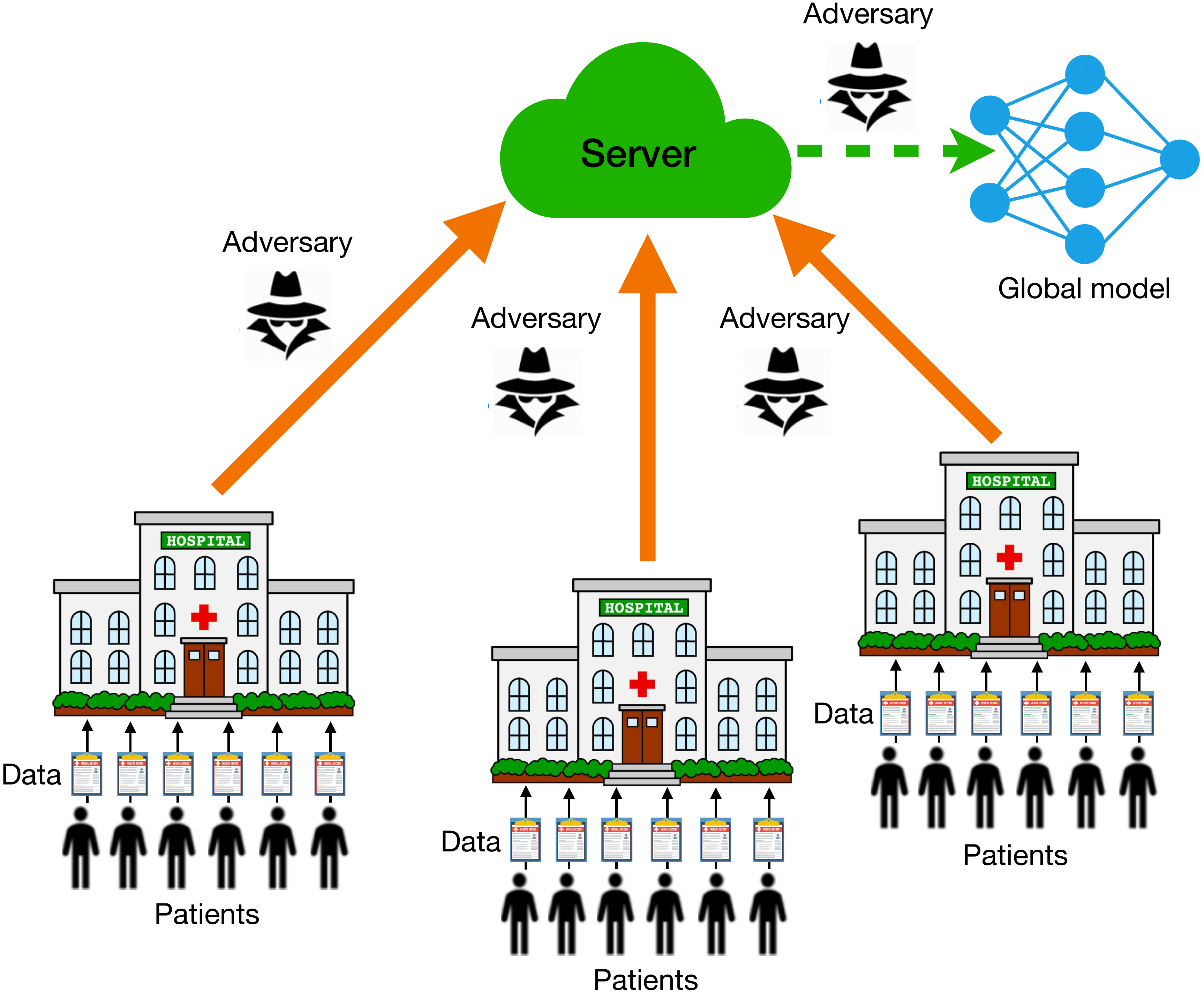}
 \vspace{-.3in}
\caption{\footnotesize 
\footnotesize
ISRL-DP ensures that \textit{each patient's data cannot be leaked, even if the server/other silos collude to decode the data of hospital~$i$.}
In contrast, user-level DP protects the full data of hospital $i$ and leaves hospitals vulnerable to attacks on the server.
}\label{fig:LDP def}
\vspace{-.1in}
\end{figure}

In this work, we consider an intermediate privacy notion between the two extremes of local DP and central/user-level DP: \textit{inter-silo record-level differential privacy} (ISRL-DP). ISRL-DP realistically assumes that \textit{people trust their own silo, but do not trust the server or other silos}. An algorithm is ISRL-DP if all of the \textit{communications} that silo $i$ sends satisfy \textit{item-level DP} (for all $i$). \edit{See~Figure~\ref{fig:LDP def} for a pictorial description and~\cref{sec: prelims} for the precise definition.} ISRL-DP eradicates all the drawbacks of central/user-level DP and local DP discussed above: 1) The item-level DP guarantee for each silo ensures that \textit{no person's data (e.g. medical record) can be leaked}. 2) Privacy of each silo's communications \textit{protects silo data against attacks from an adversarial server and/or other silos}. By post-processing~\citep{dwork2014}, it also implies that the final trained model is private. 3) By relaxing the overly strict trust assumptions of local DP, ISRL-DP allows for \textit{better model accuracy}. ISRL-DP has been considered (under different names) in~\citep{truex20, huang19, huang2020differentially, wu2019value, wei19, dobbe2020, zhao2020local, arachchige2019local, seif20, lr21fl, noble2022differentially, virginia}.

Although ISRL-DP was largely motivated by cross-silo applications, it can also be useful in \textit{cross-device} FL without a trusted server. This is because \textit{ISRL-DP implies user-level DP} if the ISRL-DP parameter is small enough: see~\cref{app: dp relationships} and also~\citep{lr21fl}. However, unlike user-level DP, ISRL-DP has the benefit of preventing leaks to the untrusted server and other users. 

Another intermediate DP notion between the low-trust local models and the high-trust central/user-level models is the \textit{shuffle model} of DP~\citep{prochlo, cheu2019distributed, esarevisited, efm20, fmt20, flame, girgis21a, ghazishuffle}. In the shuffle model, silos send their local updates to a secure \textit{shuffler}. The shuffler randomly permutes silos' updates (anonymizing them), and then sends the shuffled messages to the server. $\Al$ is shuffle DP (SDP) if the shuffled messages satisfy central DP. \edit{Figure~\ref{fig: privacy table} compares the trust assumptions of the different notions of DP FL discussed above.  }

\begin{figure}[ht]
 \vspace{-.1in}
 \centering
 \includegraphics[width=.5\textwidth]
{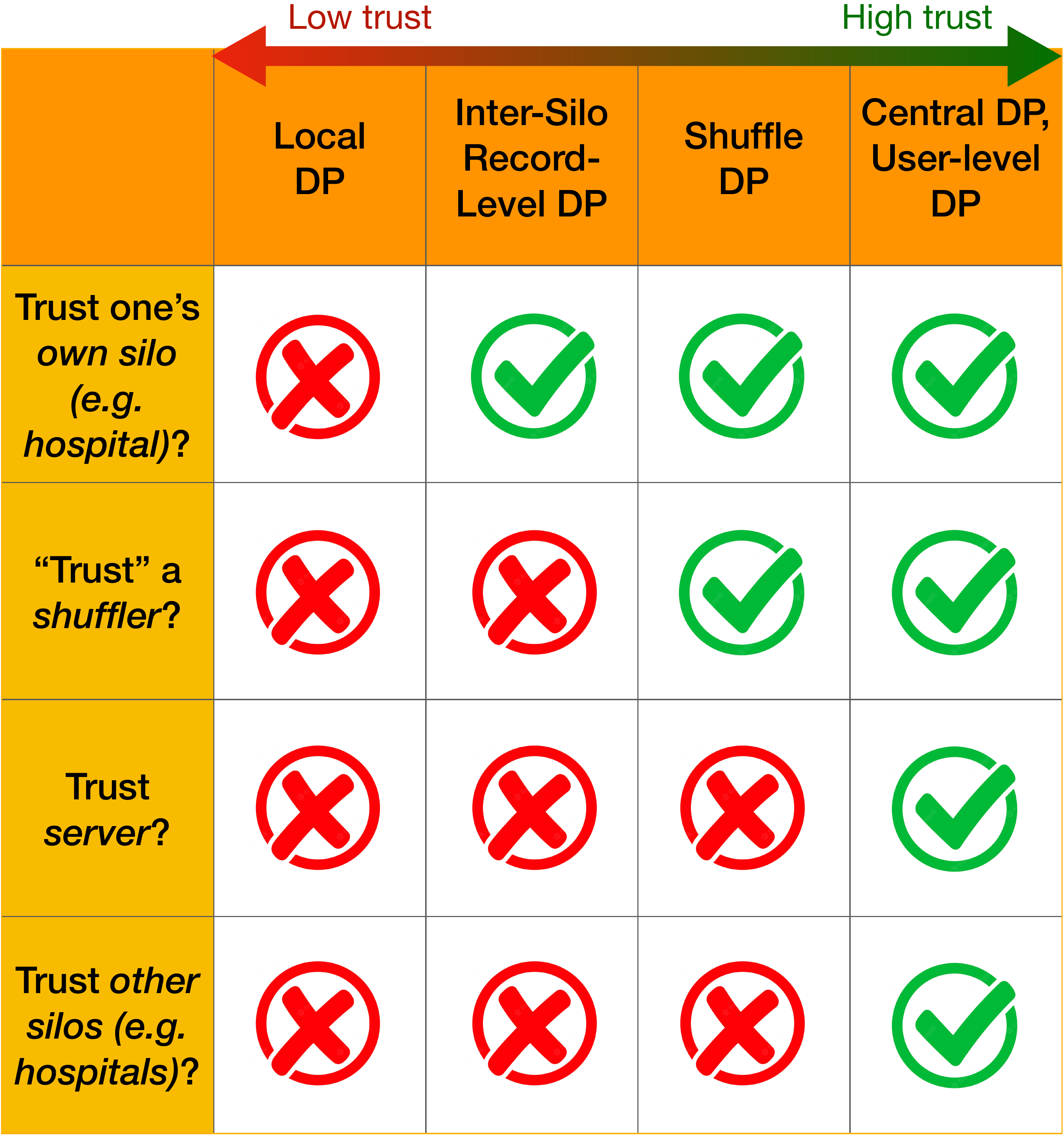}
 \vspace{-.3in}
\caption{\footnotesize 
Trust assumptions of DP FL notions. ``Trust'' is in quotes because silo messages must already satisfy (at least a weak level of) ISRL-DP in order to realize SDP: anonymization alone cannot provide DP~\citep{dwork2014}.}
\label{fig: privacy table}
\vspace{-.1in}
\end{figure}

\textbf{Problem setup:} Consider a \edit{horizontal} FL setting with $N$ silos (e.g. hospitals). Each silo has a local data set with $n$ samples (e.g. patient records): $X_i = (x_{i, 1}, \cdots , x_{i, n}) \in \XX^n$ for $i \in [N] \triangleq \{1, \cdots, N\}$. Let $X_i \sim \DD_i^n$, for unknown distributions $\DD_i$, which may vary across silos (``\textit{heterogeneous}''). 
In the $r$-th round of communication, silos receive the global model $w_r$ from the server and use their local data to improve the model. Then, silos send local updates to the server (or other silos, in peer-to-peer FL), who updates the global model to $w_{r+1}$.
 Given a loss function $f: \mathbb{R}^d \times \XX \to \mathbb{R} \bigcup \{
+\infty\}$, %
let
\begin{equation}
\label{eq: SO}
F_i(w):= \mathbb{E}_{x_i \sim \mathcal{D}_i}[f(w, x_i)].
\vspace{-0.12cm}
\end{equation} 
\normalsize
At times, we consider empirical risk minimization (ERM), with $\widehat{F}_i(w) :=  \frac{1}{n}\sum_{j=1}^{n} f(w,x_{i,j})$.
We aim to 
solve the FL problem: \small
\small
\begin{equation}
\small
\label{eq: FL}
\min_{w \in \mathbb{R}^d} \left\{F(w):= \frac{1}{N}\sum_{i=1}^N F_i(w)\right\},
\end{equation}
\normalsize
\noindent or~$\small \min_{w \in \mathbb{R}^d} \{\widehat{F}_{\bx}(w):= \frac{1}{N}\sum_{i=1}^N \widehat{F}_i(w)\} \normalsize$ for ERM, while keeping silo data private. Here $\bx = (X_1, \cdots, X_N) \in \XX_1^n \times \cdots \times \XX_N^n =: \mathbb{X}$ is a distributed database. 
We allow for constrained FL by considering $f$ that takes the value $+\infty$ outside of some closed set $\WW \subset \mathbb{R}^d$. When $F_i$ takes the form~\cref{eq: SO} (not necessarily ERM), we  refer to the problem as \textit{stochastic optimization} (SO) for emphasis. For SO, we assume that the samples $x_{i,j}$ are independent. For ERM, we make no assumptions on the data. 
The \textit{excess risk} of an algorithm $\Al$ for solving~\cref{eq: FL} is $\expec F(\Al(\mathbf{X})) - F^*$, where $F^* = \inf_{w} F(w)$ and the expectation is taken over both the random draw of $\mathbf{X} = (X_1, \ldots, X_N)$ and the randomness of $\Al$. For ERM, the \textit{excess empirical risk} of $\Al$ is $\expec \hf_\bx(\Al(\bx)) - \hf_\bx^*$, where the expectation is taken solely over the randomness of $\Al$. For general non-convex loss functions, meaningful excess risk guarantees are not tractable in polynomial time. Thus, we use the norm of the gradient to measure the utility (stationarity) of FL algorithms.\footnote{In the non-smooth case, we instead use the norm of the \textit{proximal gradient mapping}, defined in~\cref{sec: Prox-SVRG}. }   

\textbf{Contributions and Related Work:} 
For \textit{(strongly) convex} losses, the optimal performance of ISRL-DP and SDP FL algorithms is mostly understood~\citep{lr21fl, girgis21a}. In this work, we consider the following questions for \textit{non-convex} losses:
\vspace{-.05in}
\begin{center}
\noindent\fbox{
    \parbox{0.4\textwidth}{
    \vspace{-0.cm}
\textbf{Question 1.} What is the best performance that any inter-silo record-level DP algorithm can achieve for solving~\cref{eq: FL} with \textit{non-convex} $F$?

\vspace{0.2cm}

\textbf{Question 2.} With a trusted shuffler (but no trusted server), what performance is attainable?

    }
    }
    \end{center}
\vspace{0.1cm}

Our \underline{first contribution} in \cref{sec: PL SO} is a nearly complete answer to \textbf{Questions 1 and 2} for the subclass of non-convex loss functions that satisfy the {\it Proximal Polyak-\L ojasiewicz (PL)} inequality~\citep{karimi2016linear}. The Proximal PL (PPL) condition is a generalization of the classical PL inequality~\citep{polyak} and covers many important ML models: e.g. some classes of neural nets \edit{such as wide neural nets~\citep{liu2022loss,lei2021sharper}}, linear/logistic regression, LASSO, strongly convex losses~\citep{karimi2016linear}. For \textit{heterogeneous} FL with \textit{non-convex proximal PL} losses, our ISRL-DP algorithm attains excess risk that \textit{nearly matches the strongly convex, i.i.d. lower bound}~\citep{lr21fl}. Additionally, the excess risk of our SDP algorithm nearly matches the \textit{strongly convex, i.i.d., central DP lower bound}~\citep{bft19} and is attained without convexity, without i.i.d. data, and without a trusted server. Our excess risk bounds nearly match the optimal statistical rates in certain practical parameter regimes, resulting in ``privacy almost for free.''

To obtain these results, we invent a new method of analyzing noisy proximal gradient algorithms that does not require convexity, applying tools from the analysis of \textit{objective perturbation}~\citep{chaud, kifer2012private}. Our novel analysis is necessary because privacy noise cannot be easily separated from the non-private optimization terms in the presence of the proximal operator and non-convexity.

Our \underline{second contribution} in \cref{subsec:PL ERM} is a nearly complete answer to \textbf{Questions 1 and 2} for \textit{federated ERM} with proximal PL losses. We provide novel, communication-efficient, proximal \textit{variance-reduced} ISRL-DP and SDP algorithms for non-convex ERM. Our algorithms have near-optimal excess empirical risk that almost match the \textit{strongly convex} ISRL-DP and \textit{CDP} lower bounds~\citep{lr21fl, bst14}, without requiring convexity.

Prior works~\citep{wang2017ermrevisited, kang2021weighted, zhangijcai} on DP PL optimization considered an \textit{extremely narrow} PL function class: \textit{unconstrained} optimization with Lipschitz continuous\footnote{Function $h: \mathbb{R}^d \to \mathbb{R}^m$ is $L$-Lipschitz on $\WW \subset \mathbb{R}^d$ if $\|h(w) - h(w')\| \leq L\|w - w'\|$ for all $w, w' \in \WW$.} losses satisfying the \textit{classical} PL inequality~\citep{polyak}. The combined assumptions of Lipschitz continuity and the PL condition on $\mathbb{R}^d$ (unconstrained) are very strong and rule out most interesting PL losses (e.g. neural nets, linear/logistic regression, LASSO, strongly convex), since the Lipschitz parameter of such losses is infinite or prohibitively large.\footnote{In particular, the DP strongly convex, Lipschitz lower bounds of \citep{bst14, bft19,lr21fl} do not imply lower bounds for the unconstrained Lipschitz, PL function class considered in these works, since their hard instances are not Lipschitz on all of $\mathbb{R}^d$.} By contrast, the \textit{Proximal} PL function class that we consider allows for such losses, which are Lipschitz on a restricted parameter domain. 

\underline{Third}, we address \textbf{Questions 1 and 2} for general \textit{non-convex/non-smooth} (non-PL) loss functions in \cref{sec: Prox-SVRG}. We develop the first DP optimization (in particular, FL) algorithms for non-convex/non-smooth loss functions. Our ISRL-DP and SDP algorithms have significantly better utility than all previous ISRL-DP and \textit{CDP} FL algorithms for \textit{smooth} losses~\citep{wang2019efficient, ding2021, hu21, noble2022differentially}. 
We complement our upper bound with the first non-trivial ISRL-DP \textit{lower bound} for non-convex FL in \cref{sec: lower bound}.

As a consequence of our analyses, we also obtain bounds for FL algorithms that satisfy \textit{both ISRL-DP and user-level DP simultaneously}, in ~\cref{app: hybrid lower bounds}. Such a privacy requirement would be useful in \textit{cross-device} FL with users (e.g. cell phone) who do not trust the server or other users with their sensitive data (e.g. text messages). 

Finally, \underline{numerical experiments} in~\cref{sec: experiments} showcase the practical performance of our algorithm on several benchmark data sets. In each experiment, our algorithm attains better accuracy than the baselines for most privacy levels. 

See~\cref{table: ldp summary} for a summary of our results and~\cref{app: related work} for a thorough discussion of related work. 

\begin{figure}[ht]
 \vspace{-.1in}
 \centering
 \includegraphics[width=.5\textwidth]
{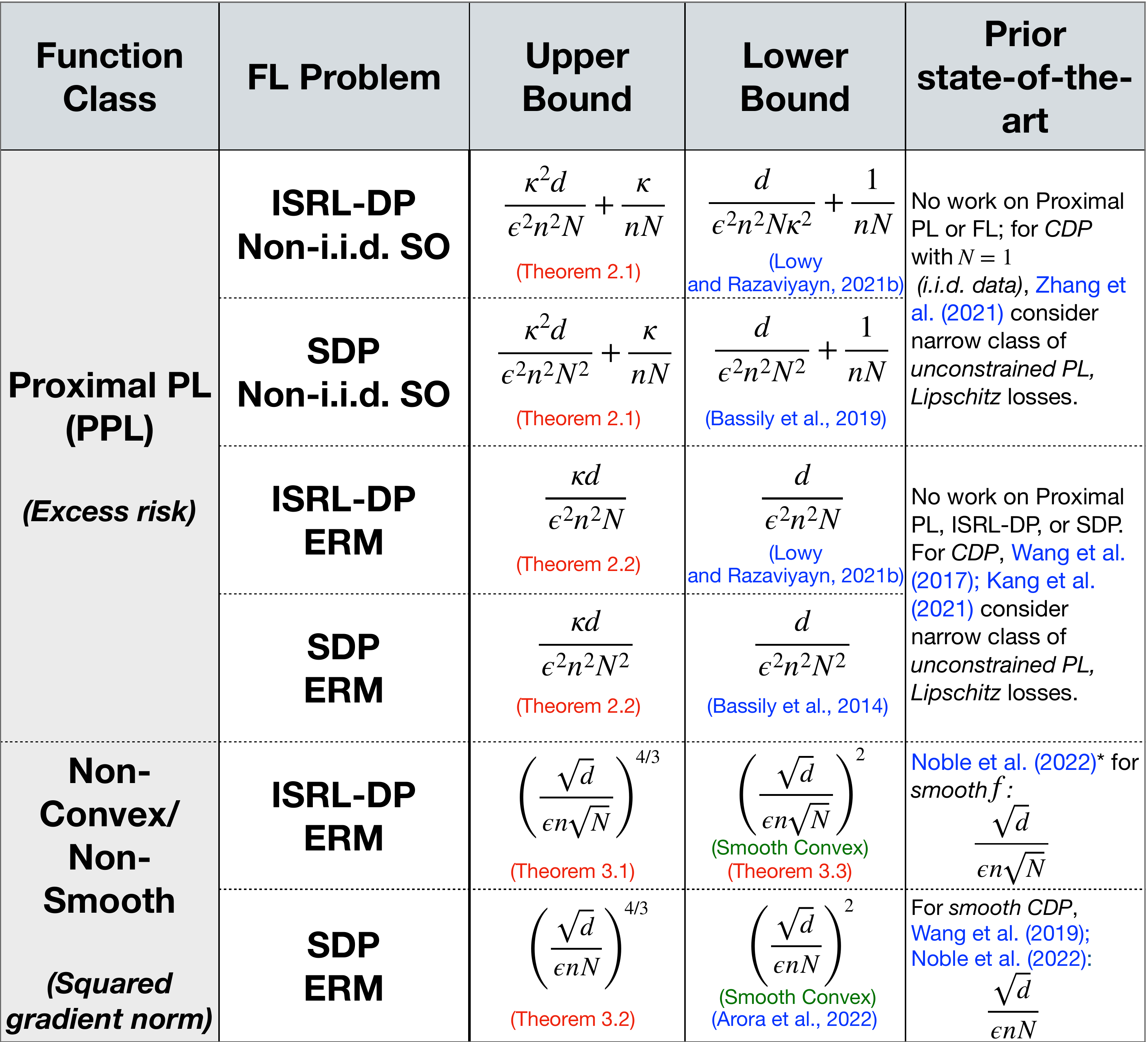}
 \vspace{-.3in}
\caption{\footnotesize 
Summary of results for $M=N$, log terms omitted. $\kappa = \beta/\mu$, where $\beta$ is the smoothness parameter and $\mu$ is the proximal-PL parameter of the loss. *\citep{noble2022differentially} mostly analyzes CDP FL but we observe that a ISRL-DP bound can also be obtained with a small modification of their algorithm and analysis. 
  \normalsize}
\label{table: ldp summary}
\vspace{-.1in}
\end{figure}

\subsection{Preliminaries}
\label{sec: prelims}
\textbf{Differential Privacy:} Let $\mathbb{X} = \XX_1^n \times \cdots \XX^n$ and $\rho:\mathbb{X}^2 \to [0, \infty)$ be a  distance between databases. Two databases $\bx, \bx' \in \mathbb{X}$ are \textit{$\rho$-adjacent} if $\rho(\bx, \bx') \leq 1$. DP ensures that (with high probability) an adversary cannot distinguish between the outputs of algorithm $\Al$ when it is run on adjacent databases: 
\begin{definition}[Differential Privacy]
\label{def: DP}
Let $\epsilon \geq 0, ~\delta \in [0, 1).$ A randomized algorithm $\Al: \mathbb{X} \to \mathcal{W}$ is \textit{$(\epsilon, \delta)$-differentially private} (DP) (with respect to $\rho$) if for all $\rho$-adjacent data sets $\bx, \bx' \in \mathbb{X}$ and all measurable subsets $S \subseteq \WW$, we have 
\begin{equation}
\label{eq: DP}
\mathbb{P}(\Al(\bx) \in S) \leq e^\epsilon \mathbb{P}(\Al(\bx') \in S) + \delta.
\end{equation}
\end{definition}
\normalsize 

\begin{definition}[Inter-Silo Record-Level Differential Privacy]
\label{def: informal ISRL-DP}
Let $\rho_i: \XX^2 \to [0, \infty)$, $\rho_i(X_i, X'_i) := \sum_{j=1}^{n} \mathds{1}_{\{x_{i,j} \neq x_{i,j}'\}}$, $i \in [N]$. 
A randomized algorithm $\Al$ is $(\epsilon, \delta)$-ISRL-DP if for all $i \in [N]$ and all $\rho_i$-adjacent silo data sets $X_i, X'_i$, the full transcript of silo $i$'s sent messages 
satisfies~\cref{eq: DP} 
for any fixed settings of other silos' messages and data.
\end{definition}

\begin{definition}[Shuffle Differential Privacy~\citep{prochlo, cheu2019distributed}]
\label{def: SDP}
A randomized algorithm $\Al$ is \textit{$(\epsilon, \delta)$-shuffle DP} \textit{(SDP)} if for all $\rho$-adjacent databases $\bx, \bx' \in \mathbb{X}
$ 
and all measurable subsets $S$, the collection of all uniformly randomly permuted messages that are sent by the shuffler 
satisfies~\cref{eq: DP}, with $\rho(\bx, \bx') := \sum_{i=1}^N \sum_{j=1}^{n} \mathds{1}_{\{x_{i,j} \neq x'_{i,j}\}}$.
\end{definition}
In~\cref{app: SDP}, we recall the basic DP building blocks
that our algorithms employ. 

\textbf{Notation and Assumptions:}
Denote by $\|\cdot\|$ the $\ell_2$ norm. Let $\WW$ be a closed convex set. For differentiable (w.r.t. $w$) $f^0: \WW \times \XX \to \mathbb{R}$, denote its gradient w.r.t. $w$ by $\nabla f^0(w,x)$. Function 
$h$ is \textit{$\beta$-smooth} if $\nabla h$ is $\beta$-Lipschitz.  A \textit{proper} function has range $\mathbb{R}\bigcup \{+\infty\}$ and
is not identically equal to $+\infty$. Function $g$ is \textit{closed} if~$\forall \alpha \in \mathbb{R}$, the 
set $\{w \in \mbox{dom}(g) \vert g(w) \leq \alpha \}$ is closed. 
The indicator function of $\WW$ is \footnotesize $
\iota_{\WW}(w) := \begin{cases}
0 &\mbox{if $w \in \mathcal{W}$} \\
+\infty &\mbox{otherwise}
\end{cases}.$
\normalsize
The \textit{proximal operator} of function $f^1$ is 
$
\prox_{\eta f^1}(z) := \argmin_{y \in \mathbb{R}^d}\left(\eta f^1(y) + \frac{1}{2}\|y - z\|^2 \right), ~\text{for}~\eta > 0.$
Write $a \lesssim b$ if $\exists C > 0$ such that $a \leq Cb$ and $a = \widetilde{\mathcal{O}}(b)$ if $a \lesssim \log^2(\theta) b$ for some parameters $\theta$. Let $\small \hat{\Delta}_{\bx} := \hf_{\bx}(w_0) - \hf_{\bx}^*$, with $\hf_{\bx}^* = \inf_w \hf_{\bx}(w)$.
We assume the loss function $f(w,x) = f^0(w,x) + f^1(w)$ is \textit{non-convex/non-smooth composite}, where $f^0$ is bounded below, and: 
\begin{assumption}
\label{ass: lip}
$f^0(\cdot, x)$ is $L$-Lipschitz (on $\WW$ if $f^1 = \iota_{\WW} + g$ for some convex $g \geq 0$; on $\mathbb{R}^d$ otherwise) and $\beta$-smooth for all $x \in \XX$. 
\end{assumption}
\begin{assumption}
\label{ass: f1}
$f^1$ is a proper, closed, convex function. 
\end{assumption}

\edit{Examples of functions satisfying~\cref{ass: f1} include indicator functions of convex sets $\iota_{\WW}$ and $\ell_p$-regularizers $\lambda \|w\|_p$ with $p \in [1, \infty]$.} We allow FL networks in which some silos may be unable to participate in every round (e.g. due to internet/wireless communication  problems):  
\begin{assumption}
\label{ass: unreliable}
In each round of communication~$r$, a uniformly random subset~$S_r$ of $M = |S_r| \in [N]$ silos receives the global model and can send messages.\footnote{In the Appendix, we prove general versions of some of our results with $|S_r| = M_r$ for \textit{random} $M_r$.} 
\end{assumption}
\cref{ass: unreliable} is realistic for cross-device FL. However, in cross-silo FL, typically $M \approx N$~\citep{kairouz2019advances}. 

\section{
ALGORITHMS FOR PROXIMAL-PL LOSSES
}
\label{sec:PL}
\vspace{-.07in}
\subsection{Noisy Distributed Proximal SGD for Heterogeneous FL (SO)}
\label{sec: PL SO}
\vspace{-.07in}
We propose a simple distributed \textit{Noisy Proximal SGD (Prox-SGD)} method: in each round $r \in [R]$, available silos $i \in S_r$ draw local samples $\{x_{i,j}^r\}_{j=1}^{K := \lfloor n/R \rfloor}$ from $X_i$ (without replacement) and compute 
$\widetilde{g}_r^{i} := \frac{1}{K} \sum_{j=1}^{K} \nabla f^0(w_r, x_{i,j}^r) + u_i$, where $u_i \sim \mathcal{N}(0, \sigma^2 \mathbf{I}_d)$ for ISRL-DP. For SDP, $u_i$ has Binomial distribution~\citep{cheu2021shuffle}. The server (or silos) aggregates $\widetilde{g}_{r} := \frac{1}{M} \sum_{i \in S_r} \widetilde{g}_r^{i}$ and then updates the global model $w_{r+1} :=\prox_{\frac{1}{2\beta}f^1}(w_r - \frac{1}{2\beta}\widetilde{g}_r)$. The use of \texttt{prox} is needed to handle the potential non-smoothness of $f$. Pseudocodes are in~\cref{app: sdp prox grad SO}. 
\begin{assumption}
The loss is $\mu$-PPL in expectation: $\forall w \in \mathbb{R}^d$,
\label{ass: stochProximal PL}
\vspace{-.02in}
 \small 
\begin{align*}
\small
\mu \expec[\hf_{\cal{S}}(w) - \inf_{w'} \hf_{\cal{S}}(w')] 
&
\leq -\beta \expec\Bigg[\min_{y}\Big[\langle \nabla \hf_{\cal{S}}^0(w), y - w \rangle 
\\
&\;\; 
+ \frac{\beta}{2}\|y - w\|^2 +f^1(y) - f^1(w)\Big]\Bigg],
\vspace{-.02in}
\end{align*}
\normalsize
where $\hf_{\cal{S}}(w) := \frac{1}{MK} \sum_{i \in S} \sum_{j=1}^K f(w, x_{i,j})$, $S \subseteq [N]$ is a uniformly random subset of size $M$, $\mathcal{S} = \{ x_{i,j} \}_{i \in S, j \in [K]}$, and $x_{i,j} \sim \DD_i$. Denote $\kappa = \beta/\mu$.
\end{assumption}
\vspace{-.07in}
As discussed earlier, many interesting losses (e.g. neural nets, linear regression) satisfy~\cref{ass: stochProximal PL}.
\begin{theorem}[Noisy Prox-SGD: Heterogeneous PPL FL]
\label{thm: hetero pl fl proxgrad}
Grant~\cref{ass: stochProximal PL}. Let $\epsilon \leq \min\{8\ln(1/\delta), 15\}, \delta \in (0,1/2)$, $n \geq \widetilde{\Omega}(\kappa)$. Then, there exist $\sigma^2$ and $K$ such that:\\
1. ISRL-DP Prox-SGD is $(\epsilon, \delta)$-ISRL-DP, and in $R = \widetilde{\mathcal{O}}(\kappa)$ communications, we have:
\small
\begin{equation} 
\label{eq: ISRL-DP Prox SGD}
\small
\EPLL = \widetilde{\mathcal{O}}\left(\frac{L^2}{\mu}\left(\frac{\kappa^2 d \ln(1/\delta)}{\epsilon^2 n^2 M} + \frac{\kappa}{Mn}\right)\right).
\end{equation}
\normalsize
2. SDP Prox-SGD is $(\epsilon, \delta)$-SDP for $M \geq N \min(\epsilon/2, 1)$, and if $R = \widetilde{\mathcal{O}}(\kappa)$, then:
\small
\begin{equation}
\label{eq: SDP Prox SGD}
\small
\EPLL = \widetilde{\mathcal{O}}\left(\frac{L^2}{\mu}\left(\frac{\kappa^2 d \ln^2(d/\delta)}{\epsilon^2 n^2 N^2} + \frac{\kappa}{Mn}\right)\right).
\end{equation}
\normalsize
\end{theorem}
\begin{remark}[Near-Optimality and ``privacy almost for free'']
Let $M=N$. Then, the bound in~\cref{eq: SDP Prox SGD} nearly matches the central DP \textit{strongly convex, i.i.d.} lower bound
of \citep{bft19} up to the factor $\widetilde{O}(\kappa^2)$ without a trusted server, without convexity, and without i.i.d. silos. Further, if $\kappa d \log^2(d/\delta)/\epsilon^2 \lesssim nN$, then~\cref{eq: SDP Prox SGD} matches the \textit{non-private} \textit{strongly convex, i.i.d.} lower bound of~\citep{agarwal} up to a $\widetilde{\mathcal{O}}(\kappa)$ factor, providing privacy nearly for free, without convexity/homogeneity. The bound in~\cref{eq: ISRL-DP Prox SGD} is larger than the \textit{i.i.d.}, \textit{strongly convex}, ISRL-DP lower bound of~\citep{lr21fl} by a factor of $\widetilde{\mathcal{O}}(\kappa^4)$.\footnote{In the terminology of~\citep{lr21fl}, Noisy Prox-SGD is $C$-compositional with $C = \sqrt{R} = \widetilde{O}(\sqrt{\kappa})$.} 
Moreover, if $\kappa d \ln(1/\delta)/\epsilon^2 \lesssim n$, then the ISRL-DP rate in~\cref{eq: ISRL-DP Prox SGD} matches the \textit{non-private}, strongly convex, i.i.d. lower bound~\citep{agarwal} up to $\widetilde{\mathcal{O}}(\kappa)$. 
\vspace{-.1in}
\end{remark}
\cref{thm: hetero pl fl proxgrad} is proved in~\cref{app: proof ofProximal PL SO}. Privacy follows from parallel composition~\citep{mcsherry2009privacy} and the guarantees of the Gaussian mechanism~\citep{dwork2014} and binomial-noised shuffle vector summation protocol~\citep{cheu2021shuffle}. The main idea of the excess loss proofs is to view each noisy proximal evaluation as an execution of \textit{objective perturbation}~\citep{chaud}. Using techniques from the analysis of objective perturbation, we bound the key term arising from descent lemma by the corresponding noiseless minimum plus an error term that scales with $\|\frac{1}{M}\sum_{i \in S_r} u_i\|^2$. 

Our novel proof approach is necessary because with the proximal operator and without convexity, privacy noise cannot be easily separated from the non-private terms. By comparison, in the convex case, the proof of \citep[Theorem 2.1]{lr21fl} uses non-expansiveness of projection and independence of the noise and data to bound $\expec\|w_{r+1} - w_r\|^2$, which yields a bound on $\expec F(w_r) - F^*$ by convexity. On the other hand, in the unconstrained PL case considered in~\citep{wang2017ermrevisited, kang2021weighted, zhangijcai}, the excess loss proof is easy, but the result is essentially vacuous since Lipschitzness on $\mathbb{R}^d$ is incompatible with all PL losses that we are aware of. The works mentioned above considered the simpler i.i.d. or ERM cases: Balancing divergent silo distributions and privately reaching consensus on a single parameter $w_R$ that optimizes the global objective $F$ poses an additional difficulty.

\vspace{-.05in}
\subsection{
Noisy Distributed Prox-PL-SVRG for Federated ERM}
\label{subsec:PL ERM}
\vspace{-.05in}
In this subsection, we
assume $\hf_{\bx}$ satisfies the proximal-PL inequality with parameter $\mu > 0$; i.e. for all $w \in \mathbb{R}^d$:
\small
\begin{align*}
\mu[\hf_{\bx}(w) - \inf_{w'} \hf_{\bx}(w')] &\leq - \beta \min_{y}\Big[\langle \nabla \hf_{\bx}^0(w), y - w \rangle
\\
&\;\;\; 
+ \frac{\beta}{2}\|y - w\|^2 + \hf_{\bx}^1(y) - \hf_{\bx}^1(w)\Big].
\end{align*}
\normalsize 
We propose new, variance-reduced accelerated ISRL-DP/SDP algorithms in order to achieve near-optimal rates in fewer communication rounds than would be possible with Noisy Prox-SGD. Our ISRL-DP \cref{alg: LDP ProxPLSVRG} for Proximal PL ERM, which builds on~\citep{proxSVRG}, iteratively runs ISRL-DP Prox-SVRG (\cref{alg: LDP ProxSVRG}) with re-starts. See~\cref{app: SDP proxSVRG} for our SDP algorithm, which is nearly identical to \cref{alg: LDP ProxPLSVRG} except that we use the binomial noise-based shuffle protocol of~\citep{cheu2021shuffle} instead of Gaussian noise.

\begin{algorithm}[ht]
\caption{ISRL-DP FedProx-SVRG $(w_0)$
}
\label{alg: LDP ProxSVRG}
\begin{algorithmic}[1]
\STATE {\bfseries Input:} 
$E \in \mathbb{N}, K \in [n], Q := \lfloor \frac{n}{K} \rfloor, \bx \in \mathbb{X}, \eta > 0, \sigma_1, \sigma_2 \geq 0, \widebar{w}_0 = w_0^Q = w_0 \in \mathbb{R}^d$.
 \FOR{$r \in \{0, 1, \cdots, E-1\}$} 
 \STATE Server updates $w^0_{r+1} = w^Q_r$.
\FOR{$i \in S_r$ \textbf{in parallel}}
\STATE Server sends global model $w_r$ to silo $i$. 
\STATE Silo $i$ draws noise $u^{i}_1 \sim \mathcal{N}(0, \sigma_1^2 \mathbf{I}_d)$ and computes $\widetilde{g}_{r+1}^{i} :=
\nabla \hf_i^0(\widebar{w}_r) + u^{i}_1$.
\FOR{$t \in \{0,1, \cdots Q - 1\}$}
\STATE Silo $i$ draws $K$ samples $x_{i,j}^{r+1, t}$ uniformly 
from $X_i$ with replacement (for $j \in [K]$) and noise $u^{i}_2 \sim \mathcal{N}(0, \sigma_2^2 \mathbf{I}_d)$. 
\STATE Silo $i$ computes $\widetilde{v}_{r+1}^{t, i} = \frac{1}{K} \sum_{j=1}^K[\nabla f^0(w_{r+1}^t, x_{i,j}^{r+1, t}) - \nabla f^0(\wb_r, x_{i,j}^{r+1, t})] + \widetilde{g}_{r+1}^i + u_2^i$.
\STATE Server aggregates $\widetilde{v}_{r+1}^t = 
\frac{1}{M}
\sum_{i \in S_{r+1}} \widetilde{v}_{r+1}^{t,i}$, updates $w_{r+1}^{t+1} := \prox_{\eta f^1}(w_{r+1}^t - \eta \widetilde{v}_{r+1}^t)$.
\ENDFOR 
\STATE Server updates $\wb_{r+1} = w_{r+1}^Q$.
\ENDFOR \\
\ENDFOR \\
\STATE {\bfseries Output:} $\wpr \sim \text{Unif}(\{w_{r+1}^t\}_{r=0, \cdots, E-1; t=0, \cdots Q-1})$.
\end{algorithmic}
\end{algorithm}
\begin{algorithm}[ht]
\caption{ISRL-DP
FedProx-PL-SVRG 
}
\label{alg: LDP ProxPLSVRG}
\begin{algorithmic}[1]
\STATE {\bfseries Input:} 
$E \in \mathbb{N}, K \in [n], Q := \lfloor \frac{n}{K} \rfloor, \bx \in \mathbb{X}, \eta > 0, \sigma_1, \sigma_2 \geq 0, \widebar{w}_0 = w_0^Q = w_0 \in \mathbb{R}^d$.
 \FOR{$s \in [S]$} 
 \STATE $w_s = \texttt{ISRL-DP FedProx-SVRG}(w_{s-1})$
\ENDFOR \\
\STATE {\bfseries Output:} $\wpr := w_S$.
\end{algorithmic}
\end{algorithm}

The key component of ISRL-DP Prox-SVRG is in line 9 of~\cref{alg: LDP ProxSVRG}, where instead of using standard noisy stochastic gradients, silo $i$ %
computes the difference between the stochastic gradient at the current iterate $w_{r+1}^t$ and the iterate $\widebar{w}_r$ from the previous epoch, thereby reducing the variance of the noisy gradient estimator--which is still unbiased--and facilitating faster convergence (i.e. better communication complexity). Notice that the $\ell_2$-sensitivity of the variance-reduced gradient in line 9 is larger than the sensitivity of standard stochastic gradients (e.g. in line 6), so larger $\sigma_2^2 > \sigma_1^2$ is required for ISRL-DP. However, the sensitivity only increases by a constant factor, which does not significantly affect utility. \edit{\cref{alg: LDP ProxPLSVRG} runs~\cref{alg: LDP ProxSVRG} $S$ times with re-starts.} For a suitable choice of algorithmic parameters, we have:
\begin{theorem}[Noisy Prox-PL-SVRG: ERM]
\label{thm: ISRL-DP prox PL SVRG ERM}
Let $\epsilon \leq \min\{15, 2\ln(2/\delta)\}, \delta \in (0,1/2)$, $M=N$, $\bx \in \XX^n$, and $\kappa = \beta/\mu$. Then, in $\widetilde{\mathcal{O}}(\kappa)$ communication rounds, we have: \\
1. ISRL-DP Prox-PL-SVRG is $(\epsilon, \delta)$-ISRL-DP and
\small
~$\small
\ESR = \widetilde{
\mathcal{O}}\left(\kappa \frac{L^2 d \ln(1/\delta)}{\mu \epsilon^2 n^2 N} 
\right).$
\normalsize 
\\
2. SDP Prox-PL-SVRG is $(\epsilon, \delta)$-SDP and 
\small
~$\ESR = \widetilde{\mathcal{O}}\left(\kappa \frac{L^2 d \ln(1/\delta)}{\mu \epsilon^2 n^2 N^2} 
\right).
$
\normalsize
\vspace{-.05in}
\end{theorem}
Expectations are solely over $\Al$. A similar result holds for $M < N$, provided silo data is not too heterogeneous. See~\cref{app: proofs of prox pl svrg ERM} for details and the proof. 
\begin{remark}[Near-Optimality]
The ISRL-DP and SDP bounds in~\cref{thm: ISRL-DP prox PL SVRG ERM} nearly match (respectively) the ISRL-DP and CDP \textit{strongly convex} ERM lower bounds \citep{lr21fl, bst14} (for $f^1 = \iota_{\WW})$) up to the factor $\widetilde{\mathcal{O}}(\kappa)$, and are attained without convexity. 
\end{remark}
\section{
ALGORITHMS FOR NON-CONVEX/NON-SMOOTH COMPOSITE LOSSES
}
\label{sec: Prox-SVRG}
In this section, we consider private FL with general non-convex/non-smooth composite losses: i.e. we make
no additional assumptions on $f$ beyond \cref{ass: lip} and \cref{ass: f1}. In particular, we do not assume the Proximal PL
condition, allowing for a range of constrained/unconstrained non-convex and non-smooth FL problems. For such a function class, finding global optima is not possible in polynomial time; optimization algorithms may only find \textit{stationary points} in polynomial time. Thus, we measure the utility of our algorithms in terms of the norm of the \textit{proximal gradient mapping}: 
\[
\widehat{\mathcal{G}}_{\eta}(w, \bx):= \frac{1}{\eta}[w - \prox_{\eta f^1}(w - \eta \nabla \hf_{\bx}^0(w))]
\]
\normalsize
For proximal algorithms, this is a natural measure of stationarity~\citep{proxSVRG, spiderboost} which generalizes the standard (for smooth unconstrained) notion of first-order stationarity. 
In the smooth unconstrained case, $\|\widehat{\mathcal{G}}_{\eta}(w, \bx)\|$ reduces to $\|\nabla \hf_{\bx}(w)\|$, which is often used to measure convergence in non-convex optimization. 
Building on~\citep{fangspider, spiderboost, arora2022faster}, we propose~\cref{alg: LDP proxspider} for ISRL-DP FL with non-convex/non-smooth composite losses. \cref{alg: LDP proxspider} is inspired by the optimality of non-private SPIDER for non-convex ERM~\citep{carmon2019lower}. 

\begin{algorithm}[ht]
\caption{ISRL-DP FedProx-SPIDER}
\label{alg: LDP proxspider}
\begin{algorithmic}[1]
\STATE {\bfseries Input:} 
$R \in \mathbb{N}, K_1, K_2 \in [n], \bx \in \mathbb{X}, \eta > 0, \sigma_1^2, \sigma_2^2, \hat{\sigma}_2^2 \geq 0, q \in \mathbb{N}, w_0 \in \WW$.
 \FOR{$r \in \{0, 1, \cdots, R\}$} 
\FOR{$i \in S_r$ \textbf{in parallel}}
\STATE Server sends global model $w_r$ to silo $i$. 
\IF{$r \equiv 0~(\text{mod} ~q)$}
\STATE Silo $i$ draws $K_1$ samples $\{x_{i,j}^r\}_{j=1}^{K_1}$ u.a.r. from $X_i$ (with replacement) and $u_1^i \sim \mathcal{N}(0, \sigma_1^2 \mathbf{I}_d)$. 
\STATE Silo $i$ computes $h_r^i = \frac{1}{K_1}\sum_{j=1}^{K_1} \nabla f^0(w_r, x_{i,j}^r) + u_1^i$. 
\STATE Server aggregates $h_r = \frac{1}{M_r}\sum_{i \in S_r} h_r^i$. 
\ELSE 
\STATE Silo $i$ draws $K_2$ samples $\{x_{i,j}^r\}_{j=1}^{K_1}$ u.a.r. from $X_i$ (with replacement) and $u_2^i \sim \mathcal{N}(0, \mathbf{I}_d \min\{\sigma_2^2\|w_r - w_{r-1}\|^2, \hat{\sigma}_2^2\})$.
\STATE Silo $i$ computes $H_r^i = \frac{1}{K_2} \sum_{j=1}^{K_2}[\nabla f^0(w_{r}, x_{i,j}^{r}) - \nabla f^0(w_{r-1}, x_{i,j}^{r})] + u_2^i$.
\STATE Server aggregates $H_r = \frac{1}{M_r}\sum_{i \in S_r} H_r^i$ and $h_r = h_{r-1} + H_r$.  
\ENDIF
\ENDFOR 
\STATE Server updates $w_{r+1} = \prox_{\eta f^1}(w_r - \eta h_r)$.
\ENDFOR \\
\STATE {\bfseries Output:} $\wpr \sim \text{Unif}(\{w_{r}\}_{r=1, \cdots, R})$.
\end{algorithmic}
\end{algorithm}

The essential elements of the algorithms are: the variance-reduced Stochastic Path-Integrated Differential EstimatoR of the gradient in line 11; and the non-standard choice of privacy noise in line 10 (inspired by~\citep{arora2022faster}), in which we choose $\sigma_2^2 = \frac{16 \beta^2 R \ln(1/\delta)}{\epsilon^2 n^2}$.    
With careful choices of $\eta, \sigma_1^2, \hat{\sigma}_2^2, q, R$ in ISRL-DP FedProx-SPIDER, our algorithm achieves state-of-the-art utility: 
\begin{theorem}[ISRL-DP FedProx-SPIDER, $M=N$ version]
\label{thm: ISRL-DP proxspider}
Let $\epsilon \leq 2 \ln(1/\delta)$. Then, ISRL-DP FedProx-SPIDER is $(\epsilon, \delta)$-ISRL-DP. Moreover, if $M=N$, then
\vspace{-.1cm}
\small
\begin{equation*}
\small
\EGMN \lesssim
\left(\frac{\sqrt{L \beta \hat{\Delta}_{\bx} d \ln(1/\delta)}}{\epsilon n\sqrt{N}}\right)^{4/3} + \frac{L^2 d \ln(1/\delta)}{\epsilon^2 n^2 N}.
\normalsize
\vspace{-.2cm}
\end{equation*}
\end{theorem}
\noindent See~\cref{app: proxspider} for the general statement for $M \leq N$, and the detailed proof. \cref{thm: ISRL-DP proxspider} provides the first utility bound for \textit{any} kind of DP optimization problem (even \textit{centralized}) with non-convex/non-smooth losses. In fact, the only work we are aware of that addresses DP non-convex optimization with $f^1 \neq 0$ is~\citep{bg21}, which considers CDP constrained smooth non-convex SO with $N=1$ (i.i.d.) and $f^1 = \iota_{\WW}$. However, their noisy Franke-Wolfe method could not handle general non-smooth $f^1$. 
Further, handing $N > 1$ heterogeneous silos requires additional work. 

The improved utility that our algorithm offers compared with existing DP FL works (discussed in~\cref{sec: intro}) stems from the variance-reduction that we get from: a) using smaller privacy noise that scales with $\beta \|w_t - w_{t-1}\|$ and shrinks as $t$ increases (in expectation);
and b) using SPIDER updates. By $\beta$-smoothness of~$f^0$, we can bound the sensitivity of the local updates and use standard DP arguments to prove ISRL-DP of~\cref{alg: LDP proxspider}. A key step in the proof of the utility bound in~\cref{thm: ISRL-DP proxspider} involves extending classical ideas from~\citep[p. 269-271]{bubeck} for constrained convex optimization to the noisy distributed non-convex setting and leveraging non-expansiveness of the proximal operator in the right way. 
Our SDP FedProx-SPIDER~\cref{alg: SDP proxspider} is described in~\cref{app: proxspider}. SDP FedProx-SPIDER is similar to~\cref{alg: LDP proxspider} except that Gaussian noises get replaced by appropraitely re-calibrated binomial noises plus shuffling. It's privacy and utility guarantees are provided in~\cref{thm: sdp proxspider}:
\begin{theorem}[SDP FedProx-SPIDER, $M=N$ version]
\label{thm: sdp proxspider}
Let $\epsilon \leq \ln(1/\delta),
~\delta \in (0, \frac{1}{2})$. Then, there exist algorithmic parameters such that SDP FedProx-SPIDER is $(\epsilon, \delta)$-SDP and
\small
\vspace{-.1cm}
\begin{align*}
\small
\EGMN &= \widetilde{\mathcal{O}}\Bigg(\left[\frac{\sqrt{L \beta \hat{\Delta}_{\bx} d \ln(1/\delta)}}{\epsilon n N}\right]^{4/3} \\
&\;\;\;\;+ \frac{d L^2 \ln(1/\delta)}{\epsilon^2 n^2 N^2} 
\Bigg).
\normalsize
\vspace{-.2cm}
\end{align*}
\end{theorem}
\vspace{-.05in}
Our \textit{non-smooth, SDP} federated ERM bound in~\cref{thm: sdp proxspider} improves over the state-of-the-art~\textit{CDP, smooth} federated ERM bound of \citep{wang2019efficient}, which is $\mathcal{O}(\sqrt{d}/\epsilon n N)$. We obtain this improved utility even under the weaker assumptions of \textit{non-smooth} loss and \textit{no trusted server}. 

\vspace{-.05in}
\subsection{Lower Bounds}
\label{sec: lower bound}
\vspace{-.05in}
We complement our upper bounds with lower bounds: 
\begin{theorem}[Smooth Convex Lower Bounds, Informal]
\label{thm: LDP lower bound}
Let $\epsilon \lesssim 1$ and $2^{-\Omega(nN)} \leq \delta \leq 1/(nN)^{1 + \Omega(1)}$. 
Then, there is an $L$-Lispchitz, $\beta$-smooth, convex loss $f: \mathbb{R}^d \times \XX \to \mathbb{R}$ and a database $\bx \in \XX^{n \times N}$ such that any compositional, symmetric \footnote{See~\cref{app: lower bounds} for precise definitions; to the best of our knowledge, every DP ERM algorithm proposed in the literature is compositional and symmetric.} $(\epsilon, \delta)$-ISRL-DP algorithm $\Al$ satisfies 
\vspace{-.11cm}
\small
\[
\vspace{-.16cm}
\small
\expec\|\nabla \hf_{\bx}(\Al(\bx))\|^2 = \Omega\left(L^2\min\left\{1, \frac{d \ln(1/\delta)}{\epsilon^2 n^2 N}\right\}\right).
\]
\normalsize
Further, any $(\epsilon, \delta)$-SDP algorithm satisfies 
\small
\vspace{-.11cm}
\[
\small
\expec\|\nabla \hf_{\bx}(\Al(\bx))\|^2 = \Omega\left(L^2\min\left\{1, \frac{d \ln(1/\delta)}{\epsilon^2 n^2 N^2}\right\}\right).
\vspace{-.2cm}
\]
\normalsize
\end{theorem}

The proof (and formal statement) of the ISRL-DP lower bound is relegated to~\cref{app: lower bounds}; the SDP lower bound follows directly from the CDP lower bound of~\citep{arora2022faster}. Intuitively, it is not surprising that there is a gap between the non-convex/non-smooth upper bounds in~\cref{thm: ISRL-DP proxspider} and the smooth, convex lower bounds, since smooth convex optimization is easier than non-convex/non-smooth optimization.\footnote{For example, the non-private sample complexity of smooth convex SO is significantly smaller than the sample complexity of non-private non-convex SO~\citep{ny, foster2019complexity, carmon2019lower}.} As discussed in~\citep[Appendix B.2]{arora2022faster}, the non-private lower bound of~\citep{carmon2019lower} provides some evidence that their CDP ERM bound (which our SDP bound matches when $M=N$) is tight {for noisy gradient methods}.\footnote{Note that the non-private first-order oracle complexity lower bound of~\citep{carmon2019lower} requires a very high dimensional construction, restricting its applicability to the private setting.} By~\cref{thm: LDP lower bound}, this would imply that our ISRL-DP ERM bound is also tight. 
Rigorously proving \edit{tight bounds} is left as an interesting open problem.

\vspace{-.1in}
\section{NUMERICAL EXPERIMENTS}
\label{sec: experiments}
\vspace{-.1in}
To evaluate the performance of ISRL-DP FedProx-SPIDER, we compare it against standard FL baselines for privacy levels ranging from $\epsilon = 0.75$ to $\epsilon = 18$: Minibatch SGD (MB-SGD), Local SGD (a.k.a. Federated Averaging)~\citep{mcmahan2017originalFL}, ISRL-DP MB-SGD~\citep{lr21fl}, and ISRL-DP Local SGD. We fix $\delta = 1/n^2$. 
Note that FedProx-SPIDER generalizes MB-SGD (take $q = 1$). Therefore, ISRL-DP FedProx-SPIDER always performs at least as well as ISRL-DP MB-SGD, with performance being identical when the optimal phase length hyperparameter is $q = 1$. 

\edit{The main takeaway from our numerical experiments is that ISRL-DP FedProx-SPIDER outperforms the other ISRL-DP baselines for most privacy levels. To quantify the advantage of our algorithm, we computed the percentage improvement in test error over baselines in each experiment and privacy ($\epsilon$) level, and averaged the results: our algorithm improves over ISRL-DP Local SGD by 6.06\% on average and improves over
ISRL-DP MB-SGD by 1.72\%. Although the advantage over MB-SGD may not seem substantial, it is promising that our algorithm dominated MB-SGD in every experiment: ISRL-DP MB-SGD never outperformed ISRL-DP SPIDER for any value of~$\epsilon$.} More details about the experiments and additional results are provided in~\cref{app: experiment details}. \edit{All codes are publicly available at: \href{https://github.com/ghafeleb/Private-NonConvex-Federated-Learning-Without-a-Trusted-Server}{https://github.com/ghafeleb/Private-NonConvex-Federated-Learning-Without-a-Trusted-Server}.}

\textbf{\ul{Neural Net (NN) Classification with MNIST:}} 
Following~\citep{woodworth2020, lr21fl}, we partition the MNIST~\citep{mnist} data set into $N = 25$ heterogeneous silos, each containing one odd/even digit pairing. The task is to classify digits as even or odd. 
We use a two-layer perceptron with a hidden layer of 64 neurons. As Figure~\ref{fig:M25R25}
and Figure~\ref{fig:M12R50} 
show, \textit{ISRL-DP FedProx-SPIDER tends to outperform both ISRL-DP baselines}.
\begin{figure}[h] 
  \centering
  \subfigure{
    \includegraphics[width
     =1.0
    \linewidth]{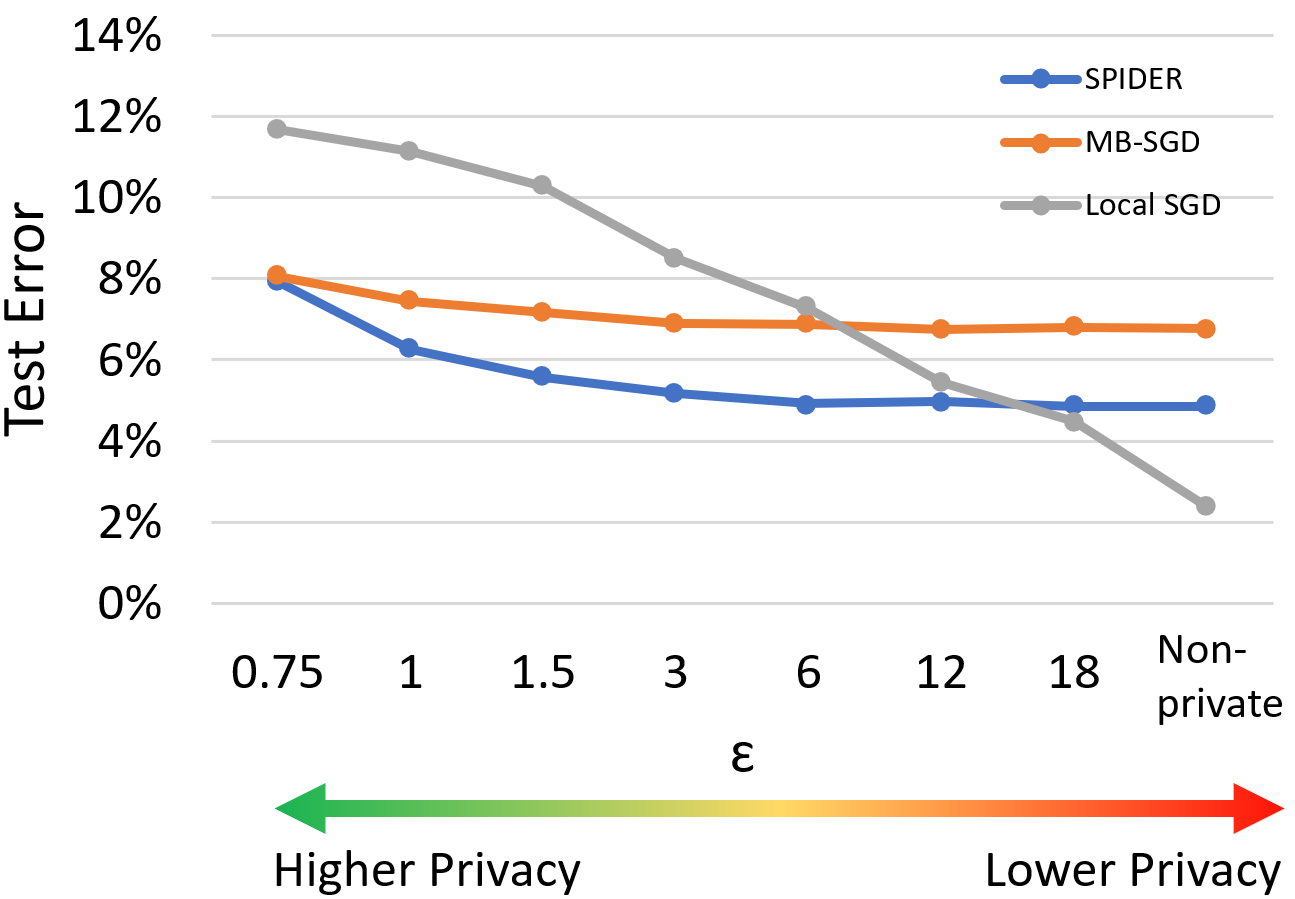}  
    }
  \caption{MNIST data. $M = 25, R = 25$.}
\label{fig:M25R25}
\end{figure}
\begin{figure}[h] 
  \centering
\subfigure{
\includegraphics[width = .90
\linewidth]{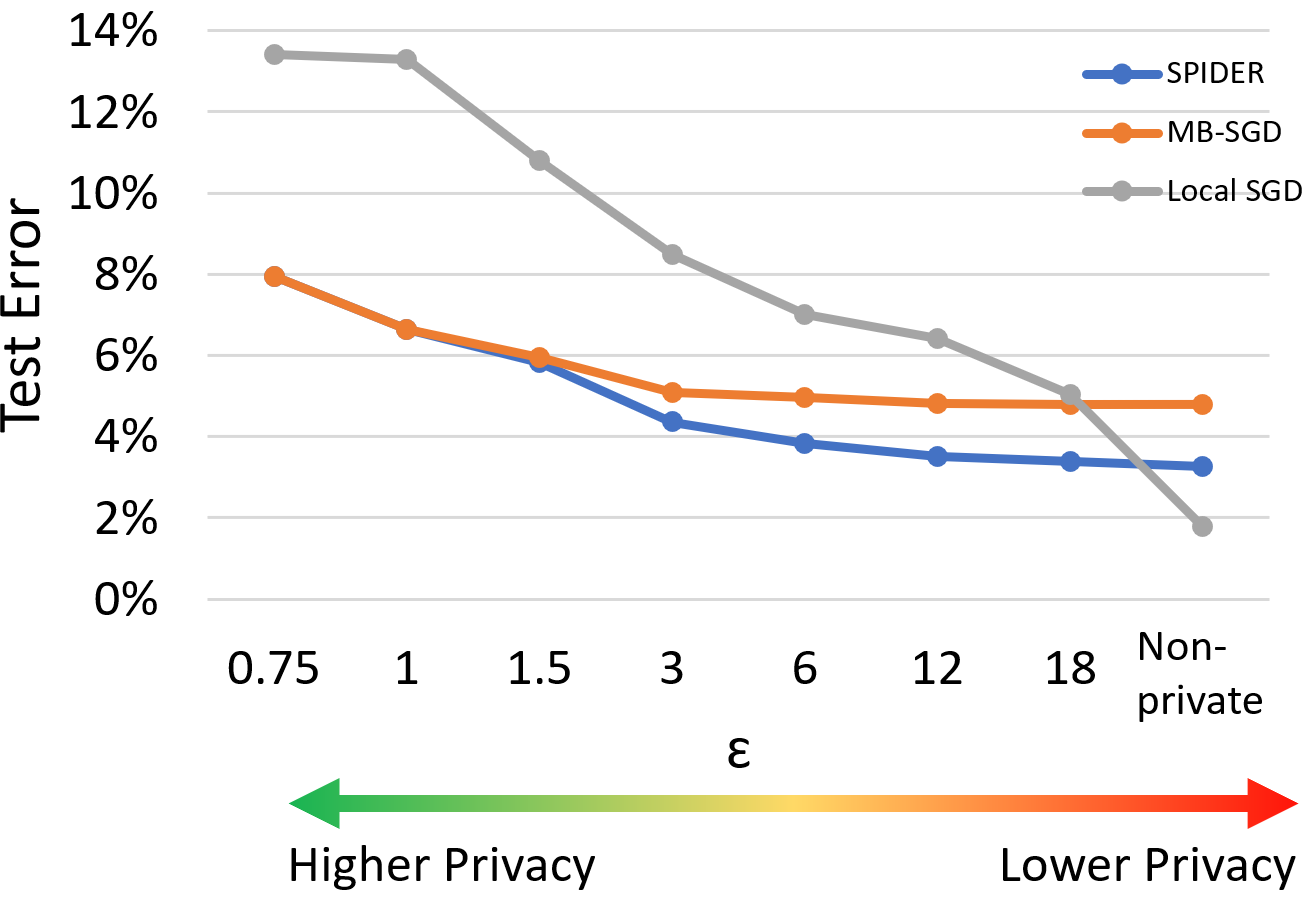}
}
  \caption{MNIST data. $M = 12, R = 50$.}
\label{fig:M12R50}
\end{figure}

\noindent \textbf{\ul{Convolutional NN Classification with CIFAR-10:}} 
 CIFAR-10 dataset~\citep{krizhevsky2009learning} includes 10 image classes and we partition it into $N=10$ heterogeneous silos, each containing one class. Following~\citep{pytorchCNN}, we use a 5-layer CNN with two 5x5 convolutional layers (the first with 6 channels, the second with 16 channels, each followed by a ReLu activation and a 2x2 max pooling) and three fully connected layers with 120, 84, 10 neurons in each fully connected layer. The first and second fully connected layers are followed by a ReLu activation. 
 As Figure~\ref{fig:CIFAR10M10R50} and Figure~\ref{fig:CIFAR10M10R100} show, \textit{ISRL-DP FedProx-SPIDER outperforms both ISRL-DP baselines for most tested privacy levels}. 
\begin{figure}[h] 
  \centering
  \subfigure{
    \includegraphics[width = .90\linewidth]{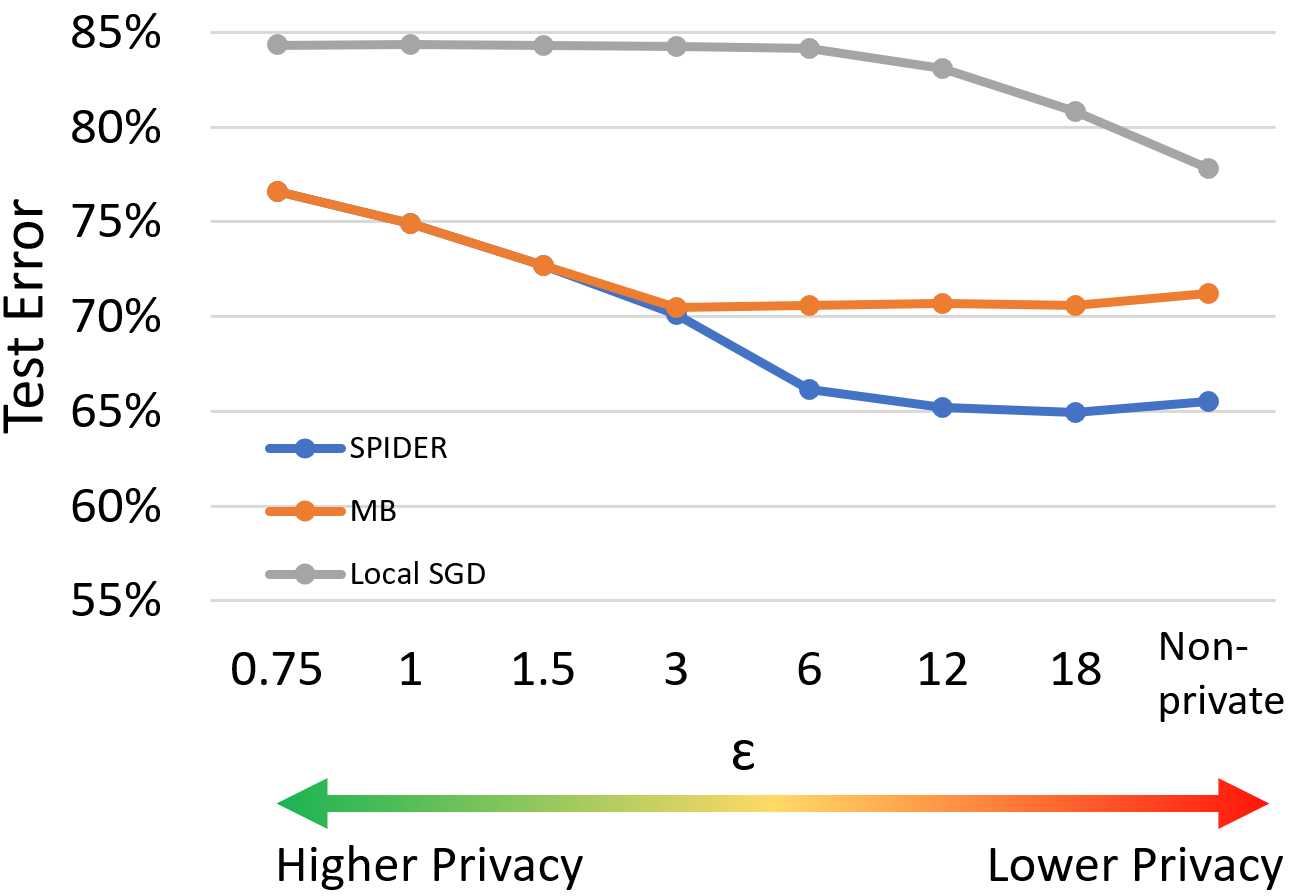}  
    }
  \caption{CIFAR-10 data. $M = 10, R = 50$.}
\label{fig:CIFAR10M10R50}
\end{figure}
\begin{figure}[h] 
  \centering
\subfigure{
\includegraphics[width = .90\linewidth]{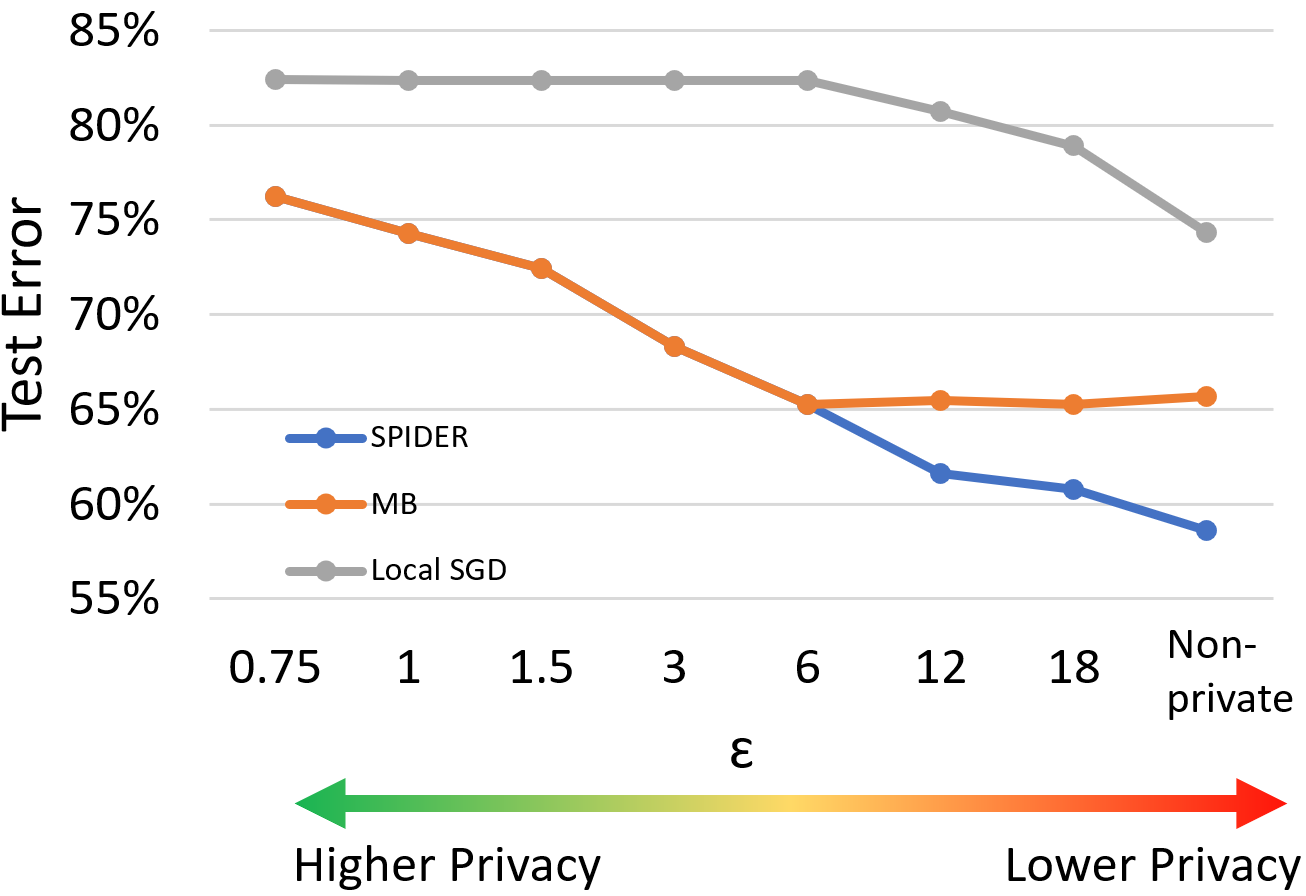}
}
  \caption{CIFAR-10 data. $M = 10, R = 100$.}
\label{fig:CIFAR10M10R100}
\end{figure}

\textbf{\ul{Neural Net Classification with Breast Cancer Data:}} 
With the Wisconsin Breast Cancer (Diagnosis) (WBCD) dataset~\citep{Dua2019}, our task is to diagnose breast cancers as malignant vs. benign. We partition the data set into $N=2$ heterogeneous silos, one containing malignant labels and the other benign labels. We use a 2-layer perceptron with 5 neurons in the hidden layer.
As Figure~\ref{fig:M2R25} shows, \textit{ISRL-DP FedProx-SPIDER outperforms both ISRL-DP baselines for most tested privacy levels}. 

\begin{figure}[ht] 
  \centering
  \begin{tabular}{@{}c@{}}
    \includegraphics[width
    =0.90
    \linewidth]{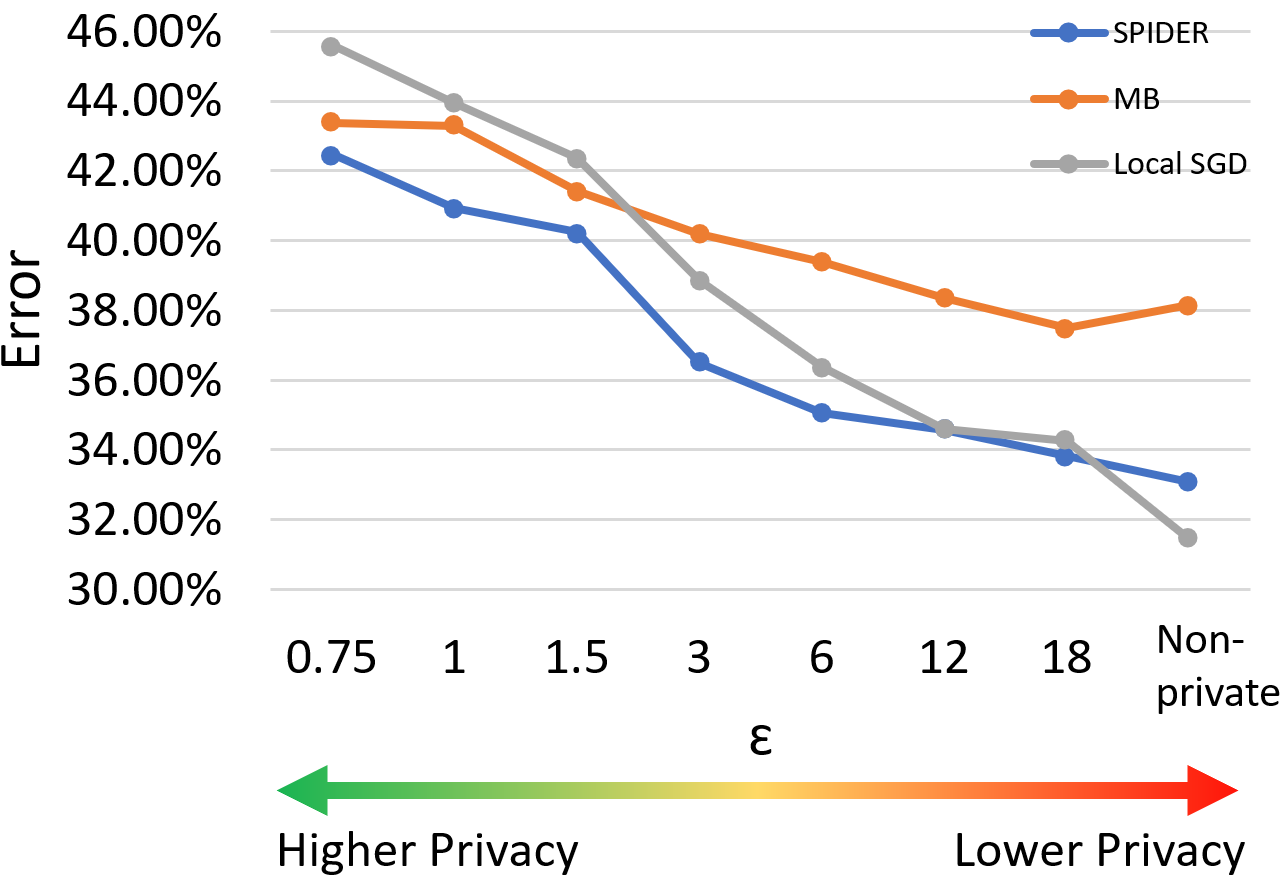} \\[\abovecaptionskip]
  \end{tabular}
  \caption{WBCD data. $M = 2, R = 25$.}
  \label{fig:M2R25}
\end{figure} 

\section{
CONCLUDING REMARKS AND OPEN QUESTIONS
}
We considered non-convex FL in the absence of trust in the server or other silos. We discussed the merits of ISRL-DP and SDP in this setting. For two broad classes of non-convex loss functions, we provided novel ISRL-DP/SDP FL algorithms and utility bounds that advance the state-of-the-art. 
For proximal-PL losses, our algorithms are nearly optimal and show that neither convexity or i.i.d. data is required to obtain fast rates. Numerical experiments demonstrated the practicality of our algorithm at providing both high accuracy and privacy on several learning tasks and data sets. 
An interesting open problem is proving tight bounds on the gradient norm for private non-convex FL. We discuss limitations and societal impacts of our work in~\cref{app: limitations,app: impacts}. 

\subsubsection*{Acknowledgements}
This work was supported in part with funding from the AFOSR Young Investigator Program award and a gift from the USC-Meta Center for Research and Education in AI and Learning. %

\bibliography{references.bib}
\bibliographystyle{abbrvnat}
\setcitestyle{authoryear,open={((},close={))}} %
\clearpage
\appendix
\onecolumn 
\section*{SUPPLEMENTARY MATERIAL}

\section{Inter-Silo Record Level Differential Privacy (ISRL-DP): Rigorous Definition}
\label{app: ISRL-DP}
Let $\mathcal{A}
$ 
be a randomized algorithm, where the silos communicate over $R$ rounds for their FL task. In each round of communication $r \in [R]$, each silo $i$ transmits a message $Z_r^{(i)} \in \ZZ$ (e.g. stochastic gradient) to the server or other silos, and the messages are aggregated. The transmitted message $Z_r^{(i)}$ is a (random) function of previously communicated messages and the data of user $i$; that is, $Z^{(i)}_r:= \rand_r(\bz_{1:r-1}, X_i)$, where $\bz_{1:r-1}:= \{Z_t^{(j)}\}_{j \in [N], t \in [r-1]}$. 
The silo-level local privacy of $\Al$ is completely characterized by the \textit{randomizers} $\rand_r: \ZZ^{(r-1) \times N} \times \XX^{n_i} \to \ZZ$ ($i \in [N], ~r \in [R]$).\footnote{We assume $\rand_r(\bz_{1:r-1}, X_i)$ does not depend on $X_j$ ($j \neq i$) given $\bz_{1:r-1}$ and $X_i$; i.e. the distribution of the random function $\rand_r$ is completely characterized by $\bz_{1:r-1}$ and $X_i$. Thus, randomizers of $i$ cannot ``eavesdrop'' on another silo's data. This is consistent with the local data principle of FL. We allow for $Z^{(i)}_r$ to be empty/zero if silo $i$
does not output anything to the server in round~$r$.}   
$\Al$ is $(\epsilon, \delta)$-ISRL-DP if for all $i \in [N]$, the full transcript of silo $i$'s communications 
(i.e. the collection of all $R$ messages $\{Z^{(i)}_r\}_{r \in [R]}$)
is $(\epsilon, \delta)$-DP, conditional on the messages 
and data 
of all other silos. 
We write \small{
\small $\Al(\bx) \edsim \Al(\bx')$} \normalsize if~\cref{eq: DP} holds for all measurable subsets $S$.

\begin{definition}{(Inter-Silo Record Level Differential Privacy)}
\label{def: ISRL-DP}
Let $\rho_i: \XX^2 \to [0, \infty)$, $\rho_i(X_i, X'_i) := \sum_{j=1}^{n} \mathds{1}_{\{x_{i,j} \neq x_{i,j}'\}}$, $i \in [N]$. A randomized algorithm $\Al: \XX_1^{n} \times \cdots \times \XX_N^{n} \to \ZZ^{R \times N}$ is \textit{
$(\epsilon, \delta)$-ISRL-DP
} if for all $i \in [N]$ and all $\rho_i$-adjacent $X_i, X_i' \in \XX^{n}$, we have
\small
$
\small
(\rand_1(X_i), \rand_2(\bz_1, X_i), \cdots ,\rand_R(\bz_{1:R-1}, X_i)) \edsim  
(\rand_1(X'_i), \rand_2(\bz'_1, X'_i), \cdots , \rand_R(\bz'_{1:R-1}, X'_i)),
$
where $\bz_r := \{\rand_r(\bz_{1:r-1}, X_i)\}_{i=1}^N$ and $\bz'_r := \{\rand_r(\bz'_{1:r-1}, X_i')\}_{i=1}^N$.
\normalsize
\end{definition}

\section{Relationships Between Notions of Differential Privacy}
\label{app: dp relationships}
In this section, we collect a couple of facts about the relationships between different notions of DP, which were proved in \citep{lr21fl}. Suppose $\Al$ is $(\epsilon, \delta)$-ISRL-DP. Then:
\begin{itemize}
    \item $\Al$ is $(\epsilon, \delta)$-CDP; and  
    \item $\Al$ is $(n\epsilon, ne^{(n-1)\epsilon}\delta)$-user-level DP. 
\end{itemize}

Thus, if $\epso \lesssim 1/n$ and $\delo = o(1/n^2)$, then any $(\epso, \delo)$-ISRL-DP algorithm also provides a meaningful $(\epsilon, \delta)$-user-level DP guarantee, with $\epsilon \lesssim 1$.

\section{Further Discussion of Related Work}
\label{app: related work}
\noindent \textbf{DP Optimization with the Polyak-\L ojasiewicz (PL) Condition:}
For \textit{unconstrained} central DP optimization, \citep{wang2017ermrevisited, kang2021weighted, zhangijcai} provide bounds for Lipschitz losses satisfying the \textit{classical} PL inequality. However, the combined assumptions of Lipschitzness and PL on $\mathbb{R}^d$ (unconstrained) are very strong and rule out most interesting PL losses, such as strongly convex, least squares, and neural nets, since the Lipschitz parameter $L$ of such losses is infinite or prohibitively large.\footnote{In particular, the DP ERM/SCO strongly convex, Lipschitz lower bounds of \citep{bst14, bft19} do not imply lower bounds for the unconstrained Lipschitz, PL function class considered in these works, since the quadratic hard instance of \citep{bst14} is not $L$-Lipschitz on all of $\mathbb{R}^d$ for any $L < +\infty$.} We address this gap by considering the \textit{Proximal} PL condition, which admits such interesting loss functions. There was no prior work on DP optimization (in particular, FL) with the Proximal PL condition. 

\noindent \textbf{DP Smooth Non-convex Distributed ERM:} Non-convex federated ERM has been considered in previous works under stricter assumptions of \textit{smooth} loss and (usually) a trusted server. 
\citep{wang2019efficient} provide state-of-the-art \textit{CDP} upper bounds for distributed ERM of order $\EG \lesssim \left(\frac{\sqrt{d}}{\epsilon nN} \right)$ with perfect communication ($M=N$), relying on a trusted server (in conjunction with secure multi-party computation) to perturb the aggregated gradients. Similar bounds were attained by~\citep{noble2022differentially} for $M < N$ with a DP variation of SCAFFOLD~\citep{scaffold}. 
In \cref{thm: sdp proxspider}, we improve on these utility bound under the \textit{weaker trust model of shuffle DP} (no trusted server) and with \textit{unreliable communication} (i.e. arbitrary $M \in [N]$). We also improve over the state-of-the-art ISRL-DP bound of~\citep{noble2022differentially}, in~\cref{thm: ISRL-DP proxspider}. 
A number of other works have also addressed private non-convex federated ERM (under various notions of DP), but have fallen short of the state-of-the-art utility and communication complexity bounds: 
\begin{itemize}
    \item The noisy FedAvg algorithm of \citep{hu21} \textit{is not ISRL-DP} for any $N > n$ since the variance of the Gaussian noise $\sigma^2 \approx TKL^2 \log(1/\delta)/nN\epsilon^2$ decreases as $N$ increases; moreover, for their prescribed stepsize $\eta = \frac{\sqrt{N}}{\sqrt{T}}$, the resulting rate (with $T = RK$) from \cite[Theorem 2]{hu21} is $\mathbb{E}\| \nabla \widehat{F}(\widehat{w}_R) \|^2 = \widetilde{O}\left(\frac{d\sqrt{NT}K}{\epsilon^2 nN} + \frac{NK^2}{T} + \frac{\sqrt{N}}{\sqrt{T}} + \frac{dK^2}{\epsilon^2 n}\right)$ which grows unbounded with $T$. Moreover, $T$ and $K$ are not specified in their work, so it is not clear what bound their algorithm is able to attain, or how many communication rounds are needed to attain it. 
\item Theorems 3 and 7 of \citep{ding2021} provide ISRL-DP upper bounds on the empirical gradient norm which hold for sufficiently large $R \geq T_{\min}^{\text{nc}}$ for some unspecified $T_{\min}^{\text{nc}}$. The resulting upper bounds are bigger than $\frac{d \sigma^2}{R^{1/3}} \approx \frac{d R^{2/3}}{\epsilon^2 n^2}$. In particular, the bounds becomes trivial for large $R$ (diverges) and no utility bound expressed in terms of problem parameters (rather than unspecified design parameters $R$ or $T$) is provided. Also, no communication complexity bound is provided. 
\end{itemize}

\noindent \textbf{DP Smooth Non-convex Centralized ERM ($N=1$):}
In the centralized setting with a single client and smooth loss function, several works \citep{zhang2017efficient, wang2017ermrevisited, wang2019efficient, arora2022faster} have considered CDP (unconstrained) non-convex ERM (with gradient norm as the utility measure): the state-of-the-art bound is $\expec\|\nabla \hf_{\bx}(\wpr)\|^2 = \mathcal{O}\left(\frac{\sqrt{d \ln(1/\delta)}}{\epsilon n}\right)^{4/3}$~\citep{arora2022faster}. Our private FedProx-SPIDER algorithms build on the DP SPIDER-Boost of~\citep{arora2022faster}, by parallelizing their updates for FL and incorporating proximal updates to cover non-smooth losses.

\noindent \textbf{Non-private FL:} In the absence of privacy constraints, there is a plethora of works studying the convergence of FL algorithms in both the convex~\citep{koloskova, li2020fedavg, scaffold, woodworth2020local, woodworth2020, yuan20} and non-convex~\citep{li2020federated, zhang2020fedpd, scaffold} settings. We do not attempt to provide a comprehensive survey of these works here; see~\citep{kairouz2019advances} for such a survey. However, we briefly discuss some of the well known non-convex FL works: \begin{itemize}
    \item The ``FedProx'' algorithm of~\citep{li2020federated} augments FedAvg~\citep{mcmahan2017originalFL} with a regularization term in order to decrease ``client drift'' in heterogeneous FL problems with smooth non-convex loss functions. (By comparison, we use a proximal term in our private algorithms to deal with \textit{non-smooth} non-convex loss functions, and show that our algorithms effectively handle heterogeneous client data via careful analysis.)  
    \item \citep{zhang2020fedpd} provides primal-dual FL algorithms for non-convex loss functions that have optimal communication complexity (in a certain sense). 
    \item The SCAFFOLD algorithm of~\citep{scaffold} can be viewed as a hybrid between Local SGD (FedAvg) and MB-SGD, as~\citep{woodworth2020} observed. Convergence guarantees for their algorithm with non-convex loss functions are provided. 
\end{itemize}

\section{Differential Privacy Building Blocks}
\label{app: SDP}
\textbf{\ul{Basic DP tools:}} We begin by reviewing the privacy guarantees of the \textit{Gaussian mechanism}. The classical $(\epsilon, \delta)$-DP bounds for the Gaussian mechanism~\citep[Theorem A.1]{dwork2014} were only proved for $\epsilon \leq 1$, so we shall instead state the privacy guarantees in terms of \textit{zero-concentrated differential privacy (zCDP)}--which should not be confused with central differential privacy (CDP)--and then convert these into $(\epsilon, \delta)$-DP guarantees. We first recall~\citep[Definition 1.1]{bun16}:
\begin{definition}
A randomized algorithm $\mathcal{A}: \XX^n \to \mathcal{W}$ satisfies $\rho$-zero-concentrated differential privacy ($\rho$-zCDP) if for all $X, X' \in \XX^n$ differing in a single entry (i.e. $d_{\text{hamming}}(X, X') = 1$), and all $\alpha \in (1, \infty)$, we have \[
D_\alpha(\Al(X) || \Al(X')) \leq \rho \alpha,
\]
where $D_\alpha(\Al(X) || \Al(X'))$ is the $\alpha$-\Renyi divergence\footnote{For distributions $P$ and $Q$ with probability density/mass functions $p$ and $q$, $D_\alpha(P || Q) := \frac{1}{\alpha - 1}\ln \left(\int p(x)^{\alpha} q(x)^{1 - \alpha} dx\right)$~\citep[Eq. 3.3]{renyi}.} between the distributions of $\Al(X)$ and $\Al(X')$.
\end{definition}

zCDP is weaker than pure DP, but stronger than approximate DP in the following sense: 
\begin{proposition}{\citep[Proposition 1.3]{bun16}}
\label{prop:bun1.3}
If $\Al$ is $\rho$-zCDP, then $\Al$ is  $(\rho + 2\sqrt{\rho \log(1/\delta)}, \delta)$ for any $\delta > 0$. In particular, if $\epsilon \leq 2\ln(1/\delta)$, then any $\frac{\epsilon^2}{8 \ln(1/\delta)}$-zCDP algorithm is $(\epsilon, \delta)$-DP. 
\end{proposition}

The privacy guarantee of the Gaussian mechanism is as follows:
\begin{proposition}{\citep[Proposition 1.6]{bun16}}
\label{prop: gauss}
Let $q: \XX^n \to \mathbb{R}$ be a query with $\ell_2$-sensitivity $\Delta := \sup_{X \sim X'}\|q(X) - q(X')\|$. Then the Gaussian mechanism, defined by $\mathcal{M}: \XX^n \to \mathbb{R}$, $\mathcal{M}(X) := q(X) + u$ for $u \sim \mathcal{N}(0, \sigma^2)$, is $\rho$-zCDP if $\sigma^2 \geq \frac{\Delta^2}{2\rho}$. Thus, if $\epsilon \leq 2\ln(1/\delta)$ and $\sigma^2 \geq \frac{4 \Delta^2 \ln(1/\delta)}{\epsilon^2}$, then $\mathcal{M}$ is $(\epsilon, \delta)$-DP. 
\end{proposition}

Our multi-pass algorithms will also use advanced composition for their privacy analysis: 
\begin{theorem}{\citep[Theorem 3.20]{dwork2014}}
\label{thm: advanced composition}
Let $\epsilon \geq\ 0, \delta, \delta' \in [0, 1)$. Assume $\mathcal{A}_1, \cdots, \mathcal{A}_R$, with $\Al_r: \XX^n \times \WW \to \WW$, are each $(\epsilon, \delta)$-DP ~$\forall r = 1, \cdots, R$. Then, the adaptive composition $\Al(X) := \Al_R(X, \Al_{R-1}(X, \Al_{R-2}(X, \cdots)))$ is $(\epsilon', R\delta + \delta')$-DP for $\epsilon' = \sqrt{2R \ln(1/\delta')} \epsilon + R\epsilon(e^{\epsilon} - 1)$.
\end{theorem}

Note that the moments accountant~\citep{abadi16} provides a slightly tighter composition bound (by a logarithmic factor) and can be used instead of~\cref{thm: advanced composition} if one is concerned with logarithmic factors. We use the moments accountant for our numerical experiments: see~\cref{app: experiment details}. Sometimes it is more convenient to analyze the compositional privacy guarantees of our algorithm through the lens of zCDP: 

\begin{lemma}{\cite[Lemma 2.3]{bun16}}
\label{lem: zCDP composition}
Suppose $\mathcal{A}: \XX^n \to \mathcal{Y}$ satisfies $\rho$-zCDP and $\mathcal{A}': \XX^n \times \mathcal{Y} \to \ZZ$ satisfies $\rho'$-zCDP (as a function of its first argument). Define the composition of $\mathcal{A}$ and $\mathcal{A}'$, $\mathcal{A}'': \XX^n \to \ZZ$ by
$\mathcal{A}''(X) = \mathcal{A}'(X, \mathcal{A}(X))$. Then $\mathcal{A}''$ satisfies $(\rho + \rho')$-zCDP. In particular, the composition of $T$ $\rho$-zCDP mechanisms is a $T\rho$-zCDP mechanism. 
\end{lemma}

\noindent \textbf{\ul{Shuffle Private Vector Summation:}}
Here we recall the shuffle private vector summation protocol $\mathcal{P}_{\text{vec}}$ of \citep{cheu2021shuffle}, and its privacy and utility guarantee. First, we will need the scalar summation protocol,~\cref{alg: P1D}. Both of~\cref{alg: P1D} and~\cref{alg: Pvec} decompose into a local randomizer $\mathcal{R}$ that silos perform and an analyzer component $\mathcal{A}$ that the shuffler executes. Below we use $\mathcal{S}(\mathbf{y})$ to denote the shuffled vector $\mathbf{y}$: i.e. the vector with same dimension as $\mathbf{y}$ whose components are random permutations of the components of $\mathbf{y}$. 

\begin{algorithm}
\caption{$\mathcal{P}_{\text{1D}}$, a shuffle protocol for summing scalars \citep{cheu2021shuffle}}
\label{alg: P1D}
\begin{algorithmic}[1]
\STATE {\bfseries Input:} 
Scalar database $X = (x_1, \cdots x_N) \in [0,L]^N$; $g, b \in \mathbb{N}; p \in (0, \frac{1}{2})$. 
\STATE {\bfseries procedure: Local Randomizer $\mathcal{R}_{1D}(x_i)$}
\begin{ALC@g}
\STATE $\widebar{x}_i \gets \lfloor x_i g/L \rfloor$.
\STATE Sample rounding value $\eta_1 \sim \textbf{Ber}(x_i g/L - \widebar{x}_i)$.
\STATE Set $\hat{x}_i \gets \widebar{x}_i + \eta_1$.
\STATE Sample privacy noise value $\eta_2 \sim \textbf{Bin}(b,p)$.
\STATE Report $\mathbf{y}_i \in \{0,1\}^{g + b}$ containing $\hat{x}_i + \eta_2$ copies of $1$ and $g + b - (\hat{x}_i + \eta_2)$ copies of $0$.
\end{ALC@g}
\STATE {\bfseries end procedure}
\STATE{\bfseries procedure: Analyzer} $\mathcal{A}_{\text{1D}}(\mathcal{S}(\mathbf{y}))$
\begin{ALC@g}
\STATE Output estimator $\frac{L}{g}((\sum_{i=1}^{N}\sum_{j=1}^{b+g} (\mathbf{y}_i)_j) - pbn)$.
\end{ALC@g}
\STATE {\bfseries end procedure}
\end{algorithmic}
\end{algorithm}

The vector summation protocol~\cref{alg: Pvec} invokes the scalar summation protocol,~\cref{alg: P1D}, $d$ times. In the Analyzer procedure, we use $\mathbf{y}$ to denote the collection of all $Nd$ shuffled (and labeled) messages that are returned by the the local randomizer applied to all of the $N$ input vectors. Since the randomizer labels these messages by coordinate, $\mathbf{y}_j$ consists of $N$ shuffled messages labeled by coordinate $j$ (for all $j  \in [d]$).

\begin{algorithm}
\caption{$\mathcal{P}_{\text{vec}}$, a shuffle protocol for vector summation \citep{cheu2021shuffle}}
\label{alg: Pvec}
\begin{algorithmic}[1]
\STATE {\bfseries Input:} database of $d$-dimensional vectors $\bx = (\mathbf{x}_1, \cdots, \mathbf{x}_N$); privacy parameters $\epsilon, \delta$; $L$.
\STATE {\bfseries procedure:} Local Randomizer $\mathcal{R}_{\text{vec}}(\mathbf{x}_i)$
\begin{ALC@g}
\FOR{$j \in [d]$} 
\STATE Shift component to enforce non-negativity: $\mathbf{w}_{i,j} \gets \mathbf{x}_{i,j} + L$
\STATE $\mathbf{m}_j \gets \mathcal{R}_{1D}(\mathbf{w}_{i,j})$
\ENDFOR
\STATE Output labeled messages $\{(j, \mathbf{m}_j)\}_{j \in [d]}$
\end{ALC@g}
\STATE {\bfseries end procedure}
\STATE {\bfseries procedure: Analyzer} $\mathcal{A}_{\text{vec}}(\mathbf{y})$ 
\begin{ALC@g}
\FOR{$j \in [d]$}
\STATE Run analyzer on coordinate $j$'s messages $z_j \gets \mathcal{A}_{\text{1D}}(\mathbf{y}_j)$ 
\STATE Re-center: $o_j \gets z_j - L$
\ENDFOR
\STATE Output the vector of estimates $\mathbf{o} = (o_1, \cdots o_d)$
\end{ALC@g}
\STATE {\bfseries end procedure}
\end{algorithmic}
\end{algorithm}

When we use~\cref{alg: Pvec} in our SDP FL algorithms, each of the $M_r = M$ available silos contributes $K$ messages, so $N = MK$ in the notation of~\cref{alg: Pvec}. Also, $x_i$ represents $K$ stochastic gradients, and available silos perform $\mathcal{R}_{\text{vec}}$ on each one (in parallel) before sending the collection of all of these randomized, discrete stochastic gradients--denoted $\mathcal{R}_{\text{vec}}(\mathbf{x}_i)$--to the shuffler. The shuffler permutes the elements of $\mathcal{R}_{\text{vec}}(\mathbf{x}_1), \cdots \mathcal{R}_{\text{vec}}(\mathbf{x}_M)$, then executes $\mathcal{A}_{\text{vec}}$, and sends $\frac{1}{M} \mathbf{o}$--which is a noisy estimate of the average stochastic gradient--to the server. When there is no confusion, we will sometimes hide input parameters other than $\bx$ and denote $\mathcal{P}_{\text{vec}}(\bx) := \mathcal{P}_{\text{vec}}(\bx; \epsilon, \delta; L)$. We now provide the privacy and utility guarantee of~\cref{alg: Pvec}: 

\begin{theorem}[\citep{cheu2021shuffle}]
\label{thm:cheu vecsum}
For any $0 < \epsilon \leq 15, 0 < \delta < 1/2, d, N \in \mathbb{N}$, and $L > 0$, there are choices of parameters $b, g \in \mathbb{N}$ and $p \in (0, 1/2)$ for $\mathcal{P}_{\text{1D}}$ (\cref{alg: P1D}) such that, for $\bx = (\mathbf{x}_1, \cdots \mathbf{x}_N)$ containing vectors of maximum norm $\max_{i\in [N]} \|\mathbf{x}_i\|\leq L$, the following holds: 1) $\mathcal{P}_{\text{vec}}$ is $(\epsilon, \delta)$-SDP; and 2) $\mathcal{P}_{\text{vec}}(\bx)$ is an unbiased estimate of $\sum_{i=1}^N \mathbf{x}_i$ with bounded variance \[
\expec\left[\left\|\mathcal{P}_{\text{vec}}(\bx; \epsilon, \delta; L) - \sum_{i=1}^N \mathbf{x}_i\right\|^2\right] = \mathcal{O}\left(\frac{d L^2 \log^2\left(\frac{d}{\delta}\right)}{\epsilon^2}\right).
\]
\end{theorem}

\section{
Supplemental Material for~\cref{sec:PL}: Proximal PL Loss Functions}

\subsection{Noisy Proximal Gradient Methods forProximal PL FL (SO) - Pseudocodes}
\label{app: sdp prox grad SO}
\begin{algorithm}[ht]
\caption{ISRL-DP Noisy Distributed Proximal Gradient Method}
\label{alg: LDP Prox}
\begin{algorithmic}[1]
\STATE {\bfseries Input:} 
$R \in \mathbb{N}, X_i \in \XX^{n} (i \in [N]), \sigma^2 \geq 0, K \leq \frac{n}{R}, w_0 \in \mathbb{R}^d$. 
 \FOR{$r \in \{0, 1, \cdots, R-1\}$} 
 \FOR{$i \in S_r$ \textbf{in parallel}} 
  \STATE Server sends global model $w_r$ to silo $i$. 
 \STATE Silo $i$ draws $\{x_{i,j}^r\}_{j=1}^K$
uniformly from $X_i$  (without replacement) and noise $u_i \sim \mathcal{N}(0, \sigma^2 \mathbf{I}_d).$
 \STATE Silo $i$ sends $\widetilde{g}_r^{i} := \frac{1}{K} \sum_{j=1}^{K} \nabla f^0(w_r, x_{i,j}^r) + u_i$ to server.
 \ENDFOR
 \STATE Server aggregates $\widetilde{g}_{r} := \frac{1}{M_r} \sum_{i \in S_r} \widetilde{g}_r^{i}$.
 \STATE Server updates 
 $w_{r+1} :=\prox_{\frac{1}{2\beta}f^1}(w_r - \frac{1}{2\beta}\widetilde{g}_r)$
 \ENDFOR \\
\STATE {\bfseries Output:} $w_R$.
\end{algorithmic}
\end{algorithm}

\begin{algorithm}[ht]
\caption{SDP Noisy Distributed Proximal Gradient Method}
\label{alg: SDP Prox}
\begin{algorithmic}[1]
\STATE {\bfseries Input:} 
Number of rounds $R \in \mathbb{N}$,
data sets  $X_i \in \XX^{n_i}$ for $i \in [N]$, 
loss function $f(w, x) = f^0(w,x) + f^1(w,x)$, privacy parameters
$\epsilon, \delta$, 
local batch size $K \leq \frac{n}{R}$, $w_0 \in \mathbb{R}^d$. 
 \FOR{$r \in \{0, 1, \cdots, R-1\}$} 
 \FOR{$i \in S_r$ \textbf{in parallel}} 
  \STATE Server sends global model $w_r$ to silo $i$. 
 \STATE Silo $i$ draws 
 $K$
 samples $\{x_{i,j}^r\}_{j=1}^K$ 
uniformly from $X_i$ (without replacement) and computes $\{\nabla f^0(w_r, x_{i,j}^r)\}_{j \in [K]}$.
 \ENDFOR
 \STATE Server updates $\widetilde{g}_{r} := \frac{1}{M_r K} \mathcal{P}_{\text{vec}}(\{\nabla f^0(w_r, x_{i,j}^r)\}_{i \in S_{r}, j \in [K]}; \frac{N}{2M}\epsilon, \delta; L)$ and $w_{r+1} :=\prox_{\frac{1}{2\beta}f^1}(w_r - \frac{1}{\beta}\widetilde{g}_r)$
 \ENDFOR \\
\STATE {\bfseries Output:} $w_R$.
\end{algorithmic}
\end{algorithm}

\subsection{Proof of~\cref{thm: hetero pl fl proxgrad}: Heterogeneous FL (SO)}
\label{app: proof ofProximal PL SO}
First we provide a precise re-statement of the result, in which we assume $M_r = M$ is fixed, for convenience:
\begin{theorem}[Re-statement of~\cref{thm: hetero pl fl proxgrad}]
Grant~\cref{ass: lip},~\cref{ass: f1}, and ~\cref{ass: stochProximal PL} for $K = \lfloor \frac{n}{R} \rfloor$ and $R$ specified below.
Let $\epsilon \leq \min\{8\ln(1/\delta), 15\}, \delta \in (0,1/2)$, $n \geq \widetilde{\Omega}(\kappa)$. There is a choice of $\sigma^2$ such that:\\

1. ISRL-DP Prox-SGD is $(\epsilon, \delta)$-ISRL-DP and 
\small
\begin{equation} 
\small
\EPLL = \widetilde{\mathcal{O}}\left(\frac{L^2}{\mu}\left(\frac{\kappa^2 d \ln(1/\delta)}{\epsilon^2 n^2 M} + \frac{\kappa}{Mn}\right)\right)
\end{equation}
\normalsize
in $R = \left \lceil 2\kappa \ln\left(\frac{\mu \Delta}{L^2} \min\left\{Mn, \frac{\epsilon^2 n^2 M}{d \ln(1/\delta)}\right\}\right) \right\rceil$ communication rounds, where $\Delta \geq F(w_0) - F^*$. \\
2. SDP Prox-SGD is $(\epsilon, \delta)$-SDP for $M \geq N \min(\epsilon/2, 1)$, and 
\small
\begin{equation}
\small
\EPLL = \widetilde{\mathcal{O}}\left(\frac{L^2}{\mu}\left(\frac{\kappa^2 d \ln^2(d/\delta)}{\epsilon^2 n^2 N^2} + \frac{\kappa}{Mn}\right)\right)
\end{equation}
\normalsize
in $R = \left \lceil 2\kappa \ln\left(\frac{\mu \Delta}{L^2} \min\left\{Mn, \frac{\epsilon^2 n^2 N^2}{d \ln(1/\delta)
}\right\}\right) \right\rceil$ communication rounds. 
\end{theorem}

\begin{proof} We prove part 1 first. 
\textbf{Privacy:} First, by independence of the Gaussian noise across silos, it is enough show that transcript of silo $i$'s interactions with the server is DP for all $i \in [N]$ (conditional on the transcripts of all other silos). Since the batches sampled by silo $i$ in each round are disjoint (as we sample without replacement), the parallel composition theorem of DP \citep{mcsherry2009privacy} implies that it suffices to show that each round is $(\epsilon, \delta)$-ISRL-DP. Then by post-processing \citep{dwork2014}, we just need to show that that the noisy stochastic gradient $\widetilde{g}_r^i$ in line 6 of the algorithm is $(\epsilon, \delta)$-DP. Now, the $\ell_2$ sensitivity of this stochastic gradient is bounded by $\Delta_2 := \sup_{|X_i \Delta X'_i| \leq 2, w \in \WW} \|\frac{1}{K} \sum_{j=1}^{K} \nabla f(w, x_{i,j}) - \nabla f(w, x'_{i,j})\| \leq 2L/K,$ by $L$-Lipschitzness of $f.$ Hence~\cref{prop: gauss} implies that $\widetilde{g}_r^i$ in line 6 of the algorithm is $(\epsilon, \delta)$-DP for $\sigma^2 \geq \frac{8 L^2 \ln(1/\delta)}{\epsilon^2 K^2}$. Therefore, ISRL-DP Prox-SGD is $(\epsilon, \delta)$-ISRL-DP.

\textbf{Excess loss:}
Denote the stochastic approximation of $F$ in round $r$ by $\hf_r(w) := \frac{1}{MK} \sum_{i \in S_r} \sum_{j=1}^K f(w, x_{i,j}^r)$, and $\widebar{u}_r := \frac{1}{M} \sum_{i \in S_r} u_i \sim \mathcal{N}(0, \frac{\sigma^2}{M} \mathbf{I}_d)$. By $\beta$-smoothness, we have \begin{align}
    \expec F(w_{r+1}) &= \expec \left[F^0(w_{r+1}) + f^1(w_r) + f^1(w_{r+1}) - f^1(w_r)\right] \nonumber \\
    &\leq \expec F(w_r) + \expec\left[\langle \nabla \hf_r^0(w_r), w_{r+1} - w_r \rangle + \frac{\beta}{2}\|w_{r+1} - w_r\|^2 + f^1(w_{r+1}) - f^1(w_r) + \langle \widebar{u}_r, w_{r+1} \rangle\right] \nonumber \\
    &\;\;\; + \expec \langle \nabla F^0(w_r) - \nabla \hf_r^0(w_r), w_{r+1} - w_r \rangle - \expec \langle \widebar{u}_r, w_{r+1} \rangle \\
    &\leq \expec F(w_r) + \expec\left[\langle \nabla \widehat{F}_r^0(w_r), w_{r+1} - w_r \rangle + \beta\|w_{r+1} - w_r\|^2 + f^1(w_{r+1}) - f^1(w_r) + \langle \widebar{u}_r, w_{r+1} \rangle \right] \nonumber \\
     &\;\;\; - \langle \widebar{u}_r, w_{r+1} \rangle + \expec\left[\frac{1}{2\beta}\|\nabla F^0(w_r) - \nabla \hf^0_r(w_r)\|^2\right],
    \label{eq27}
\end{align}
where we used Young's inequality to bound \begin{align}
\expec \langle \nabla F^0(w_r) - \nabla \hf_r^0(w_r), w_{r+1} - w_r \rangle &\leq \underbrace{\expec\left[\frac{1}{2\beta}\|\nabla F^0(w_r) - \nabla \hf^0_r(w_r)\|^2\right]}_{\textcircled{a}} + \expec\left[\frac{\beta}{2}\|w_{r+1} - w_r\|^2\right]
\end{align}
in the last line above. 
We bound \textcircled{a} as follows: 
\begin{align}
\expec\left[\frac{1}{2\beta}\left\|\nabla F^0(w_r) - \nabla \hf^0_r(w_r)\right\|^2\right] &= \frac{1}{2\beta}\expec\left\|\frac{1}{MK} \sum_{i \in S_r} \sum_{j =1}^K \nabla F^0(w_r) - \nabla f^0(w_r, x_{i,j}^r)\right\|^2 \\
&= \frac{1}{2 \beta M^2 K^2}  \sum_{i \in S_r} \sum_{j =1}^K \expec \|\nabla F^0(w_r) - \nabla f^0(w_r, x_{i,j}^r)\|^2 \\
&\leq \frac{L^2}{\beta MK},
\end{align}
by independence of the data and $L$-Lipschitzness of $f^0$.

Next, we will bound $\expec\left[\langle \nabla \hf_r^0(w_r), w_{r+1} - w_r \rangle + \beta\|w_{r+1} - w_r\|^2 + f^1(w_{r+1}) - f^1(w_r) + \langle \widebar{u}_r, w_{r+1} \rangle\right]$. 
Denote 
$H^{\text{priv}}_r(y):= \langle \hf_r^0(w_r), y - w_r \rangle + \beta\|y - w_r\|^2 + f^1(y) - f^1(w_r) + \langle \widebar{u}_r, y\rangle$ and $H_r(y):=  \langle \nabla \hf_r^0(w_r), y - w_r \rangle + \beta\|y - w_r\|^2 + f^1(y) - f^1(w_r)$
Note that $H_r$ and $H^{\text{priv}}_r$ are $2\beta$-strongly convex. Denote the minimizers of these two functions by $y_*$ and $y_*^{\text{priv}}$ respectively. Now, conditional on $w_r, S_r$, and $\widebar{u}_r$, we claim that \begin{equation}
\label{claimyyy}
    H_r(y_*^{\text{priv}}) - H_r(y_*) \leq \frac{\|\widebar{u}_r\|^2}{2\beta}.
\end{equation}
To prove~\cref{claimyyy}, we will need the following lemma: 
\begin{lemma}[\citep{lowy2021}]
\label{lemmaB2outpert}
Let $H(y), h(y)$ be convex functions on some convex closed set $\mathcal{Y} \subseteq \mathbb{R}^d$ and suppose that $H(w)$ is $2\beta$-strongly convex. Assume further that $h$ is $L_{h}$-Lipschitz. Define $y_{1} = \arg\min_{y \in \mathcal{Y}} H(y)$ and $y_{2} = \arg\min_{y \in \mathcal{Y}} [H(y) + h(y)]$. Then $\|y_{1} - y_{2}\|_2 \leq \frac{L_{h}}{2\beta}.$ 
\end{lemma}
We apply~\cref{lemmaB2outpert} with $H(y):= H_r(y), ~h(y):= \langle \widebar{u}_r, y\rangle$, $L_h = \|\widebar{u}_r\|$, $y_1 = y_*$, and $y_2 = y_*^{\text{priv}}$ to get \[
\|y_* - y_*^{\text{priv}}\| \leq \frac{\|\widebar{u}_r\|}{2\beta}.\] On the other hand, \[
H_r^{\text{priv}}(y_*^{\text{priv}}) = H_r(y_*^{\text{priv}}) + \langle \widebar{u}_r, y_*^{\text{priv}} \rangle \leq H_r^{\text{priv}}(y_*) = H_r(y_*) + \langle \widebar{u}_r, y_* \rangle.
\]
Combining these two inequalities yields \begin{align}
    H_r(y_*^{\text{priv}}) - H_r(y_*) &\leq \langle \widebar{u}_r,   y_* - y_*^{\text{priv}}  \rangle \nonumber \\
    &\leq \|\widebar{u}_r\| \|y_* - y_*^{\text{priv}}\| \nonumber \\
    &\leq \frac{\|\widebar{u}_r\|^2}{\beta},
\end{align}
as claimed. Also, note that $w_{r+1} = y_*^{\text{priv}}$. Further, by~\cref{ass: stochProximal PL}, we know 
\begin{align}
    \expec H_r(y_*) &= \expec \min_y \left[\langle \nabla \hf_r^0(w_r), y - w_r \rangle + \beta\|y - w_r\|^2 + f^1(y) - f^1(w_r) \right] \\
    &\leq -\frac{\mu}{2\beta}\expec[\hf_r(w_r) - \min_w \hf_r(w)] \leq-\frac{\mu}{2\beta}[F(w_r) - F^*].
\end{align}
Combining this with~\cref{claimyyy}, we get: \begin{align*}
   &\expec\left[\langle \nabla \hf_r^0(w_r), w_{r+1} - w_r \rangle + \beta\|w_{r+1} - w_r\|^2 + f^1(w_{r+1}) - f^1(w_r) + \langle \widebar{u}_r, w_{r+1} \rangle\right] \\
   &= \expec H_r^{\text{priv}}(y_*^{\text{priv}}) \\
   &=\expec H_r(y_*^{\text{priv}}) + \expec \langle \widebar{u}_r, w_{r+1}\rangle \\
    &\leq \expec H_r(y_*) + \frac{d \sigma^2}{\beta M} + \expec\langle \widebar{u}_r, w_{r+1}\rangle \\
    &\leq -\expec\frac{\mu}{2\beta}[F(w_r) - F^*] + \frac{d \sigma^2}{\beta M} + \expec \langle \widebar{u}_r, w_{r+1}\rangle.
\end{align*}
Plugging the above bounds back into~\cref{eq27}, we obtain \begin{equation}
\expec F(w_{r+1}) \leq \expec F(w_r) - \frac{\mu}{2\beta}\expec[F(w_r) - F^*] + \frac{2d \sigma^2}{\beta M} + \frac{2L^2}{\beta MK},
\end{equation}
whence \begin{equation}
\label{eq: thingz}
\expec[F(w_{r+1}) - F^*] \leq \expec[F(w_r) - F^*]\left(1 - \frac{\mu}{2\beta}\right) + \frac{2d \sigma^2}{\beta M} + \frac{2L^2}{\beta MK}.
\end{equation}
Using~\cref{eq: thingz} recursively and plugging in $\sigma^2$, we get \begin{equation}
    \expec[F(w_R) - F^*] \leq \Delta \left(1 - \frac{\mu}{2\beta}\right)^R + \frac{L^2}{\mu}\left[\frac{16 d \ln(1/\delta)}{\epsilon^2 K^2 M} + \frac{1}{MK}\right]. 
\end{equation}
Finally, plugging in $K$ and $R$, and observing that $\frac{1}{\ln\left(\frac{\beta}{\beta - \mu}\right)} \leq \kappa$, we conclude \[
\EPLL \lesssim \frac{L^2}{\mu}\left[\ln^2\left(\frac{\mu \Delta}{L^2} \min\left\{Mn, \frac{\epsilon^2 n^2 M}{d}\right\}\right)\left(\frac{\kappa^2 d \ln(1/\delta)}{\epsilon^2 n^2 M} + \frac{\kappa}{Mn} \right) \right].
\]
Next, we move to part 2. \\
\textbf{Privacy:} Since the batches used in each iteration are disjoint by our sampling (without replacement) strategy, the parallel composition theorem \citep{mcsherry2009privacy} implies that it is enough to show that each of the $R$ rounds is $(\epsilon, \delta)$-SDP. This follows immediately from~\cref{thm:cheu vecsum} and privacy amplification by subsambling~\citep{ullman2017} (silos only): in each round, the network ``selects'' a uniformly random subset of $M_r = M$ silos out of $N$, and the shuffler executes a $(\frac{N}{2M} \epsilon, \delta)$-DP (by $L$-Lipschitzness of $f^0(\cdot, x) \forall x \in \XX$) algorithm $\mathcal{P}_{\text{vec}}$ on the data of these $M$ silos (line 8), implying that each round is $(\epsilon, \delta)$-SDP. 

\textbf{Utility:} The proof is very similar to the proof of part 1, except that the variance of the Gaussian noise $\frac{d\sigma^2}{M}$ is replaced by the variance of  $\mathcal{P}_{\text{vec}}$. Denoting $Z := \frac{1}{MK}\mathcal{P}_{\text{vec}}(\{\nabla f^0(w_r, x_{i,j}^r)\}_{i \in S_{r}, j\in [K]}; \frac{N}{2M} \epsilon, \delta) - \frac{1}{MK} \sum_{i \in S_{r+1}} \sum_{j=1}^K \nabla f^0(w_r, x_{i,j}^r)$, we have (by~\cref{thm:cheu vecsum})
\[
\expec \|Z\|^2 = \mathcal{O}\left(\frac{d L^2 \ln^2(d/\delta)}{M^2 K^2 (\frac{N}{2M} \epsilon)^2}\right) = \mathcal{O}\left(\frac{d L^2 \ln^2(d/\delta)}{\epsilon^2 K^2 N^2} \right).
\]
Also, $Z$ is independent of the data and gradients. Hence we can simply replace $\frac{d\sigma^2}{M}$ by 
$\mathcal{O}\left(\frac{d L^2 \ln^2(d/\delta)}{\epsilon^2 K^2 N^2} \right)$
and follow the same steps as the proof of~\cref{thm: hetero pl fl proxgrad}. This yields (c.f.~\cref{eq: thingz})
\begin{equation}
\label{eq: thingz2}
\expec[F(w_{r+1}) - F^*] \leq \expec[F(w_r) - F^*]\left(1 - \frac{\mu}{2\beta}\right) + \mathcal{O}\left(\frac{d L^2 \ln^2(d/\delta)}{\epsilon^2 K^2 N^2} \right) + \frac{2L^2}{\beta MK}.
\end{equation}
Using~\cref{eq: thingz2} recursively, we get \begin{equation}
    \expec[F(w_R) - F^*] \leq \Delta \left(1 - \frac{\mu}{2\beta}\right)^R + \frac{L^2}{\mu}\left[\mathcal{O}\left(\frac{d L^2 \ln^2(d/\delta)}{\epsilon^2 K^2 N^2} \right) + \frac{1}{MK}\right]. 
\end{equation}
Finally, plugging in $R$ and $K = n/R$, and observing that $\frac{1}{\ln\left(\frac{\beta}{\beta - \mu}\right)} \leq \kappa$, we conclude \[
\EPLL \lesssim \frac{L^2}{\mu}\left[\ln^2\left(\frac{\mu \Delta}{L^2} \min\left\{Mn, \frac{\epsilon^2 n^2 N^2}{d}\right\}\right)\left(\frac{\kappa^2 d \ln^2(d/\delta)}{\epsilon^2 n^2 M} + \frac{\kappa}{Mn} \right) \right].
\]
\end{proof}

\subsection{
SDP Noisy Distributed Prox-PL SVRG Pseudocode
}
\label{app: SDP proxSVRG}
Our SDP Prox-SVRG algorithm is described in~\cref{alg: SDP ProxSVRG}. 
\begin{algorithm}[ht]
\caption{SDP 
Prox-SVRG $(w_0, E, K, \eta, \epsilon, \delta)$}
\label{alg: SDP ProxSVRG}
\begin{algorithmic}[1]
\STATE {\bfseries Input:} Number of epochs $E \in \mathbb{N}$, local batch size $K \in [n]$, epoch length $Q = \lfloor \frac{n}{K} \rfloor$, data sets $X_i \in \XX^{n}$, 
loss function $f(w, x) = f^0(w,x) + f^1(w)$, step size $\eta$, privacy parameters $\epsilon, \delta$,
initial parameters $\widebar{w}_0 = w_0^Q = w_0 \in \mathbb{R}^d$; $\mathcal{P}_{\text{vec}}$ privacy parameters $\widetilde{\epsilon} := \frac{\epsilon Nn}{8 MK\sqrt{4 EQ \ln(2/\delta)}}$ and $\widetilde{\delta} := \frac{\delta}{2EQ}$.
 \FOR{$r \in \{0, 1, \cdots, E-1\}$} 
 \STATE Server updates $w^0_{r+1} = w^Q_r$.
\FOR{$i \in S_r$ \textbf{in parallel}}
\STATE Server sends global model $w_r$ to silo $i$. 
\STATE silo $i$ computes $\{\nabla f^0(\wb_r, x_{i,j})\}_{j=1}^n$. 
\STATE Server updates $\small \widetilde{g}_{r+1} := \frac{1}{M_{r+1}n}\mathcal{P}_{\text{vec}}(\{\nabla f^0(\wb_r, x_{i,j})\}_{i \in S_{r+1}, j\in [n]}; 
\widetilde{\epsilon}, \widetilde{\delta}; L)$.
\FOR{$t \in \{0,1, \cdots Q - 1\}$}
\STATE silo $i$ draws $\{x_{i,j}^{r+1, t}\}_{j=1}^K$ uniformly
from $X_i$ with replacement, and computes $\{\nabla f^0(w_{r+1}^t, x_{i,j}^{r+1, t})\}_{j=1}^K$. 
\STATE Server updates $\small \widetilde{p}_{r+1}^t := \frac{1}{M_{r+1} K}\mathcal{P}_{\text{vec}}(\{\nabla f^0(w_{r+1}^t, x_{i,j}^{r+1, t}) - \nabla f^0(\wb_{r+1}, x_{i,j}^{r+1, t})\}_{i \in S_{r+1}, j \in [K]}; \widetilde{\epsilon}, \widetilde{\delta}; 2L)$ \normalsize
\STATE Server updates $\widetilde{v}_{r+1}^t := \widetilde{p}_{r+1}^t + \widetilde{g}_{r+1}$ and $w_{r+1}^{t+1} := \prox_{\eta f^1}(w_{r+1}^t - \eta \widetilde{v}_{r+1}^t)$.
\ENDFOR 
\STATE Server updates $\wb_{r+1} := w_{r+1}^Q$.
\ENDFOR \\
\ENDFOR \\
\STATE {\bfseries Output:} $\wpr \sim \text{Unif}(\{w_{r+1}^t\}_{r=0, 1, \cdots, E-1; t=0, 1, \cdots Q-1})$.
\end{algorithmic}
\end{algorithm}
\begin{algorithm}[ht]
\caption{SDP
Prox-PL-SVRG 
}
\label{alg: SDP ProxPLSVRG}
\begin{algorithmic}[1]
 \FOR{$s \in [S]$} 
  \STATE $w_s = \texttt{SDP Prox-SVRG}(w_{s-1}, E, K, \eta, \frac{\epsilon}{2\sqrt{2S}}, \frac{\delta}{2S})$.
\ENDFOR \\
\STATE {\bfseries Output:} $\wpr := w_S$.
\end{algorithmic}
\end{algorithm}

\subsection{
Proof of~\cref{thm: ISRL-DP prox PL SVRG ERM}: Federated ERM
}
\label{app: proofs of prox pl svrg ERM}

For the precise/formal version of~\cref{thm: ISRL-DP prox PL SVRG ERM}, we will need an additional notation: the heterogeneity parameter $\upsilon_\bx^2$, which has appeared in other works on FL (e.g. \citep{woodworth2020}). Assume: 
\begin{assumption}
\label{ass: stochastic gradient variance}
$\frac{1}{N}\sum_{i=1}^N \|\nabla \hf^0_i(w) - \nabla \hf^0_{\bx}(w)\|^2 \leq \hat{\upsilon}_{\bx}^2$ 
for all $i \in [N], ~w \in \WW$.
\end{assumption}

Additionally, motivated by practical FL considerations (especially cross-device FL~\citep{kairouz2019advances}), we shall actually prove a more general result, which holds even when the number of active silos in each round is \textit{random}:
\begin{assumption}
\label{ass: unif silos}
In each round $r$, a uniformly random subset $S_r$ of $M_r \in [N]$  silos can communicate with the server, where $\{M_r\}_{r \geq 0}$ are independent with $\frac{1}{M}:= \mathbb{E}(\frac{1}{M_r})$.
\end{assumption}

We will require the following two lemmas for the proofs in this Appendix section: 
\begin{lemma}[\citep{proxSVRG}]
\label{lemma2 sv}
Let $\hf(w) = \hf^0(w) + f^1(w)$, where $\hf^0$ is $\beta$-smooth and $f^1$ is proper, closed, and convex. Let
$y := \prox_{\eta f^1}(w - \eta d')$ for some $d' \in \mathbb{R}^d$. Then for all $z \in \mathbb{R}^d$, we have: 
\begin{align*}
\hf(y) \leq \hf(z) + \langle y - z, \nabla \hf^0(w) - d' \rangle + \left[\frac{\beta}{2} - \frac{1}{2 \eta}\right]\|y - w\|^2 + \left[\frac{\beta}{2} + \frac{1}{2 \eta}\right]\|z - w\|^2 - \frac{1}{2 \eta}\|y - z\|^2. 
\end{align*}
\end{lemma}

\begin{lemma} 
\label{lemma3 svrg}
For all $t \in \{0,1, \cdots, Q-1\}$ and $r \in \{0, 1, \cdots, E-1\}$, the iterates of~\cref{alg: LDP ProxSVRG} satisfy: \[
\expec\|\nabla \hf^0(w_{r+1}^t) - \widetilde{v}_{r+1}^t\|^2 \leq \frac{8 \mathds{1}_{\{MK < Nn\}}}{MK}\beta^2 \expec\|w_{r+1}^t - \wb_r\|^2 + \frac{2 (N-M) \hat{\upsilon}_{\bx}^2}{M (N-1)}\mathds{1}_{\{N > 1\}} + \frac{d(\sigma_1^2 + \sigma_2^2)}{M}.
\]
Moreover, the iterates of~\cref{alg: SDP ProxSVRG} satisfy \[
\expec\|\nabla \hf^0(w_{r+1}^t) - \widetilde{v}_{r+1}^t\|^2 \leq \frac{8 \mathds{1}_{\{MK < Nn\}}}{MK}\beta^2 \expec\|w_{r+1}^t - \wb_r\|^2 + \frac{2 (N-M) \hat{\upsilon}_{\bx}^2}{M (N-1)}\mathds{1}_{\{N > 1\}} + \mathcal{O}\left(\frac{dL^2 R \ln^2(dR/\delta)\ln(1/\delta)}{\epsilon^2 n^2 N^2} \right),
\]
where $R = EQ$. 
\end{lemma}
\begin{proof}
We begin with the first claim (\cref{alg: LDP ProxSVRG}). Denote \begin{align*}
\zeta_{r+1}^t &:= \frac{1}{M_{r+1} K}\sum_{i \in S_{r+1}} \sum_{j=1}^K [\underbrace{\nabla f^0(w_{r+1}^t, x_{i,j}^{r+1, t}) - \nabla f^0(\widebar{w}_r, x_{i,j}^{r+1,t})}_{:= \zeta_{r+1}^{t, i, j}}] \\
&= \widetilde{v}_{r+1}^t - \widetilde{g}_{r+1} - \widebar{u}_2, 
\end{align*}
where $\widetilde{g}_{r+1} := \frac{1}{M_{r+1}}\sum_{i \in S_{r+1}} \widetilde{g}^i_{r+1} =  \frac{1}{M_{r+1}}\sum_{i \in S_{r+1}} \nabla \hf^0_i(\widebar{w}_r) + \widebar{u}_1$, and $\widebar{u}_j = \frac{1}{M_{r+1}} \sum_{i \in S_{r+1}} u^i_j$ for $j = 1,2$. Note $\expec \zeta_{r+1}^{t, i, j} = \nabla \hf^0_i(w_{r+1}^t) - \nabla \hf^0_i(\widebar{w}_r)$. Then, conditional on all iterates through $w_{r+1}^t$ and $\widebar{w}_r$, we have: 
\begin{align}
    \expec \left \|\nabla \hf^0(w_{r+1}^t) - \widetilde{v}_{r+1}^t \right\|^2 
    &= \expec \left\|
    \frac{1}{M_{r+1} K} \sum_{i \in S_{r+1}} \sum_{j=1}^K [
    \zeta_{r+1}^{t, i , j} + \widetilde{g}_{r+1}^i - \nabla \hf^0(w_{r+1}^t)] + \widebar{u}_2
    \right\|^2 \\
    &= \expec \left\|
    \frac{1}{M_{r+1} K} \sum_{i \in S_{r+1}} \sum_{j=1}^K [
    \zeta_{r+1}^{t, i , j} + \nabla \hf^0_i(\widebar{w}_r) + u_1^i - \nabla \hf^0(w_{r+1}^t)] + \widebar{u}_2
    \right\|^2 \\ 
    &= \frac{d(\sigma_1^2 + \sigma_2^2)}{M} + \underbrace{\expec \left\|
    \frac{1}{M_{r+1} K} \sum_{i \in S_{r+1}} \sum_{j=1}^K [
    \zeta_{r+1}^{t, i , j} + \nabla \hf^0_i(\widebar{w}_r) - \nabla \hf^0(w_{r+1}^t)]\right\|^2}_{:= \textcircled{a}}, 
\end{align}
by independence of the Gaussian noise and the gradients. Now, 
\begin{align}
    \textcircled{a} &= \expec\left\| \frac{1}{M_{r+1}K} \sum_{i \in S_{r+1}} \sum_{j=1}^K \left\{[\zeta_{r+1}^{t, i, j} - \expec \zeta_{r+1}^{t, i, j}] + \nabla \hf^0_i(w_{r+1}^t) - \nabla \hf^0(w_{r+1}^t) \right\}\right\|^2 \\
    &\leq 2 \expec\left\|\frac{1}{M_{r+1} K}\sum_{i \in S_{r+1}} \sum_{j=1}^K \zeta_{r+1}^{t, i, j} - \expec \zeta_{r+1}^{t, i, j} \right\|^2 + 2 \expec \left \|\frac{1}{M_{r+1} K} \sum_{i \in S_{r+1}} \sum_{j=1}^K \nabla \hf^0_i(w_{r+1}^t) - \nabla \hf^0(w_{r+1}^t) \right\|^2. 
    \label{eq: thing}
\end{align}
We bound the first term (conditional on $M_{r+1}$ and all iterates through round $r$) in~\cref{eq: thing} using~\cref{lem: lei}: 
\small
\begin{align*}
\small 
    \expec\left\|\frac{1}{M_{r+1} K}\sum_{i \in S_{r+1}} \sum_{j=1}^K \zeta_{r+1}^{t, i, j} - \expec \zeta_{r+1}^{t, i, j} \right\|^2  
    &\leq \frac{\mathds{1}_{\{M_{r+1} K < Nn\}}}{M_{r+1} K N n}\sum_{i=1}^N \sum_{j=1}^n \expec \left \|\zeta_{r+1}^{t, i, j} - \expec \zeta_{r+1}^{t, i, j} \right\|^2 \\
    &\leq \small \frac{\mathds{1}_{\{M_{r+1} K < Nn\}}}{M K N n}\sum_{i=1}^N \sum_{j=1}^n 2\expec\left[\left \| \nabla f^0(w_{r+1}^t, x_{i,j}^{r+1, t}) - f^0(\widebar{w}_r, x_{i,j}^{r+1, t})\right\|^2\right] \\
    &\;\;\; + \expec\|\nabla \hf^0(w_{r+1}^t) - \nabla \hf^0(\widebar{w}_r) \|^2 \\
    &\leq \frac{4 \mathds{1}_{\{M_{r+1} K < Nn\}}}{M_{r+1} K} \beta^2 \|w_{r+1}^t - \widebar{w}_r\|^2, 
\end{align*}
\normalsize
where we used Cauchy-Schwartz and $\beta$-smoothness in the second and third inequalities. Now if $M=N$, then $M_{r+1} = N$ (with probability $1$) and taking expectation with respect to $M_{r+1}$ (conditional on the $w$'s) bounds the left-hand side by $\frac{4 \mathds{1}_{\{K < n\}}}{M K} \beta^2 \|w_{r+1}^t - \widebar{w}_r\|^2 = \frac{4 \mathds{1}_{\{M K < Nn\}}}{M K} \beta^2 \|w_{r+1}^t - \widebar{w}_r\|^2$, via~\cref{ass: unif silos}. On the other hand, if $M < N$, then taking expectation with respect to $M_{r+1}$ (conditional on the $w$'s) bounds the left-hand-side by $\frac{4}{M K} \beta^2 \|w_{r+1}^t - \widebar{w}_r\|^2 = \frac{4 \mathds{1}_{\{M K < Nn\}}}{M K} \beta^2 \|w_{r+1}^t - \widebar{w}_r\|^2$ (since the indicator is always equal to $1$ if $M < N$). In either case, taking total expectation with respect to $\wb_r, w_{r+1}^t$ 
yields \[
\expec\left\|\frac{1}{M_{r+1} K}\sum_{i \in S_{r+1}} \sum_{j=1}^K \zeta_{r+1}^{t, i, j} - \expec \zeta_{r+1}^{t, i, j} \right\|^2  \leq \frac{4 \mathds{1}_{\{M K < Nn\}}}{M K} \beta^2 \expec \|w_{r+1}^t - \widebar{w}_r\|^2.
\]

We can again invoke~\cref{lem: lei} to bound (conditional on $M_{r+1}$ and $w_{r+1}^t$) the second term in~\cref{eq: thing}: \begin{align*}
\expec \left \|\frac{1}{M_{r+1} K} \sum_{i \in S_{r+1}} \sum_{j=1}^K \nabla \hf^0_i(w_{r+1}^t) - \nabla \hf^0(w_{r+1}^t) \right\|^2 
&\leq \mathds{1}_{\{N > 1\}}\frac{N - M_{r+1}
}{(N-1)M_{r+1}} \times \frac{1}{N} \sum_{i=1}^N \|\nabla \hf^0_i(w_{r+1}^t) - \nabla \hf^0(w_{r+1}^t) \|^2 \\
& \leq \mathds{1}_{\{N > 1\}} \frac{N - M_{r+1}}{(N-1)M_{r+1}} \hat{\upsilon}_{\bx}^2.
\end{align*}
Taking total expectation and combining the above pieces completes the proof of the first claim.  

The second claim is very similar, except that the Gaussian noise terms $\widebar{u}_1$ and $\widebar{u}_2$ get replaced by the respective noises due to $\mathcal{P}_{\text{vec}}$: $Z_1 := \frac{1}{Mn}\mathcal{P}_{\text{vec}}(\{\nabla f^0(\wb_r, x_{i,j})\}_{i \in S_{r+1}, j\in [n]}; \widetilde{\epsilon}, \widetilde{\delta}) - \frac{1}{M} \sum_{i \in S_{r+1}} \nabla \hf_i^0(\wb_r)$ and \small $\small Z_2 := \frac{1}{MK}\left[\mathcal{P}_{\text{vec}}(\{\nabla f^0(w_{r+1}^t, x_{i,j}^{r+1, t}) - \nabla f^0(\wb_{r+1}, x_{i,j}^{r+1, t})\}_{i \in S_{r+1}, j \in [K]}; 
\widetilde{\epsilon}, \widetilde{\delta}) -\sum_{i \in S_{r+1}} \sum_{j=1}^K (\nabla f^0(w_{r+1}^t, x_{i,j}^{r+1, t}) - f^0(\wb_r, x_{i,j}^{r+1, t})\right]$.\normalsize\\
By~\cref{thm:cheu vecsum}, we have 
\[
\expec \|Z_1\|^2 = \mathcal{O}\left(\frac{d L^2 \ln^2(d/\widetilde{\delta})}{M^2 n^2 \widetilde{\epsilon}^2}\right) = \mathcal{O}\left(\frac{dL^2 R \ln^2(dR/\delta)\ln(1/\delta)}{\epsilon^2 n^2 N^2} \right)
\]
and 
\[
\expec \|Z_2\|^2 = \mathcal{O}\left(\frac{d L^2 \ln^2(d/\widetilde{\delta})}{M^2 K^2 \widetilde{\epsilon}^2}\right) = \mathcal{O}\left(\frac{dL^2 R \ln^2(dR/\delta)\ln(1/\delta)}{\epsilon^2 n^2 N^2} \right).
\]
\end{proof}

Below we provide a precise re-statement of~\cref{thm: ISRL-DP prox PL SVRG ERM} 
for $M < N$, including choices of algorithmic parameter. 
\begin{theorem}[Complete Statement of~\cref{thm: ISRL-DP prox PL SVRG ERM}]
Assume $\epsilon \leq \min\{2 \ln(2/\delta), 15\}$ and let $R := EQ$. Suppose $\hf_{\bx}$ is $\mu$-PPL and grant~\cref{ass: lip}, ~\cref{ass: f1}, ~\cref{ass: unif silos}, and~\cref{ass: stochastic gradient variance}. Then:\\
1. \cref{alg: LDP ProxPLSVRG} is $(\epsilon, \delta)$-ISRL-DP if $\sigma_1^2 = \frac{256 L^2 SE \log(2/\delta) \log(5E/\delta)}{\epsilon^2 n^2}$, $\sigma_2^2 = \frac{1024 L^2 SR \log(2/\delta) \log(2.5R/\delta)}{\epsilon^2 n^2}$, and $K \geq \frac{\epsilon n}{4 \sqrt{2SR \ln(2/\delta)}}$. Further, if 
$K \geq \left(\frac{n^2}{M}\right)^{1/3}$, $R = 12 \kappa$, and $\small  S \geq \log_2\left(\frac{\hat{\Delta}_{\bx} \mu M \epsilon^2 n^2}{\kappa d L^2}\right)$, then there is $\eta$ such that $\forall \bx \in \mathbb{X}$, 
\small
\[  
\small
\ES = \widetilde{
\mathcal{O}}\left(\kappa \frac{L^2 d \ln(1/\delta)}{\mu \epsilon^2 n^2 M} + \frac{(N-M) \hat{\upsilon}_{\bx}^2}{M (N-1)}\mathds{1}_{\{N > 1\}} \right)
\]
\normalsize in $\widetilde{\mathcal{O}}(\kappa)$ communications. \\
2.~\cref{alg: SDP ProxSVRG} is $(\epsilon, \delta)$-SDP, provided $M_{r+1} = M \geq \min\left\{\frac{(\epsilon N L)^{3/4} (d\ln^3(d/\delta))^{3/8}}{n^{1/4} (\beta \hat{\Delta}_{\bx})^{3/8}}, N\right\}$ for all $r$. Further, if 
$K \geq \left(\frac{n^2}{M}\right)^{1/3}$, $R = 12 \kappa$, 
and
$\small  S \geq \log_2\left(\frac{\hat{\Delta}_{\bx} \mu \epsilon^2 N^2 n^2}{\kappa d L^2}\right)$, then there is $\eta$ such that $\forall \bx \in \mathbb{X}$,  \[\small
\ES = \widetilde{\mathcal{O}}\left(\kappa \frac{L^2 d \ln(1/\delta)}{\mu \epsilon^2 n^2 N^2} + \frac{(N-M) \hat{\upsilon}_{\bx}^2}{\mu M (N-1)}\mathds{1}_{\{N > 1\}} \right).
\]
\normalsize
\end{theorem}
\begin{proof}
\textbf{1.} \textbf{Privacy:} 
For simplicity, assume $S = 1$. It will be clear from the proof (and advanced composition of DP~\citep{dwork2014} or moments accountant~\citep{abadi16}) that the privacy guarantee holds for all $S$ due to the factor of $S$ appearing in the numerators of $\sigma_1^2$ and $\sigma_2^2$. 
Then by independence of the Gaussian noise across silos, it is enough show that transcript of silo $i$'s interactions with the server is DP for all $i \in [N]$ (conditional on the transcripts of all other silos). Further, by the post-processing property of DP, it suffices to show that all $E-1$ computations of $\widetilde{g}_{r+1}^i$ (line 7) are $(\epsilon/2, \delta/2)$-ISRL-DP and all $R = EQ$ computations of $\widetilde{v}_{r+1}^{t,i}$ (line 10) by silo $i$ (for $r \in \{0,1, \cdots, E-1\}, t \in \{0, 1, \cdots, Q-1\}$) are $(\epsilon, \delta)$-ISRL-DP. Now, by the advanced composition theorem (see Theorem 3.20 in \citep{dwork2014}), it suffices to show that: 1) each of the $E$ computations of $\widetilde{g}_{r+1}^i$ (line 7) is $(\widetilde{\epsilon}_1/2, \widetilde{\delta}_1/2)$-ISRL-DP, where $\widetilde{\epsilon}_1 = \frac{\epsilon}{2\sqrt{2E\ln(2/\delta)}}$ and $\widetilde{\delta}_1 = \frac{\delta}{2E}$; and 
2)each and $R = EQ$ computations of $\widetilde{v}_{r+1}^{t,i}$ (line 10) is $(\widetilde{\epsilon}_2/2, \widetilde{\delta}_2/2)$-ISRL-DP, where $\widetilde{\epsilon}_2 = \frac{\epsilon}{2\sqrt{2R\ln(2/\delta)}}$
and
$\widetilde{\delta}_2 = \frac{\delta}{2R}.$ 

We first show 1): The $\ell_2$ sensitivity of the (noiseless versions of) gradient evaluations in line 7 is bounded by $\Delta^{(1)}_2 := \sup_{|X_i \Delta X'_i| \leq 2, w \in \WW} \|\frac{1}{K_2} \sum_{j=1}^{n} \nabla f^0(w, x_{i,j}) - \nabla f^0(w, x'_{i,j})\| \leq 2L/n,$ by $L$-Lipschitzness of $f^0.$
Here $\WW$ denotes the constraint set if the problem is constrained (i.e. $f^1 = \iota_{\WW} + h$ for closed convex $h$); and $\WW = \mathbb{\mathbb{R}^d}$ if the problem is unconstrained. Hence the privacy guarantee of the Gaussian mechanism implies that taking $\sigma^2_1 \geq \frac{8L^2 \ln(1.25/(\widetilde{\delta}_1/2))}{(\widetilde{\epsilon}_1/2)^2 n^2} = \frac{256 L^2 E \ln(2/\delta) \ln(5E/\delta)}{\epsilon^2 n^2}$ suffices to ensure that each update in line 7 is $(\widetilde{\epsilon}_1/2, \widetilde{\delta}_1/2)$-ISRL-DP.

Now we establish 2): First, condition on the randomness due to local sampling of each local data point $x_{i,j}^{r+1, t}$ (line 9). Now, the $\ell_2$ sensitivity of the (noiseless versions of) stochastic minibatch gradient (ignoring the already private $\widetilde{g}_{r+1}^i$) in line 10 is bounded by $\Delta^{(2)}_2 := \sup_{|X_i \Delta X'_i| \leq 2, w, w' \in \WW} \|\frac{1}{K} \sum_{j=1}^{K} \nabla f^0(w, x_{i,j}) - \nabla f^0(w, x'_{i,j}) - f^0(w', x_{i,j}) + \nabla f^0(w', x'_{i,j}) \| \leq 2 \sup_{|X_i \Delta X'_i| \leq 2, w \in \WW} \|\frac{1}{K} \sum_{j=1}^{K} \nabla f^0(w, x_{i,j}) - \nabla f^0(w, x'_{i,j})\| \leq 4L/K,$ by $L$-Lipschitzness of $f^0$; $\WW$ is as defined above. 
Thus, the standard privacy guarantee of the Gaussian mechanism (Theorem A.1 in \citep{dwork2014}) implies that (conditional on the randomness due to sampling) taking $\sigma^2_1 \geq \frac{8L^2 \ln(1.25/(\widetilde{\delta}_2/2))}{(\widetilde{\epsilon}_2/2)^2 K_2^2} = \frac{32L^2 \ln(2.5/\widetilde{\delta}_2)}{\widetilde{\epsilon}_2^2 K_2^2}$ suffices to ensure that each such update is $(\widetilde{\epsilon}_2/2, \widetilde{\delta}_2/2)$-ISRL-DP. Now we invoke the randomness due to sampling: \citep{ullman2017} implies that round $r$ (in isolation) is $(\frac{2\widetilde{\epsilon}_2 K}{n}, \widetilde{\delta}_2)$-ISRL-DP. The assumption on $K$ ensures that $\epsilon' := \frac{n}{2K}\frac{\epsilon}{2 \sqrt{2 R \ln(2/\delta})} \leq 1$, so that the privacy guarantees of the Gaussian mechanism and amplification by subsampling stated above indeed hold.  Therefore, with sampling, it suffices to take $\sigma^2_1 \geq \frac{128 L^2 \ln(2.5/\widetilde{\delta}_2)}{n^2 \widetilde{\epsilon}_2^2} = \frac{1024 L^2 R \ln(5R/\delta)\ln(2/\delta)}{n^2 \epsilon^2}$ to ensure that all of the $R$ updates made in line 10 are $(\epsilon/2, \delta/2)$-ISRL-DP (for every client). Combining this with the above implies that the full algorithm is $(\epsilon, \delta)$-ISRL-DP. 

\textbf{Utility:} For our analysis, it will be useful to denote the full batch gradient update $\hat{w}_{r+1}^{t+1} := \prox_{\eta f^1}[w_{r+1}^t - \eta \nabla \hf^0(w_{r+1}^t)]]$. Fix any database $\bx \in \mathbb{X}$ (any database) and denote $\hf := \hf_{\bx}$ and $\hf^j := \hf_{\bx}^j$ for $j \in \{0, 1\}$ (for brevity of notations). 
Also, for $\alpha > 0$ and $w \in \mathbb{R}^d$ denote \[
D_{f^1}(w, \alpha) := -2\alpha \min_{y \in \mathbb{R}^d} \left[ \langle \nabla \hf^0(w), y - w \rangle + \frac{\alpha}{2}\|y - w\|^2 + f^1(y) - f^1(w) \right]
\]
Set $\eta := \frac{1}{8 \beta}\min\left(1, \frac{K^{3/2} \sqrt{M}}{n}\right)$. Then we claim \begin{equation}
\label{eq:claimy}
\frac{\beta}{2} + c_{t+1}\left(1 + \frac{n}{K}\right) \leq \frac{1}{2 \eta}
\end{equation}
for all $t \in \{0, 1, \cdots, Q-1\}$. First, if $MK = Nn$, then $c_t = c_{t+1}(2) = c_{t+2}(2)^2 = c_Q(2)^{Q-t} = 0$ since $c_Q = 0$. Next, suppose $MK < Nn$. Denote $q := \frac{K}{n}$. Then by unraveling the recursion, we get for all $t \in \{0, \cdots, Q-1\}$ that \begin{align*}
    c_t &= c_{t+1}(1 + q) + \frac{4 \eta \beta^2}{MK}\\
    &= \frac{4 \eta \beta^2}{MK}[(1+q)^{Q - t - 1} + \cdots + (1 + q)^2 + (1 + q) + 1] \\
    &= \frac{4 \eta \beta^2}{MK}\left(\frac{(1 + q)^{Q-t} - 1)}{q} \right) \\
    &\leq \frac{4 \eta \beta^2 n}{MK^2} \left(\left(1 + \frac{K}{n}\right)^{n/K} - 1\right) \\
    &\leq \frac{8 \eta \beta^2 n}{M K^2}.
\end{align*}
Then it's easy to check that with the prescribed choice of $\eta$,~\cref{eq:claimy} holds. 

Now, by~\cref{lemma2 sv} (with $w = z = w_{r+1}^t$ and $d' = \nabla \hf^0(w)$), we have \[
\hf(\hw_{r+1}^{t+1}) \leq \hf(w_{r+1}^{t}) + \left(\frac{\beta}{2} - \frac{1}{2 \eta}\right) \|\hwrt - w_{r+1}^t\|^2 - \frac{1}{2 \eta} \| \hwrt - w_{r+1}^t \|^2,
\] which implies 
\begin{equation}
\label{eq:9}
    \expec \hf(\hwrt) \leq \expec \hf(w_{r+1}^t) + \left(\frac{\beta}{2} - \frac{1}{\eta}\right) \expec \|\hwrt - w_{r+1}^t\|^2.
\end{equation}

Recall $w_{r+1}^{t+1} = \prox_{\eta f^1}(w_{r+1}^t - \eta \widetilde{v}_{r+1}^t)$. Applying~\cref{lemma2 sv} again (with $y = w_{r+1}^{t+1}, z = \hwrt, d' = \widetilde{v}_{r+1}^t, w = w_{r+1}^t$) yields \begin{align}
\small
\label{eq:11}
    \hf(w_{r+1}^{t+1}) &\leq \hf(\hwrt) + \langle w_{r+1}^{t+1} - \hwrt, \nabla \hf^0(w_{r+1}^t) - \widetilde{v}_{r+1}^t \rangle \\
    &+ \left(\frac{\beta}{2} - \frac{1}{2 \eta} \right)\|w_{r+1}^{t+1} - w_{r+1}^t\|^2 + \left(\frac{\beta}{2} + \frac{1}{2 \eta} \right)\|\hwrt - w_{r+1}^t\|^2 - \frac{1}{2 \eta}\|w_{r+1}^{t+1} - \hwrt\|^2. 
\end{align}
\normalsize
\normalsize
Further, by $\beta$-smoothness of $\hf^0$, we have: 
\begin{align}
    \hf(\hwrt) &\leq \hf^0(w_{r+1}^t) + f^1(w_{r+1}^t) + \langle \nabla \hf^0(\wrt), \hwrt - w_{r+1}^t \rangle + \frac{\beta}{2}\| \hwrt - w_{r+1}^t \|^2 + f^1(\hwrt) - f^1(w_{r+1}^t) \nonumber \\
    &\leq \hf(w_{r+1}^t) + \langle \nabla \hf^0(\wrt), \hwrt - w_{r+1}^t \rangle + \frac{1}{2 \eta}\| \hwrt - w_{r+1}^t \|^2 + f^1(\hwrt) - f^1(w_{r+1}^t) \nonumber \\
    &= \hf(w_{r+1}^t) - \frac{\eta}{2} D_{f^1}(w_{r+1}^t, \frac{1}{\eta}) \nonumber \\ 
    &\leq \hf(w_{r+1}^t) - \frac{\eta}{2} D_{f^1}(w_{r+1}^t, \beta) \nonumber \\ 
    &\leq \hf(w_{r+1}^t) - \eta \mu [\hf(\wrt) - \hf^*],
    \label{eq:34}
\end{align}
where the second inequality used $\eta \leq 1/\beta$, the third inequality used the Proximal-PL lemma (Lemma 1 in \citep{karimi2016linear}), and the last inequality used the assumption that $\hf$ satisfies the Proximal-PL inequality.

Now adding $2/3 \times$~\cref{eq:9} to $1/3 \times$~\cref{eq:34} and taking expectation gives\begin{equation}
    \expec \hf(\hwrt) \leq \expec\left[\hf(\wrt) + \frac{2}{3}\left(\frac{\beta}{2} - \frac{1}{\eta}\right)\|\hwrt - \wrt\|^2 - \frac{\eta \mu}{3}(\hf(\wrt) - \hf^*)
    \right].
    \label{eq:35}
\end{equation}
Adding~\cref{eq:35} to~\cref{eq:11} yields \begin{align}
\small
\label{eq:36}
\expec \hf(\wrtp) &\leq \expec\Bigg[
\hf(\wrt) + \left(\frac{5 \beta}{6} - \frac{1}{6 \eta}\right) \|\hwrt - \wrt\|^2 - \frac{\eta \mu}{3}(\hf(\wrt) - \hf^*) \nonumber \\ 
&\;\;\;+ \langle \wrtp - \hwrt, \nabla \hf^0(\wrt) - \widetilde{v}_{r+1}^t \rangle 
+ \left(\frac{\beta}{2} - \frac{1}{2 \eta}\right)\|\wrtp - \wrt\|^2 - \frac{1}{2 \eta}\|\wrtp - \hwrt\|^2 
\Bigg].
\end{align}
Since $\eta \leq \frac{1}{5\beta}$, Young's inequality implies \begin{align}
\expec \hf(\wrtp) &\leq \expec\Bigg[\hf(\wrt) + \left(\frac{\beta}{2} - \frac{1}{2 \eta}\right)\|\wrtp - \wrt\|^2 - \frac{\eta \mu}{3}(\hf(\wrt) - \hf^*) + \frac{\eta}{2}\|\hf(\wrt) - \widetilde{v}_{r+1}^t\|^2\Bigg] \nonumber \\
&\leq \expec\Bigg[\hf(\wrt) + \left(\frac{\beta}{2} - \frac{1}{2 \eta}\right)\|\wrtp - \wrt\|^2 - \frac{\eta \mu}{3}(\hf(\wrt) - \hf^*) + \frac{4 \eta \mathds{1}_{\{MK < Nn\}}}{MK}\beta^2 \|w_{r+1}^t - \wb_r\|^2 \nonumber \\
&\;\;\;+  \frac{\eta (N-M) \hat{\upsilon}_{\bx}^2}{M (N-1)}\mathds{1}_{\{N > 1\}} + \frac{\eta d(\sigma_1^2 + \sigma_2^2)}{2M}\Bigg]
\label{eq:38},
\end{align}
where we used~\cref{lemma3 svrg} to get the second inequality. Now, denote $\gamma_{r+1}^t := \expec[\hf(w_{r+1}^t) + c_t \|w_{r+1}^t - \wb_r\|^2]$, $c_t := c_{t+1}(1 + \frac{K}{n}) + \frac{4 \eta \mathds{1}_{\{MK < Nn\}}}{MK}\beta^2$ for $t = 0, \cdots, Q-1$, and $c_Q := 0$. Then~\cref{eq:38} is equivalent to \begin{align}
    \gamma_{r+1}^{t+1} &\leq \expec\Bigg
    [
    \hf(\wrt) + \left(\frac{\beta}{2} - \frac{1}{2 \eta}\right)\|\wrtp - \wrt\|^2 - \frac{\eta \mu}{3}(\hf(\wrt) - \hf^*) + \frac{4 \eta \mathds{1}_{\{MK < Nn\}}}{MK}\beta^2 \|w_{r+1}^t - \wb_r\|^2 \nonumber \\
&\;\;\;+  \frac{\eta (N-M) \hat{\upsilon}_{\bx}^2}{M (N-1)}\mathds{1}_{\{N > 1\}} + \frac{\eta d(\sigma_1^2 + \sigma_2^2)}{2M} + c_{t+1}\|w_{r+1}^{t+1} - \wb_r\|^2 \Bigg] \nonumber \\
&\leq \expec\Bigg
    [
    \hf(\wrt) + \left(\frac{\beta}{2} - \frac{1}{2 \eta} + c_{t+1}(1 + \frac{1}{q})\right)\|\wrtp - \wrt\|^2 - \frac{\eta \mu}{3}(\hf(\wrt) - \hf^*) \nonumber \\  &\;\;\;+ \left(\frac{4 \eta \mathds{1}_{\{MK < Nn\}}}{MK}\beta^2 + c_{t+1}(1+q)\right) \|w_{r+1}^t - \wb_r\|^2 
     \nonumber \\&\;\;\;+  
 \frac{\eta (N-M) \hat{\upsilon}_{\bx}^2}{M (N-1)}\mathds{1}_{\{N > 1\}} + \frac{\eta d(\sigma_1^2 + \sigma_2^2)}{2M} 
 \Bigg],
\end{align}
where $q:= \frac{K}{n}$ and we used Young's inequality (after expanding the square, to bound $\|\wrtp - \wb_r\|^2$) in the second inequality above. Now, applying~\cref{eq:claimy} yields \begin{align}
    \gamma_{r+1}^{t+1} &\leq \expec\Bigg
    [
    \hf(\wrt) - \frac{\eta \mu}{3}(\hf(\wrt) - \hf^*) 
    + \left(\frac{4 \eta \mathds{1}_{\{MK < Nn\}}}{MK}\beta^2 + c_{t+1}(1+q)\right) \|w_{r+1}^t - \wb_r\|^2 \nonumber \\
    &\;\;\;+\frac{\eta (N-M) \hat{\upsilon}_{\bx}^2}{M (N-1)}\mathds{1}_{\{N > 1\}} 
   + \frac{\eta d(\sigma_1^2 + \sigma_2^2)}{2M} 
   \Bigg]\nonumber\\
   &= \gamma_{r+1}^t - \frac{\eta \mu}{3}\expec(\hf(\wrt) - \hf^*)  + \frac{\eta d(\sigma_1^2 + \sigma_2^2)}{2M} 
\end{align}
Summing up, we get \begin{align*}
\expec[\hf(\wb_{r+1}) - \hf(\wb_r)] &= \sum_{t=0}^{Q-1}\gamma_{r+1}^{t+1} - \gamma_{r+1}^t = \frac{\eta \mu}{3}\sum_{t=0}^{Q-1}\expec[\hf(w_{r+1}^{t} - \hf^*] +\frac{\eta Q (N-M) \hat{\upsilon}_{\bx}^2}{M (N-1)}\mathds{1}_{\{N > 1\}} 
   \\
   &\;\;\; + \frac{\eta Q d(\sigma_1^2 + \sigma_2^2)}{2M} \\
   &\implies \frac{\eta \mu}{3}\sum_{r=0}^{E-1} \sum_{t=0}^{Q-1}\expec[\hf(w_{r+1}^{t} - \hf^*] \leq \Delta +R \eta  \left(\frac{(N-M) \hat{\upsilon}_{\bx}^2}{M (N-1)}\mathds{1}_{\{N > 1\}}
   + \frac{d(\sigma_1^2 + \sigma_2^2)}{2M}
   \right),
\end{align*}
where $\hat{\Delta} := \hf(\wb_0) - \hf^* = \hat{\Delta}_{\bx}$ and $R = EQ$. Recall $w_s := \texttt{ISRL-DP Prox-SVRG}(w_{s-1}, E, K, \eta, \sigma_1, \sigma_2)$ for $s \in [S]$.
Plugging in the prescribed $\eta$ and $\sigma_1^2, \sigma_2^2$, we get \begin{align}
\label{eq:26thing}
    \expec[\hf(w_1) - \hf^*] \leq \frac{3 \hat{\Delta} \beta}{\mu R}\left(1 + \frac{n}{K^{3/2} \sqrt{M}} \right)+ \frac{3 \hat{\upsilon}_{\bx}^2 (N-M)}{\mu M (N-1)} + \widetilde{\mathcal{O}}\left(\frac{R dL^2 \ln(1/\delta)}{\epsilon^2 n^2 M}\right).
\end{align}
Our choice of $K \geq \left(\frac{n}{\sqrt{M}}\right)^{2/3}$ implies
\begin{align}
    \expec[\hf(w_1) - \hf^*] \leq \frac{6 \hat{\Delta} \kappa }{R} + \frac{3 \hat{\upsilon}_{\bx}^2 (N-M)}{\mu M (N-1)} + \widetilde{\mathcal{O}}\left(\frac{R dL^2 \ln(1/\delta)}{\epsilon^2 n^2 M}\right).
\end{align}
Our choice of $R = 12 \kappa$ implies 
\begin{align}
\label{eq:prog}
    \expec[\hf(w_1) - \hf^*] \leq \frac{\hat{\Delta}}{2} + \frac{3 \hat{\upsilon}_{\bx}^2 (N-M)}{\mu M (N-1)} + \widetilde{\mathcal{O}}\left(\frac{\kappa dL^2 \ln(1/\delta)}{\epsilon^2 n^2 M}\right).
\end{align}
Iterating~\cref{eq:prog} $S\geq \log_2\left(\frac{\hat{\Delta}_{\bx} \mu M \epsilon^2 n^2}{\kappa d L^2}\right)$ times proves the desired excess loss bound. Note that the total number of communications is $SR = \widetilde{\mathcal{O}}(\kappa)$. \\
\textbf{2.} \textbf{Privacy:} As in \textbf{Part 1}, we shall first consider the case of $S = 1$. It suffices to show that: 1) the collection of all $E$ computations of $\widetilde{g}_{r+1}$ (line 7 of~\cref{alg: SDP ProxSVRG}) (for $r \in \{0,1, \cdots, E-1\}$) is $(\epsilon/2, \delta/2)$-DP; and 2) the collection of all $R = EQ$ computations of $\widetilde{p}_{r+1}^t$ (line 10) (for $r \in \{0,1, \cdots, E-1\}, t \in \{0, 1, \cdots, Q-1\}$) is $(\epsilon/2, \delta/2)$-DP. Further, by the advanced composition theorem (see Theorem 3.20 in \citep{dwork2014}) and the assumption on $\epsilon$, it suffices to show that: 1) each of the $E$ computations of $\widetilde{g}_{r+1}$ (line 7) is $(\epsilon'/2, \delta'/2)$-DP; and
2)each of the $R = EQ$ computations of $\widetilde{p}_{r+1}^t$ (line 10) is $(\epsilon'/2, \delta'/2)$-DP, where $\epsilon' := \frac{\epsilon}{2 \sqrt{2R \ln(2/\delta)}}$ and $\delta':= \frac{\delta}{2R}$. Now, condition on the randomness due to subsampling of silos (line 4) and local data (line 9). Then~\cref{thm:cheu vecsum} implies that each computation in line 7 and line 10 is $(\widetilde{\epsilon}, \widetilde{\delta})$-DP (with notation as defined in~\cref{alg: SDP ProxSVRG}), since the norm of each stochastic gradient (and gradient difference) is bounded by $2L$ by $L$-Lipschitzness of $f^0$. Now, invoking privacy amplification from subsampling \citep{ullman2017} and using the assumption on $M$ (and choices of $K$ and $R$) to ensure that $\widetilde{\epsilon} \leq 1$, we get that each computation in line 7 and line 10 is $(\frac{2MK}{Nn} \widetilde{\epsilon}, \widetilde{\delta})$-DP. Recalling $\widetilde{\epsilon} := \frac{\epsilon Nn}{8 MK\sqrt{4 EQ \ln(2/\delta)}}$ and $\widetilde{\delta} := \frac{\delta}{2EQ}$, we conclude that~\cref{alg: SDP ProxSVRG} is $(\epsilon, \delta)$-SDP. 
Finally, SDP follows by the advanced composition theorem~\cref{thm: advanced composition}, since~\cref{alg: SDP ProxPLSVRG} calls~\cref{alg: SDP ProxSVRG} $S$ times. 

\noindent \textbf{Excess Loss:} The proof is very similar to the proof of~\cref{thm: ISRL-DP prox PL SVRG ERM}, except that the variance of the Gaussian noises $\frac{d(\sigma^2_1 + \sigma_2^2)}{M}$ is replaced by the variance of  $\mathcal{P}_{\text{vec}}$. Denoting $Z_1 := \frac{1}{Mn}\mathcal{P}_{\text{vec}}(\{\nabla f^0(\wb_r, x_{i,j})\}_{i \in S_{r+1}, j\in [n]}; \widetilde{\epsilon}, \widetilde{\delta}) - \frac{1}{M} \sum_{i \in S_{r+1}} \nabla \hf_i^0(\wb_r)$ and 
\begin{align*}
\small Z_2 &:= \frac{1}{MK}\Bigg[\mathcal{P}_{\text{vec}}(\{\nabla f^0(w_{r+1}^t, x_{i,j}^{r+1, t}) - \nabla f^0(\wb_{r+1}, x_{i,j}^{r+1, t})\}_{i \in S_{r+1}, j \in [K]}; 
\widetilde{\epsilon}, \widetilde{\delta}) \\
&\;\;\;-\sum_{i \in S_{r+1}} \sum_{j=1}^K (\nabla f^0(w_{r+1}^t, x_{i,j}^{r+1, t}) - f^0(\wb_r, x_{i,j}^{r+1, t})\Bigg],\end{align*}\normalsize \\
we have (by~\cref{thm:cheu vecsum})
\[
\expec \|Z_1\|^2 = \mathcal{O}\left(\frac{d L^2 \ln^2(d/\widetilde{\delta})}{M^2 n^2 \widetilde{\epsilon}^2}\right) = \mathcal{O}\left(\frac{dL^2 R \ln^2(dR/\delta)\ln(1/\delta)}{\epsilon^2 n^2 N^2} \right)
\]
and 
\[
\expec \|Z_2\|^2 = \mathcal{O}\left(\frac{d L^2 \ln^2(d/\widetilde{\delta})}{M^2 K^2 \widetilde{\epsilon}^2}\right) = \mathcal{O}\left(\frac{dL^2 R \ln^2(dR/\delta)\ln(1/\delta)}{\epsilon^2 n^2 N^2} \right).
\]
Hence we can simply replace $\frac{d(\sigma^2_1 + \sigma_2^2)}{M}$ by $\mathcal{O}\left(\frac{dL^2 R \ln^2(dR/\delta)\ln(1/\delta)}{\epsilon^2 n^2 N^2} \right)$ and follow the same steps as the proof of ~\cref{thm: ISRL-DP prox PL SVRG ERM}. This yields (c.f.~\cref{eq:26thing})
\begin{align}
\expec[\hf(w_1) - \hf^*] \leq \frac{3 \hat{\Delta}_{\bx} \beta}{\mu R}\left(1 + \frac{n}{K^{3/2} \sqrt{M}} \right)+ \frac{3 \hat{\upsilon}_{\bx}^2 (N-M)}{\mu M (N-1)} + \mathcal{O}\left(\frac{dL^2 R \ln^2(dR/\delta)\ln(1/\delta)}{\epsilon^2 n^2 N^2} \right).
\end{align}
Our choice of $K \geq \left(\frac{n}{\sqrt{M}}\right)^{2/3}$ implies
\begin{align}
    \expec[\hf(w_1) - \hf^*] \leq \frac{6 \hat{\Delta}_{\bx} \kappa }{R} + \frac{3 \hat{\upsilon}_{\bx}^2 (N-M)}{\mu M (N-1)} + \mathcal{O}\left(\frac{dL^2 R \ln^2(dR/\delta)\ln(1/\delta)}{\epsilon^2 n^2 N^2} \right).
\end{align}
Our choice of $R = 12 \kappa$ implies 
\begin{align}
\label{eq:prog2}
    \expec[\hf(w_1) - \hf^*] \leq \frac{\hat{\Delta}_{\bx}}{2} + \frac{3 \hat{\upsilon}_{\bx}^2 (N-M)}{\mu M (N-1)} + \mathcal{O}\left(\frac{ \kappa dL^2 \ln^2(d \kappa/\delta)\ln(1/\delta)}{\epsilon^2 n^2 N^2} \right).
\end{align}
Iterating~\cref{eq:prog2} $S\geq \log_2\left(\frac{\hat{\Delta}_{\bx} \mu \epsilon^2 N^2 n^2}{\kappa d L^2}\right)$ times proves the desired excess loss bound. Note that the total number of communications is $SR = \widetilde{\mathcal{O}}(\kappa)$. 

\end{proof}

\section{Supplemental Material for~\cref{sec: Prox-SVRG}: Non-Convex/Non-Smooth Losses}
\label{app: proxspider}
\begin{theorem}[Complete Statement of~\cref{thm: ISRL-DP proxspider}]
Let $\epsilon \leq 2 \ln(1/\delta)$. Then, there are choices of algorithmic parameters such that ISRL-DP FedProx-SPIDER is $(\epsilon, \delta)$-ISRL-DP. Moreover, we have 
\begin{equation}
\EGMN \lesssim \left[\left(\frac{\sqrt{L \beta \hat{\Delta}_{\bx} d \ln(1/\delta)}}{\epsilon n\sqrt{M}}\right)^{4/3} + \frac{L^2 d \ln(1/\delta)}{\epsilon^2 n^2 M} + \mathds{1}_{\{M <N\}}\left(\frac{L \sqrt{\beta \gapx d \ln(1/\delta)}}{\epsilon n^{3/2} M} + \frac{L^2}{Mn} \right)\right].
\end{equation}
\end{theorem}
\begin{proof}
Choose $\eta = \frac{1}{2\beta}$, $\sigma_1^2 = \frac{16L^2 \ln(1/\delta)}{\epsilon^2 n^2}\max\left(\frac{R}{q}, 1 \right)$, $\sigma_2^2 = \frac{16 \beta^2 R \ln(1/\delta)}{\epsilon^2 n^2}$, $\hat{\sigma}_2^2 = \frac{64 L^2 R \ln(1/\delta)}{\epsilon^2 n^2}$, and $K_1 = K_2 = n$ (full batch). \\
\textbf{Privacy:} First, by independence of the Gaussian noise across silos, it is enough show that transcript of silo $i$'s interactions with the server is DP for all $i \in [N]$ (conditional on the transcripts of all other silos). Since $\epsilon \leq 2 \ln(1/\delta)$, it suffices (by~\cref{prop:bun1.3}) to show that silo $i$'s transcript is $\frac{\epsilon^2}{8 \ln(1/\delta)}$-zCDP. Then by~\cref{prop: gauss} and~\cref{lem: zCDP composition}, it suffices to bound the sensitivity of the update in line 7 of~\cref{alg: LDP proxspider} by $2L/n$ and the update in line 11 by $\frac{1}{n}\min\{2 \beta\|w_r - w_{r-1}\|, 4L\}$.The line 7 sensitivity bound holds because $\sup_{X_i \sim X'_i}\|\frac{1}{n}\sum_{j=1}^n \nabla f^0(w, x_{i,j}) - \nabla f^0(w, x'_{i,j})\| = \sup_{x, x'} \| \nabla f^0(w, x) - \nabla f^0(w, x') \| \leq 2L/n$ for any $w$ since $f^0$ is $L$-Lipschitz. The line 11 sensitivity bound holds because $\sup_{X_i \sim X'_i}\|\frac{1}{n}\sum_{j=1}^n \nabla f^0(w_r, x_{i,j} - \nabla f^0(w_{r-1}, x_{i,j}) - (f^0(w_r, x'_{i,j}) - \nabla f^0(w_{r-1}, x'_{i,j}))\| = \frac{1}{n}\sup_{x, x'} \| \nabla f^0(w_r, x - \nabla f^0(w_{r-1}, x) - (f^0(w_r, x') - \nabla f^0(w_{r-1}, x'))\| \leq \frac{1}{n}\min\{2\beta\|w_r - w_{r-1}\|,  4L\}$ since $f^0$ is $L$-Lipschitz and $\beta$-smooth.  Note that if $R < q$, then only one update in line 7 is made, and the privacy of this update follows simply from the guarantee of the Gaussian mechanism and the sensitivity bound, without needing to appeal to the composition theorem. \\
\textbf{Utility:} Fix any $\bx \in \mathbb{X}$ and denote $\GGh(w) = \GGh(w, \bx)$ for brevity of notation. Recall the notation of~\cref{alg: LDP proxspider}. 
Note that~\cref{lem1} holds with  
\begin{align*}
\tau_1^2 &= \sup_{r \equiv 0~(\text{mod}~q)}\expec\left\|h_r - \nabla \hf_\bx^0(w_r)\right\|^2 \\
&= \sup_{r \equiv 0~(\text{mod}~q)}\expec\left\|\frac{1}{M_r n} \sum_{i \in S_r} \sum_{j=1}^n \left[\nabla f^0(w_r, x_{i,j}) - \nabla \hfx(w_r)\right]\right\|^2 + \frac{d \sigma_1^2}{M} \\
&\leq \frac{2L^2}{Mn}\mathds{1}_{\{M < N\}} + \frac{d \sigma_1^2}{M},
\end{align*}
using independence of the noises across silos and~\cref{lem: lei}. 
Further, for any $r$, we have (conditional on $w_r, w_{r-1}$)
\begin{align*}
&\expec\left\|H_r - \nabla \hf_\bx^0(w_r)\right\|^2 \\
&\leq 2\left[\frac{d\sigma_2^2}{M}\|w_r - w_{r-1}\|^2 + \expec\left\|\frac{1}{M_r n} \sum_{i \in S_r} \sum_{j=1}^n \left[\nabla f^0(w_r, x_{i,j}) - \nabla f^0(w_{r-1}, x_{i,j}) - \left(\nabla \hfx(w_r) - \hfx(w_{r-1}) \right)\right]\right\|^2\right]\\
&\leq \frac{2d\sigma_2^2}{M}\|w_r - w_{r-1}\|^2 + \frac{8 \beta^2}{Mn}\|w_r - w_{r-1}\|^2\mathds{1}_{\{M < N\}},
\end{align*}
using Young's inequality, independence of the noises across silos, and~\cref{lem: lei}. 
Therefore,~\cref{lem1} holds with $\tau_2^2 = 8\left(\frac{\beta^2}{Mn}\one + \frac{d \sigma_2^2}{M} \right)$. Next, we claim that if $\eta = 1/2\beta$ and $q \leq \frac{1}{\eta^2 \tau_2^2}$, then  
\begin{equation}
\label{eq: uing}
    \expec\|\GG_\eta(\wpr)\|^2 \leq 16\left( \frac{\hat{\Delta}_{\bx}}{\eta R} + \tau_1^2 \right). 
\end{equation}
We prove~\cref{eq: uing} as follows. Let $g(w_r) := -\frac{1}{\eta}(w_{r+1} - w_r)$. By~\cref{lemma2 sv} (with $y = w_{r+1}, ~z = w = w_r, ~d' = h_r$), we have  \begin{align*}
\expec \hf_{\bx}(w_{r+1}) &\leq \expec \hf_{\bx}(w_r) + \expec \left \langle w_{r+1} - w_r, \nabla \hfx(w_r) - h_r \right \rangle + \left(\frac{\beta}{2} - \frac{1}{2\eta}\right)\expec\|w_{r+1} - w_r\|^2 - \frac{1}{2\eta}\expec\|w_{r+1} - w_r\|^2 \\
&\leq \expec \hf_{\bx}(w_r) + \frac{\eta}{2}\expec\left\|\nabla \hfx(w_r) - h_r\right\|^2 + 
\left(\frac{\beta}{2} - \frac{1}{2\eta}\right) \expec\|w_{r+1} - w_r\|^2\\
&= \expec \hf_{\bx}(w_r) + \frac{\eta}{2}\expec\left\|\nabla \hfx(w_r) - h_r\right\|^2 + 
\left(\frac{\beta}{2} - \frac{1}{2\eta}\right)\eta^2 \expec\|g(w_r)\|^2.
\end{align*}
Thus, by~\cref{lem1}, we have \begin{align*}
    \expec[\hf_{\bx}(w_{r+1}) - \hf_{\bx}(w_r)] &\leq \frac{\eta}{2}\expec\left\|\nabla \hfx(w_r) - h_r\right\|^2 + 
\left(\frac{\beta}{2} - \frac{1}{2\eta}\right)\eta^2 \expec\|g(w_r)\|^2 \\
&\leq \frac{\eta}{2} \tau_2^2 \sum_{t = s_r + 1}^r \expec[\|w_t - w_{t-1}\|^2] + \frac{\eta}{2}\tau_1^2 + \left(\frac{\beta}{2} - \frac{1}{2\eta}\right)\eta^2 \expec\|g(w_r)\|^2 \\
&=\frac{\eta^3}{2} \tau_2^2 \sum_{t = s_r + 1}^r \expec[\|g(w_t)\|^2] + \frac{\eta}{2}\tau_1^2 + \left(\frac{\beta}{2} - \frac{1}{2\eta}\right)\eta^2 \expec\|g(w_r)\|^2,
\end{align*}
where $s_r = \lfloor \frac{r}{q} \rfloor q$. 
Now we sum over a given phase (from $s_r$ to $r$), noting that $r - q \leq s_r \leq r$:
\begin{align*}
\expec[\hf_{\bx}(w_{r+1}) - \hf_{\bx}(w_{s_r})] &\leq \frac{\eta^3 \tau_2^2}{2} \sum_{k = s_r}^r \sum_{j = s_r + 1}^k \expec[\|g(w_j)\|^2] + \sum_{k=s_r}^r \left[\frac{\eta}{2}\tau_1^2 + \left(\frac{\beta}{2} - \frac{1}{2\eta}\right)\eta^2 \expec\|g(w_k)\|^2\right] \\
&\leq \frac{q \eta^3 \tau_2^2}{2} \sum_{k = s_r}^r \expec[\|g(w_k)\|^2] + \sum_{k=s_r}^r \left[\frac{\eta}{2}\tau_1^2 + \left(\frac{\beta}{2} - \frac{1}{2\eta}\right)\eta^2 \expec\|g(w_k)\|^2\right] \\
&=-\sum_{k = s_r}^r \left\{\expec[\|g(w_k)\|^2]\left(\frac{\eta}{2} - \frac{\beta \eta^2}{2} - \frac{\eta^3 \tau_2^2 q}{2}\right) - \frac{\eta \tau_1^2}{2}\right\}
\end{align*}
Denoting $A = \frac{\eta}{2} - \frac{\beta \eta^2}{2} - \frac{\eta^3 \tau_2^2 q}{2}$ and summing over all phases $P = \{p_0, p_1, \ldots  \} = \left\{0, q, \ldots, \lfloor \frac{R-1}{q} \rfloor q, R\right\}$, we get  \begin{align*}
    \expec[\hf_\bx(w_R) - \hf_\bx(w_0)] &\leq \sum_{j=1}^{|P|} \expec[\hf_\bx(w_{p_j}) - \hf_\bx(w_{p_{j-1}})] \\
    &\leq \frac{\eta R \tau_1^2}{2} -A\sum_{r=0}^R \expec[\|g(w_r)\|^2],
\end{align*}
which implies \begin{equation}
    \label{eq:ling}
    \frac{1}{R}\sum_{r=0}^R \expec[\|g(w_r)\|^2] \leq \frac{\gapx}{R A} + \frac{\eta \tau_1^2}{2A}.
\end{equation}
Now, for any $r \geq 0$, \begin{align*}
\left\|\GGh(w_r) - g(w_r)\right\|^2 &= \frac{1}{\eta^2}\left\|w_{r+1} - \prox_{\eta f^1}(w_r - \eta \hfx(w_r))\right\|^2\\
&= \frac{1}{\eta^2}\left\|\prox_{\eta f^1}(w_r - \eta h_r) - \prox_{\eta f^1}(w_r - \eta \hfx(w_r))\right\|^2\\
&\leq \frac{1}{\eta^2}\left\|-\eta h_r + \eta \hfx(w_r)\right\|^2 \\
&= \left\|h_r - \hfx(w_r) \right\|^2,
\end{align*}
by non-expansiveness of the proximal operator. 
Furthermore, conditional on the uniformly drawn $r = r^* \in \{0, 1, \ldots, R\}$, we have
\begin{align*}
\expec \left\|\GGh(w_{r^*}) - g(w_{r^*})\right\|^2 &\leq \expec \left\|h_{r^*} - \hfx(w_{r^*}) \right\|^2 \\
&\leq \tau_2^2 \sum_{k=s_{r^*} + 1}^{r^*} \expec\|w_k - w_{k-1}\|^2 + \tau_1^2 \\
&= \eta^2 \tau_2^2  \sum_{k=s_{r^*} + 1}^{r^*} \expec\|g(w_{k-1})\|^2 + \tau_1^2,
\end{align*}
by~\cref{lem1}, 
and taking total expectation yields \begin{align*}
  \expec \left\|\GGh(\wpr) - g(\wpr)\right\|^2 &\leq \frac{\eta^2 \tau_2^2}{R} \sum_{r=1}^R \sum_{k = s_r + 1}^r \expec\|g(w_{r-1})\|^2 + \tau_1^2 \\
  &\leq \frac{q \eta^2 \tau_2^2}{R} \sum_{r=1}^R \expec\|g(w_{r-1})\|^2 + \tau_1^2 \\
  &\leq q \eta^2 \tau_2^2\left[\frac{\gapx}{RA} + \frac{\eta \tau_1^2}{2A}\right] + \tau_1^2,
\end{align*}
where the last inequality follows from~\cref{eq:ling}. Hence \begin{align*}
    \expec\|\GGh(\wpr)\|^2 &\leq 2\left[q \eta^2 \tau_2^2\left[\frac{\gapx}{RA} + \frac{\eta \tau_1^2}{2A}\right] + \tau_1^2 \right] + 2\expec\|g(\wpr)\|^2\\
    &\leq 2\left[q \eta^2 \tau_2^2\left[\frac{\gapx}{RA} + \frac{\eta \tau_1^2}{2A}\right] + \tau_1^2 \right] + \frac{2 \gapx}{RA} + \frac{\eta \tau_1^2}{A},
\end{align*}
by Young's inequality and~\cref{eq:ling}. Now, our choices of $\eta = 1/2\beta$ and $q \leq \frac{1}{\tau_2^2 \eta^2}$ imply $A = \frac{\eta}{2} - \frac{\beta \eta^2}{2} - \frac{\eta^3 \tau_2^2 q}{2} \geq \frac{\eta}{4}$ and \begin{align*}
  \expec\|\GGh(\wpr)\|^2 &\leq 8\left[\left(\frac{\gapx}{R \eta} + \frac{\tau_1^2}{2}\right) + \tau_1^2 \right] + \frac{8 \gapx}{R \eta} + 4\tau_1^2 \\
  &= \frac{16 \gapx}{R \eta} + 16 \tau_1^2,
\end{align*}
proving~\cref{eq: uing}. The rest of the proof follows from plugging in $\tau_1^2$ and setting algorithmic parameters. Plugging $\tau_1^2 = \frac{2L^2}{Mn}\mathds{1}_{\{M < N\}} + \frac{d \sigma_1^2}{M} \leq \frac{2L^2}{Mn}\mathds{1}_{\{M < N\}} + \frac{16 d L^2 R \ln(1/\delta)}{q \epsilon^2 n^2 M} + \frac{16 d L^2 \ln(1/\delta)}{ \epsilon^2 n^2 M}$ into~\cref{eq: uing} yields \[
\expec\|\GGh(\wpr)\|^2 \leq 16\left(\frac{\hat{\Delta}_{\bx}}{\eta R} + \frac{2L^2}{Mn}\mathds{1}_{\{M < N\}} + \frac{16 d L^2 R \ln(1/\delta)}{q \epsilon^2 n^2 M} + \frac{16 d L^2 \ln(1/\delta)}{ \epsilon^2 n^2 M}\right).
\]
Choosing $R = \frac{\epsilon n \sqrt{Mq} \sqrt{\gapx \beta}}{L \sqrt{d \ln(1/\delta)}}$ equalizes the two terms in the above display involving $R$ (up to constants) and we get \begin{equation}
\label{eq:qing}
\expec\|\GGh(\wpr)\|^2 \leq C\left(\frac{L \sqrt{\gapx \beta} \sqrt{d \ln(1/\delta)}}{\epsilon n \sqrt{Mq}} + \frac{d L^2 \ln(1/\delta)}{ \epsilon^2 n^2 M} + \frac{L^2}{Mn}\mathds{1}_{\{M < N\}}\right)
\end{equation}
for some absolute constant $C > 0$. Further, with this choice of $R$, it suffices to choose \[
q = \left\lfloor \min\left\{\left(\frac{\epsilon n L \sqrt{M}}{\sqrt{d \ln(1/\delta) \gapx \beta}} \right)^{2/3}, \frac{nM}{\one}\right\} \right\rfloor
\] to ensure that $q \leq \frac{1}{\tau_2^2 \eta^2}$, so that~\cref{eq: uing} holds. Assume $q \geq 1$. Then plugging this $q$ into $\cref{eq:qing}$ yields \[
\expec\|\GGh(\wpr)\|^2 \leq C' \left[\left(\frac{\sqrt{L \beta \hat{\Delta}_{\bx} d \ln(1/\delta)}}{\epsilon n\sqrt{M}}\right)^{4/3} + \frac{d L^2 \ln(1/\delta)}{\epsilon^2 n^2 M} + \left(\frac{L \sqrt{\beta \gapx d \ln(1/\delta)}}{\epsilon n^{3/2} M} + \frac{L^2}{Mn} \right)\mathds{1}_{\{M <N\}}\right]
\]
for some absolute constant $C' > 0$, as desired. In case $q < 1$, then we must have $L < \frac{\sqrt{\beta \gapx d \ln(1/\delta)}}{\epsilon n \sqrt{M}}$; hence, we can simply output $w_0$ (which is clearly ISRL-DP) instead of running ~\cref{alg: LDP proxspider} and get $\EGMN \leq L^2 < \left(\frac{\sqrt{L \beta \hat{\Delta}_{\bx} d \ln(1/\delta)}}{\epsilon n\sqrt{M}}\right)^{4/3}$. 
\end{proof}

The lemmas used in the above proof are stated below. The following lemma is an immediate consequence of the martingale variance bound for SPIDER, given in~\citep[Proposition 1]{fangspider}:
\begin{lemma}[\citep{fangspider}]
\label{lem1}
Let $r \in \{0, 1, \ldots, R\}$ and $s_r = \lfloor \frac{r}{q} \rfloor q$.
With the notation of~\cref{alg: LDP proxspider}, assume that $\expec|h_{s_r} - \nabla \hf_\bx^0(w_{s_r})\|^2 \leq \tau_1^2$ and $\expec\left\|H_r - \left( \nabla \hf_\bx^0(w_{r}) -  \nabla \hf_\bx^0(w_{r-1})\right)\right\|^2 \leq \tau_2^2 \|w_r - w_{r-1}\|^2$. Then  for all $r \geq s_r + 1$, the iterates of~\cref{alg: LDP proxspider} satisfy:
\[
\expec\|h_r - \nabla \hf^0_\bx(w_r)\|^2 \leq \tau_2^2 \sum_{t=s_r + 1}^{r} \expec\|w_t - w_{t-1}\|^2 + \tau_1^2. 
\]
\end{lemma}

\begin{lemma}[\citep{lei17}]
\label{lem: lei}
Let $\{a_l\}_{l \in [\widetilde{N}]}$ be an arbitrary collection of vectors such that $\sum_{l=1}^{\widetilde{N}} a_l = 0$. Further, let $\mathcal{S}$ be a uniformly random subset of $[\widetilde{N}]$ of size $\widetilde{M}$. Then,\[
\mathbb{E}\left\|\frac{1}{\widetilde{M}} \sum_{l \in \mathcal{S}} a_l \right\|^2 = \frac{\widetilde{N} - \widetilde{M}}{(\widetilde{N} - 1) \widetilde{M}} \frac{1}{\widetilde{N}}\sum_{l=1}^{\widetilde{N}} \|a_l\|^2 \leq \frac{\mathds{1}_{\{\widetilde{M} < ~\widetilde{N}\}}}{\widetilde{M}~\widetilde{N}}\sum_{l=1}^{\widetilde{N}}\|a_l\|^2.
\]
\end{lemma}

We present SDP FedProx-SPIDER in~\cref{alg: SDP proxspider}. 

\begin{algorithm}[ht]
\caption{SDP FedProx-SPIDER}
\label{alg: SDP proxspider}
\begin{algorithmic}[1]
\STATE {\bfseries Input:} 
$R \in \mathbb{N}, K_1, K_2 \in [n], \bx \in \mathbb{X}, \eta > 0, \epsilon > 0, \delta \in (0, 1/2), q \in \mathbb{N}, w_0 \in \WW$.
 \FOR{$r \in \{0, 1, \cdots, R\}$} 
\FOR{$i \in S_r$ \textbf{in parallel}}
\STATE Server sends global model $w_r$ to silo $i$. 
\IF{$r \equiv 0~(\text{mod} ~q)$}
\STATE silo $i$ draws $K_1$ samples $\{x_{i,j}^r\}_{j=1}^{K_1}$ u.a.r. from $X_i$ (with replacement). 
\STATE silo $i$ computes $\left\{\nabla f^0(w_r, x_{i,j}^r)\right\}_{j=1}^{K_1}$. 
\STATE Server updates $h_r = \frac{1}{M K_1} \pvec\left(\left\{\nabla f^0(w_r, x_{i,j}^r)\right\}_{i \in S_r, j \in [K_1]}; \frac{\epsilon nN }{4K_1 M \sqrt{2 \ln(1/\delta)}\max\left(1, \frac{\sqrt{q}}{\sqrt{R}} \right)}, \frac{\delta q}{2R}; L\right).$
\ELSE 
\STATE silo $i$ draws $K_2$ samples $\{x_{i,j}^r\}_{j=1}^{K_1}$ u.a.r. from $X_i$ (with replacement).
\STATE silo $i$ computes $J_i = \{\nabla f^0(w_{r}, x_{i,j}^{r}) - \nabla f^0(w_{r-1}, x_{i,j}^{r})\}_{j=1}^{K_2}$.
\STATE Server receives $H_r = \frac{1}{M K_2}\pvec\Big(
\{J_i\}_{i \in S_r}
; \frac{\epsilon N n}{4 M K_2 \sqrt{2R\ln(1/\delta)}}; \frac{\delta}{2R}; \min\{2L, \beta \|w_r - w_{r-1}\|\}\Big)$, and updates $h_r = h_{r-1} + H_r$.  
\ENDIF
\ENDFOR 
\STATE Server updates $w_{r+1} = \prox_{\eta f^1}(w_r - \eta h_r)$.
\ENDFOR \\
\STATE {\bfseries Output:} $\wpr \sim \text{Unif}(\{w_{r}\}_{r=1, \cdots, R})$.
\end{algorithmic}
\end{algorithm}

\begin{theorem}[Complete Statement of~\cref{thm: sdp proxspider}]
Let $\epsilon \leq \ln(1/\delta),
~\delta \in (0, \frac{1}{2})$, and
\[
M_r = M \geq \left(\frac{\epsilon N d^2}{n^2}\right)^{1/3}\left(\frac{L}{\sqrt{\beta \gapx}}\right)^{1/3}\left[1 + \left(\frac{L}{\sqrt{\beta \gapx}}\right)^{1/3}\right].
\]
Then, there exist algorithmic parameters such that SDP FedProx-SPIDER is $(\epsilon, \delta)$-SDP. Further, 
\small
\begin{equation*}
\small
\EGMN \lesssim \left[\left(\frac{\sqrt{L \beta \hat{\Delta}_{\bx} d \ln^3(dnN/\delta)}}{\epsilon n N}\right)^{4/3} + \frac{d L^2 \ln^3(Rd/q\delta)}{\epsilon^2 n^2 N^2} + \mathds{1}_{\{M <N\}}\left(\frac{L \sqrt{\beta \gapx d \ln^3(dnN/\delta)}}{\epsilon n^{3/2} N \sqrt{M}} + \frac{L^2}{Mn} \right)\right].
\end{equation*}
\normalsize
\end{theorem}
\begin{proof}
We will choose \[
R = \left \lceil \frac{\epsilon n N}{L}\sqrt{\frac{\gapx \beta}{d \ln^3(dnN/\delta)}}\min\left\{\frac{\sqrt{Mn}}{\one}, \left(\frac{\epsilon n N L}{\sqrt{\gapx \beta d \ln^3(dnN/\delta)}} \right)\right\} \right \rceil,
\]
$\eta = 1/2\beta$, and $K_1 = K_2 = n$. \\
\noindent \textbf{Privacy:} By~\cref{thm: advanced composition}, it suffices to show that the message received by the server in each update in line 12 of~\cref{alg: SDP proxspider} (in isolation) is $\left(\frac{\epsilon}{2 \sqrt{2 R \ln(1/\delta)}}, \frac{\delta}{2R}\right)$-DP, and that each update in line 8 is $\left(\frac{\epsilon \sqrt{q}}{2 \sqrt{2 R \ln(1/\delta)}}, \frac{\delta q}{2R}\right)$-DP. 
Conditional on the random subsampling of silos,~\cref{thm:cheu vecsum} (together with the sensitivity estimates established in the proof of~\cref{thm: ISRL-DP proxspider}) implies that each update in line 12 is $(\epsilon', \delta')$-SDP, where $\epsilon' \leq \frac{\epsilon N}{4M \sqrt{2R\ln(1/\delta)}}$ and $\delta' = \frac{\delta}{2R}$; each update in line 8 is $(\epsilon'', \delta'')$-SDP, where $\epsilon'' = \epsilon' \sqrt{q}$ and $\delta'' = \delta' \sqrt{q}$. By our choice of $R$ and our assumption on $M$, we have $M \geq \frac{\epsilon N}{4 \sqrt{2 R \ln(1/\delta)}}$ and hence $\epsilon' \leq 1$. Thus, privacy amplification by subsampling (silos only) (see e.g. \citep[Problem 1]{ullman2017}) implies that the privacy loss of each round is bounded as desired, establishing that \cref{alg: SDP proxspider} is $(\epsilon, \delta)$-SDP, as long as $q \leq R$. If instead $q > R$, then the update in line 8 is only executed once (at iteration $r = 0$), so our choice of $\sigma_1^2$ ensures SDP simply by~\cref{thm:cheu vecsum} and privacy amplification by subsampling. 

\noindent \textbf{Utility:} 
Denote the (normalized) privacy noises induced by $\pvec$ in lines 8 and 12 of the algorithm by $Z_1$ and $Z_2$ respectively. By~\cref{thm:cheu vecsum}, $Z_i$ is an unbiased estimator of its respective mean and we have \[
\expec\|Z_1\|^2 \lesssim \frac{d L^2 \ln^3(Rd/q\delta)}{\epsilon^2 n^2 N^2}\max\left(\frac{R}{q}, 1 \right),
\]
and \[
\expec\|Z_2\|^2 \lesssim \frac{dR \ln^3(dR/\delta)}{\epsilon^2 n^2 N^2}\beta^2 \|w_r - w_{r-1}\|^2
\] for the $r$-th round. 
Also, note that~\cref{lem1} is satisfied with \[
\tau_1^2 = \frac{2L^2}{Mn}\one + \frac{d L^2 \ln^3(Rd/q\delta)}{\epsilon^2 n^2 N^2}\max\left(\frac{R}{q}, 1 \right),
\]
and \[
\tau_2^2 = 8\beta^2\left(\frac{\one}{Mn} + \frac{d R \ln^3(Rd/\delta)}{\epsilon^2 n^2 N^2} \right).
\]
Then by the proof of~\cref{thm: ISRL-DP proxspider}, we have \begin{equation}
\label{eq: ving}
    \expec\|\GG_\eta(\wpr)\|^2 \leq 16\left( \frac{\hat{\Delta}_{\bx}}{\eta R} + \tau_1^2 \right). 
\end{equation}
if $\eta = 1/2\beta$ and $q \leq \frac{1}{\eta^2 \tau_2^2}$. 
Thus, \[
\EGMN \lesssim \frac{\gapx}{\eta R} + \frac{L^2}{Mn}\one +\frac{d L^2 \ln^3(Rd/q\delta)}{\epsilon^2 n^2 N^2}\max\left(\frac{R}{q}, 1 \right).
\]
Our choice of $R$ together with the choice of \[
q = \left \lfloor \frac{1}{2} \min\left(\frac{Mn}{\one}, \left(\frac{\epsilon n N L}{\sqrt{\gapx \beta d \ln^3(Rd/\delta)}} \right) \right)    \right \rfloor
\]
equalizes the two terms involving $R$ (up to constants), and we obtain the desired ERM bound (upon noting that $q \leq 1/(\eta^2 \tau_2^2)$ is satisfied). 
\end{proof}

\subsection{ISRL-DP Lower Bound}
\label{app: lower bounds}
We first provide a couple of definitions. Our lower bound will hold for all \textit{non-interactive} and \textit{sequentially interactive}~\citep{duchi13,joseph2019} algorithms, as well as a broad subclass of \textit{fully interactive}\footnote{Full interactivity is the most permissive notion of interactivity, allowing for algorithms to query silos multiple times, adaptively, simultaneously, and in any sequence~\citep{joseph2019}. Sequentially interactive algorithms can only query each silo once, adaptively in sequence. Non-interactive algorithms query each silo once independently/non-adaptively.} ISRL-DP algorithms that are \textit{compositional}~\citep{joseph2019, lr21fl}: 
\begin{definition}[Compositionality]
\label{def: compositional}
Let $\mathcal{A}$ be an $R$-round $(\epsilon_0, \delta_0)$-ISRL-DP FL algorithm with data domain $\XX$. Let $\{(\epsilon_0^r, \delta_0^r)\}_{r =1}^R$ denote the minimal (non-negative) parameters of the local randomizers $\mathcal{R}^{(i)}_r$ selected at round $r$  such that $\mathcal{R}^{(i)}_r(\mathbf{Z}_{(1:r-1)}, \cdot)$ is $(\epsilon_0^r, \delta_0^r)$-DP for all $i \in [N]$ and all $\mathbf{Z}_{(1:r-1)}.$ For 
$C > 0$, we say that $\mathcal{A}$ is \textit{$C$-compositional} if $\sqrt{\sum_{r \in [R]} (\epsilon_0^r)^2} \leq C \epsilon_0.$ If such $C$ is an absolute constant, we simply say $\Al$ is \textit{compositional}. 
\end{definition}
Any algorithm that uses the composition theorems of~\citep{dwork2014, kairouz15} 
for its privacy analysis is $1$-compositional; this includes~\cref{alg: LDP proxspider} and most (but not all~\citep{lr21fl}) ISRL-DP algorithms in the literature. Define the $(\epsilon, \delta)$-ISRL-DP algorithm class $\mathbb{A}_{(\epsilon, \delta), C}$  to contain all \textit{sequentially interactive} algorithms and all \textit{fully interactive}, \textit{$C$-compositional} algorithms. If $\Al$ is sequentially interactive or $O\mathcal{O}(1)$-compositional, denote $\Al \in \mathbb{A}$. 

Next we re-state the precise form of our lower bound (using notation from~\cref{app: ISRL-DP}) and then provide the proof. 
\begin{theorem}[Precise Statement of~\cref{thm: LDP lower bound}]
Let 
~$\epsilon \in (0, \sqrt{N}], 2^{-\Omega(nN)} \leq \delta \leq 1/(nN)^{1 + \Omega(1)}$. Suppose that in each round~$r$, the local randomizers  are all $(\epsor, \delor)$-DP, for $\epsor \lesssim \frac{1}{n},~\delor = o(1/nNR)$, $M = N \geq 16\ln(2/\delor n)$. Then, there exists an $L$-Lispchitz, $\beta$-smooth
smooth, convex loss $f: \mathbb{R}^d \times \XX \to \mathbb{R}$ and 
a database $\bx \in \XX^{n \times N}$ such that any compositional and symmetric $(\epso, \delo)$-ISRL-DP algorithm $\Al$ run on $\bx$ with output $\wpr$ satisfies \[
\expec\|\nabla \hf_{\bx}(\wpr)\|^2 = \Omega\left(L^2\min\left\{1, \frac{d \ln(1/\delo)}{\epso^2 n^2 N}\right\}\right).
\]
\end{theorem}

\begin{proof}
The work of~\citep{lr21fl} showed that a compositional $(\epso, \delo)$-ISRL-DP algorithm can become an $\left(\mathcal{O}\left(\frac{\epso}{\sqrt{N}}\right), \delta \right)$-SDP algorithm when a shuffler is introduced: 
\begin{theorem}[\citep{lr21fl}]
\label{thm: R round shuffling amp}
Let $\Al \in \mathbb{A}_{(\epso, \delo), C}$
such that $\epso \in (0, \sqrt{N}]$ and $\delo \in (0,1).$ 
Assume that in each round, the local randomizers $\rand_r(\bz_{(1: r-1)}, \cdot): \XX^n \to \ZZ$ are $(\epsor, \delor)$-DP for all $i \in [N], ~r \in [R], ~\bz_{(1:r-1)} \in \ZZ^{r-1 \times N}$ with
$\epsor \leq \frac{1}{n}$. Assume $N \geq 16\ln(2/\delor n)$. If $\Al$ is $C$-compositional, then assume
~$\delor \leq \frac{1}{14nNR}$ and denote $\delta := 14Nn \sum_{r=1}^R \delor$; if instead $\Al$ is sequentially interactive, then assume $\delo = \delor \leq \frac{1}{7Nn}$ and denote $\delta := 7Nn\delo.$ Let $\Al_s: \mathbb{X} \to \WW$ be the same algorithm as $\Al$ except that in each round $r$, $\Al_s$ draws a random permutation $\pi_r$ of $[N]$ and applies $\rand_r$ to $X_{\pi_r(i)}$ instead of $X_i$. 
Then, $\Al_s$ is $(\epsilon, \delta)$-CDP, where 
$\epsilon = 
\mathcal{O}\left(
\frac{\epso \ln\left(1/nN \delo^{\min}\right) C^2}{\sqrt{N}}
\right),$ 
and $\delo^{\min} := \min_{r \in [R]} \delor$. In particular, if $\Al \in \mathbb{A}$, then $\epsilon = 
\mathcal{O}\left(
\frac{\epso \ln\left(1/nN \delo^{\min}\right)}{\sqrt{N}}
\right).$ Note that for sequentially interactive $\Al$, ~$\delo^{\min} = \delo.$ 
\end{theorem}
Next, we will
observe that 
the expected (squared) gradient norm of the output of  $\Al_s$ is the same as the expected (squared) gradient norm of the output of $\Al$ for \textit{symmetric} FL algorithms. 
The precise definition of a ``symmetric'' (fully interactive) ISRL-DP algorithm is that the aggregation functions $g_r$ (used to aggregate silo updates/messages and update the global model) are symmetric (i.e. $g_r(Z_1, \cdots, Z_N) = g_r(Z_{\pi(1)}, \cdots Z_{\pi(N)})$ for all permutations $\pi$) and in each round $r$ the randomizers $\rand_r = \mathcal{R}_r$ are the same for all silos $i \in [N]$. ($\rand_r$ can still change with $r$ though.) For example, 
all of the algorithms presented in this paper (and essentially all algorithms that we've come across in the literature, for that matter) are symmetric. This is because the aggregation functions used in each round are simple averages of the $M_r$ noisy gradients received from all silos and the randomizers used by every silo in round $r$ are identical: each adds the same Gaussian noise to the stochastic gradients. Note that for any symmetric algorithm, the distributions of the updates of $\Al$ and $\Al_s$ are both averages over all permutations of $[N]$ of the conditional (on $\pi$) distributions of the randomizers applied to the $\pi$-permuted database. 

Now, for a given $(\epso, \delo)$-ISRL-DP algorithm $\Al$, denote the shuffled algorithm derived from $\Al$ by $\Al_s$. Then apply~\cref{thm: CDP lower bound} to $\Al_s$ to obtain lower bounds on its expected squared gradient norm:

\begin{theorem}[\citep{arora2022faster}]
\label{thm: CDP lower bound}
Let 
~$\epsilon \in (0, \sqrt{N}], 2^{-\Omega(nN)} \leq \delta \leq 1/(nN)^{1 + \Omega(1)}$. Then, there exists an $L$-Lispchitz, $\beta$-smooth
smooth, convex loss $f: \mathbb{R}^d \times \XX \to \mathbb{R}$ and 
a database $\bx \in \XX^{n \times N}$ such that any $(\epsilon, \delta)$-CDP algorithm $\Al$ run on $\bx$ with output $\wpr$ satisfies \[
\expec\|\nabla \hf_{\bx}(\wpr)\|^2 = \Omega\left(L^2\min\left\{1, \frac{d \ln(1/\delta)}{\epsilon^2 n^2 N^2}\right\}\right).
\]
\end{theorem}

Applying~\cref{thm: CDP lower bound} with $\epsilon =\epso/\sqrt{N}$ yields the desired lower bound for $\Al_s$. Further, by the observations above about symmetric algorithms, this lower bound also apply to $\Al$.
\end{proof}

\section{Upper and Lower Bounds for Cross-Device FL Without a Trusted Server}
\label{app: hybrid lower bounds}
In this section, we use our results to derive upper and lower bounds for FL algorithms that satisfy both ISRL-DP and user-level DP. Algorithms that satisfy both both ISRL-DP and user-level DP provide \textit{privacy for the full data of each individual silo/user, even in the presence of an adversary that has access to the server, other silos/users, or silo/user communications}. Such a guarantee would be desirable in practical cross-device FL settings in which silos/users (e.g cell phone users) do not trust the server or other users with their sensitive data (e.g. text messages).

Assume $M=N$ for simplicity.\footnote{The extension to $M < N$ will be clear.} Given ISRL-DP parameters $(\epsilon, \delta)$ with $\epsilon \leq 1$, let $\epsilon_0 = \epsilon/n$  and $\delta_0 = \delta/4n \leq \delta/(n e^{(n-1) \epso}) = \delta/(n e^{(n-1)\epsilon/n})$. Consider the Proximal PL case for now. Run Noisy $(\epso, \delo)$-ISRL-DP Prox-SGD, which also satisfies $(\epsilon, \delta)$-user level DP by~\cref{app: dp relationships}. Thus, \cref{thm: hetero pl fl proxgrad} yields an ISRL-DP/user-level DP excess risk upper bound for heterogeneous FL with Proximal-PL losses: \begin{align}
\label{hybrid upper}
\EPL &= \widetilde{\mathcal{O}}\left(\frac{L^2}{\mu}\left(\frac{\kappa^2 \sqrt{d \ln(1/\delo)}}{\epso^2 n^2 \sqrt{N}} + \frac{\kappa}{\sqrt{Nn}}\right) \right) \nonumber \\
&= \widetilde{\mathcal{O}}\left(\frac{L^2}{\mu}\left(\frac{\kappa^2 \sqrt{d \ln(1/\delo)}}{\epsilon^2 \sqrt{N}} + \frac{\kappa}{\sqrt{Nn}}\right) \right) \\
&= \widetilde{\mathcal{O}}\left(\frac{L^2}{\mu}\left(\frac{\kappa^2 \sqrt{d \ln(1/\delo)}}{\epsilon^2 \sqrt{N}}\right) \right) \nonumber.
\end{align}
Regarding lower bounds: note that the semantics of the hybrid ISRL-DP/user-level DP notion are essentially identical to local DP, except that individual ``records/items'' are now thought of as datasets of size $n$. Thus, letting $n=1$ in the strongly convex ISRL-DP lower bound of~\citep{lr21fl} (where we think of each silo as having just one ``record'' even though that record is really a dataset) yields a lower bound that matches the upper bound attained above up to a factor of $\widetilde{\mathcal{O}}(\kappa^2)$. 
Note that the minimax risk bounds for ISRL-DP/user-level DP hybrid algorithms resemble the bounds for LDP algorithms~\citep{duchi13}, scaling with $N$, but not with $n$. A similar procedure can be used to derive upper and lower bounds for Proximal PL ERM and non-convex/non-smooth ERM, using our upper bounds in~\cref{thm: ISRL-DP prox PL SVRG ERM,thm: ISRL-DP proxspider} and lower bound in~\cref{thm: LDP lower bound}.

\clearpage
\section{Experimental Details and Additional Results}
\label{app: experiment details}
\subsection{ISRL-DP Fed-SPIDER: Alternate implementation of ISRL-DP FedProx-SPIDER}
\label{app:ISRL-DP_Fed-SPIDER}
We also evaluated an alternative implementation of ISRL-DP FedProx-SPIDER, given in \cref{alg: dp spider}. We found that this variation of ISRL-DP FedProx-SPIDER sometimes performed better in practice. 
For each $\epsilon \in \{0.75, 1, 1.5, 3, 6, 12, 18\}$, we chose the algorithm with smaller training loss and reported the test error for the corresponding algorithm as SPIDER in the plots. 
\begin{algorithm}[H]
\caption{ISRL-DP Fed-SPIDER: Alternate Implementation}
\label{alg: dp spider}
\begin{algorithmic}[1]
\STATE {\bfseries Input:} 
Number of silos $N \in \mathbb{N},$ dimension $d \in \mathbb{N}$ of data, noise parameters $\sigma^2_1$ and $\sigma^2_2$, data sets $X_i \in \XX^{n_i}$ for $i \in [N]$, loss function $f(w, x),$ number of rounds $E - 1 \in \mathbb{N}$, local batch size parameters $K_1$ and $K_2$, step size $\eta$.
 \STATE Server initializes $w_0^2 := 0$ and broadcasts. 
 \STATE Silos sync $w_0^{i,2} := w_0^2$ ($i \in [N]$).
\STATE Network determines random subset $S_0$ of $M_0 \in [N]$ available silos. 
 \FOR{$i \in S_0$ \textbf{in parallel}} 
 \STATE Silo $i$ draws $K_2$ random samples $\{x^{0,2}_{i, j}\}_{j \in [K_2]}$ (with replacement) from $X_i$ and noise $u_2^{(i)} \sim N(0, \sigma_2^2 \mathbf{I}_d).$ 
 \STATE Silo $i$ computes noisy stochastic gradient $\widetilde{v}^{i, 2}_0 := \frac{1}{K_2} \sum_{j =1}^{K_2} \nabla f(w_0^2, x_{i,j}^{0,2}) + u_2^{(i)}$ and sends to server.
 \ENDFOR
 \STATE Server aggregates $\widetilde{v}^{2}_0 := \frac{1}{M_0} \sum_{i \in S_0} \widetilde{v}^{i, 2}_0$ and broadcasts. 
 \FOR{$r \in \{0, 1, \cdots, E-2\}$}
 \STATE Network determines random subset $S_{r+1}$ of $M_{r+1} \in [N]$ available silos. 
 \FOR{$i \in S_{r+1}$ \textbf{in parallel}} 
 \STATE Server updates $w_{r+1}^0 := w_r^2$, $w_{r+1}^{1} := w_r^2 - \eta \widetilde{v}_r^2$ and broadcasts to silos.
 \STATE Silos sync $w^{i, 0}_{r+1} := w_{r+1}^0$, $\widetilde{v}^{i,0}_{r+1} := 
 \widetilde{v}^2_{r}$, and $w_{r+1}^{i,1} := w_{r+1}^{1}$ ($i \in [N])$. 
 \STATE Silo $i$ draws $K_1$ random samples $\{x^{r+1, 1}_{i, j}\}_{j \in [K_1]}$ (with replacement) from $X_i$ and noise $u_1^{(i)} \sim N(0, \sigma_1^2 \mathbf{I}_d).$
 \STATE Silo $i$ computes $\widetilde{v}^{i, 1}_{r+1} := \frac{1}{K_1} \sum_{j =1}^{K_1} [\nabla f(w_{r+1}^{1}, x^{r+1, 1}_{i, j}) - \nabla f(w_{r+1}^{0}, x^{r+1, 1}_{i, j})] + \widetilde{v}^{i, 0}_{r+1} + u_1^{(i)}$ and sends to server.
 \STATE Server aggregates 
 $\widetilde{v}_{r+1}^1 := \frac{1}{M_{r+1}} \sum_{i \in S_{r+1}} \widetilde{v}_{r+1}^{i, 1}$, updates $w_{r+1}^{2} := w_{r+1}^{1} - \eta \widetilde{v}_{r+1}^1$, and broadcasts.
 \STATE Silos sync $w_{r+1}^{i, 2} := w_{r+1}^{2}$.
 \STATE Silo $i$ draws $K_2$ random samples $\{x_{i, j}^{r+1,2}\}_{j \in [K_2]}$ (with replacement) from $X_i$ and noise $u_2^{(i)} \sim N(0, \sigma_2^2 \mathbf{I}_d).$ 
 \STATE Silo $i$ computes $\widetilde{v}^{i,2}_{r+1} := \frac{1}{K_2} \sum_{j=1}^{K_2} \nabla f(w_{r+1}^{2}, x_{i, j}^{r+1,2}) + u_2^{(i)}$ and sends to server. 
 \STATE Server updates $\widetilde{v}_{r+1}^{2} := \frac{1}{M_{r+1}} \sum_{i \in S_{r+1}} \widetilde{v}^{i,2}_{r+1}$ and broadcasts.  
 \ENDFOR
 \ENDFOR \\
\STATE {\bfseries Output:} $\wpr \sim \text{Unif}(\{ w_r^t \}_{r=1, \cdots, E-1; t=1,2})$.
\end{algorithmic}
\end{algorithm}

\subsection{MNIST experiment}
\label{app: MNIST data}
The MNIST data is available at \url{http://yann.lecun.com/exdb/mnist/}. In our implementation, we use \texttt{torchvision.datasets.MNIST} to download the MNIST data. All experiments are conducted on a device with 6-core Intel Core i7-8700.

\noindent{\textbf{Experimental setup}}: To divide the data into $N=25$ silos and pre-process it, we rely on the code provided by~\citep{woodworth2020}. The code is shared under a Creative Commons Attribution-Share Alike 3.0 license. We fix $\delta=1/n^2$ (where $n=$number of training samples per silo, is given in ``\textbf{Preprocessing}'') and test $\epsilon \in \{0.75, 1, 1.5, 3, 6, 12, 18\}$.

\noindent{\textbf{Preprocessing}}: 
First, we standardize the numerical data to have mean zero and unit variance, and flatten them. Then, we utilize PCA to reduce the dimension of flattened images from $d=784$ to $d=50$. To expedite training, we used 1/7 of the $5,421$ samples per digit, which is 774 samples per digit. As each silo is assigned data of two digits, each silo has $n=1,543$ samples. We employ an 80/20 train/test split for data of each silo.

\noindent \textbf{Gradient clipping: }
Since the Lipschitz parameter of the loss is unknown for this problem, we incorporated gradient clipping \citep{abadi16} into the algorithms. Noise was calibrated to the clip threshold $L$ to guarantee ISRL-DP (see below for more details). We also allowed the non-private algorithms to employ clipping if it was beneficial. 

\noindent{\textbf{Hyperparameter tuning}}: For each algorithm, each $\epsilon \in \{0.75, 1, 1.5, 3, 6, 12, 18\}$, and each $(M, R) \in \{(12, 25), (12, 50),$ $(25, 25), (25, 50)\}$, we swept through a range of constant stepsizes and clipping thresholds to find the (approximately) optimal stepsize and clipping threshold for each algorithm and setting. The stepsize grid consists of 5 evenly spaced points between $e^{-9}$ and 1. The clipping threshold includes 5 values of 1, 5, 10, 100, 10000. For ISRL-DP FedProx-SPIDER, we use $q \in \{1, 2, 3, 4\}$ for $R=50$ and $q \in \{1, 2\}$ for $R=25$.  Due to memory limitation, we did not check large $q$ values because it results in large batch size based on $K$ in ISRL-DP FedProx-SPIDER (see below for more details).

\noindent\textbf{Choice of $\sigma^2$ and $K$:} 
We used noise with smaller constants/log terms (compared to the theoretical portion of the paper) to get better utility (at the expense of needing larger $K$ to ensure privacy), by appealing to the moments accountant~\citep[Theorem 1]{abadi16} instead of the advanced composition theorem~\citep[Theorem 3.20]{dwork2014}. 

For ISRL-DP FedProx-SPIDER, we used $\sigma_1^2 = \frac{16L^2 \ln(1/\delta)}{\epsilon^2 n^2}\max\left(\frac{R}{q}, 1 \right)$, $\sigma_2 = \infty$, and $\hat{\sigma}_2^2 = \frac{64 L^2 R \ln(1/\delta)}{\epsilon^2 n^2}$. We chose $\sigma_2 = \infty$ because we do not have an a priori bound on the smoothness parameter $\beta$. Therefore, only the variance-reduction benefits of SPIDER are illustrated in the experiments and not the smaller privacy noise.

For ISRL-DP FedSPIDER: Alternate Implementation, we used $\sigma_1^2 = \frac{32L^2 \ln(2/\delta) R}{n^2 \epsilon^2}$ and $\sigma_2^2 = \frac{8L^2 \ln(2/\delta) R}{n^2 \epsilon^2}$ with $K_1 = K_2 = \frac{n \sqrt{\epsilon}}{2 \sqrt{R}}$ given above, which guarantees ISRL-DP by~\citep[Theorem 1]{abadi16}. Note that the larger constant $32$ is needed for ISRL-DP in $\sigma_1^2$ because the $\ell_2$ sensitivity of the updates in line 16 of~\cref{alg: LDP proxspider} is larger than simple SGD updates (which are used in MB-SGD, Local SGD, and line 20 of~\cref{alg: LDP proxspider}) by a factor of $2$.

For ISRL-DP MB-SGD and ISRL-DP Local SGD, we use the same implementation as~\citep{lr21fl}.

\noindent{\textbf{Generating Noise}}: Due to the low speed of NumPy package in generating multivariate random normal vectors, we use an alternative approach to generate noises. For ISRL-DP SPIDER and ISRL-DP MB-SGD algorithms, we generate the noises on MATLAB and save them. Then, we load them into Python when we run the algorithms. 
Since the number of required noise vectors for ISRL-DP Local SGD is much larger ($K$ times larger) than two other ISRL-DP algorithms, saving the noises beforehand requires a lot of memory. Hence, we generate the noises of ISRL-DP Local SGD
on Python by importing a MATLAB engine. 

\noindent \textbf{Plots and additional experimental results:}
See Figure~\ref{fig:M12R25} and Figure~\ref{fig:M25R50} for results of the two remaining experiments: $(M = 12, R = 25)$ and $(M=25, R = 50)$. 
The results are qualitatively similar to those presented in the main body. In particular, ISRL-DP SPIDER continues to outperform both ISRL-DP baselines in most tested privacy levels. Also, ISRL-DP MB-SGD continues to show strong performance in the high privacy regime ($\epsilon \leq 1.5$).  

\begin{figure}[ht] 
  \centering
  \begin{tabular}{@{}c@{}}
    \includegraphics[width=0.5\linewidth]{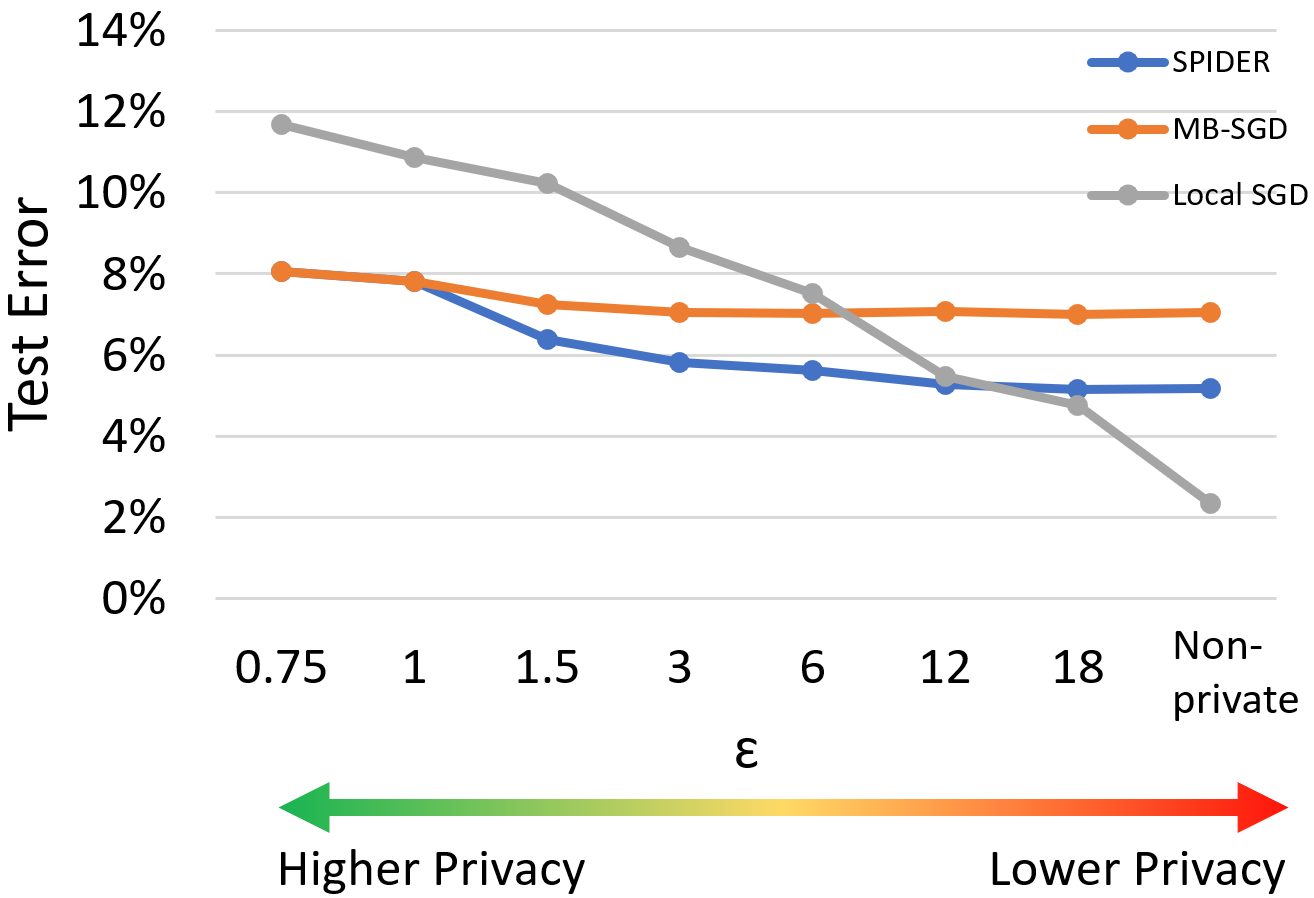} \\[\abovecaptionskip]
  \end{tabular}
  \caption{MNIST. $M = 12, R = 25$.}
  \label{fig:M12R25}
\end{figure}

\begin{figure}[ht] 
  \centering
  \begin{tabular}{@{}c@{}}
    \includegraphics[width=0.5\linewidth]{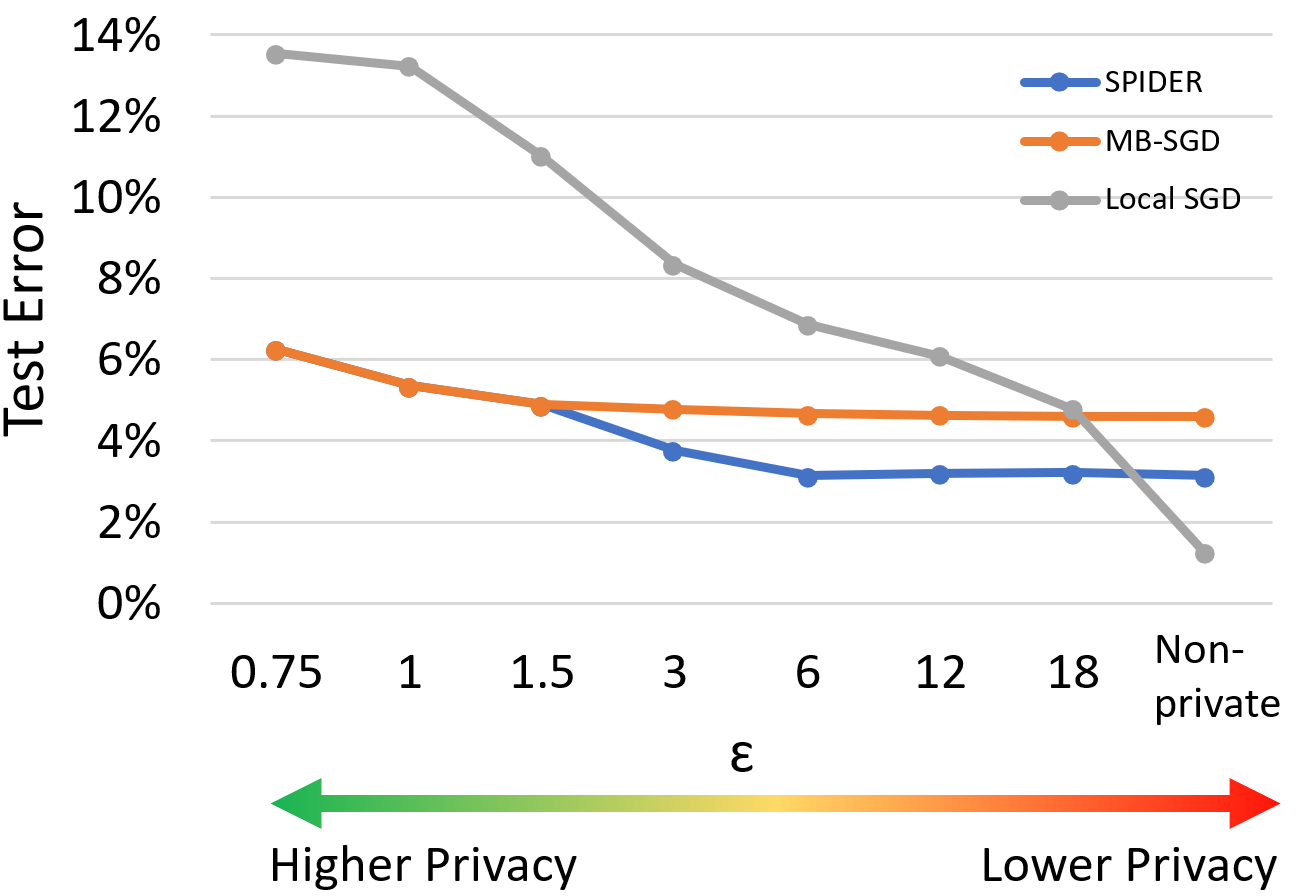} \\[\abovecaptionskip]
  \end{tabular}
  \caption{MNIST. $M = 25, R = 50$.}
  \label{fig:M25R50}
\end{figure}

\subsection{CIFAR10 experiment}
\label{app: CIFAR10 data}
We run an experiment on CIFAR10 data to further evaluate the performance of ISRL-DP SPIDER in image classification. We partition the data set into 10 heterogeneous silos, each containing one class out of 10 classes of data. We use a 5-layer CNN with two 5x5 convolutional layers (the first with 6 channels, the second with 16 channels, each followed by a ReLu activation and a 2x2 max pooling) and three fully connected layers with 120, 84, 10 neurons in each fully connected layer (the first and second fully connected layers followed by a ReLu activation). For 7 privacy levels ranging from $\epsilon = 0.75$ to $\epsilon = 18$, we compare ISRL-DP SPIDER against standard FL baselines: MB-SGD, Local SGD (a.k.a. Federated Averaging)~\citep{mcmahan2017originalFL}, ISRL-DP MB-SGD~\citep{lr21fl}, and ISRL-DP Local SGD. We fix $\delta = 1/n^2$. As Figure~\ref{fig:CIFAR10M10R50} shows, \textit{ISRL-DP SPIDER outperforms both ISRL-DP baselines for most tested privacy levels}. The results are based on the average error of 10 random assignment of train/test split of data for each algorithm/epsilon pair. 
CIFAR10 data is available at~\url{https://www.cs.toronto.edu/~kriz/cifar.html}. In our implementation, we directly download the data from \texttt{torchvision.datasets.CIFAR10}. 

\noindent{\textbf{Experimental setup}}: 
To divide the CIFAR10 data into $N=10$ heterogeneous silos, we use labels. That is, we assign one unique image class to each of 10 heterogeneous silos. 

\noindent{\textbf{Preprocessing}}: We standardize the numerical data to have mean zero and unit variance. We utilize a 80/20 train/test split for data of each client.

\noindent \textbf{Gradient clipping: }
Since the Lipschitz parameter of the loss is unknown for this problem, we incorporated gradient clipping \citep{abadi16} into the algorithms. Noise was calibrated to the clip threshold $L$ to guarantee ISRL-DP (see below for more details). We also allowed the non-private algorithms to employ clipping if it was beneficial. 

\noindent{\textbf{Hyperparameter tuning}}: It is similar to hyperparameter tuning of MNIST data. However, we check $(M, R) = (10, 50)$ and $(M, R) = (10, 100)$ here. Also, the stepsize grid of Local SGD consists of 20 evenly spaced points between $e^{-5}$ and $e^{1}$ for local SGD. The stepsize grid of MB-SGD and SPIDER with R=50 consists of 12 evenly spaced points between $e^{-5}$ and $e^{0}$ and with R=100 consists of 8 evenly spaced points between $e^{-5}$ and $e^{1}$. The clipping threshold of all algorithms includes 6 values of 0.001, 0.01, 0.1, 1, 5, 10. For ISRL-DP FedProx-SPIDER, we use $q \in \{1, 2, 3, 4\}$ for $R=50$ and $q \in \{1, 2, \dots, 8\}$ for $R=100$. Due to memory limitation, we did not check large $q$ values because it results in large batch size based on $K$ in ISRL-DP FedProx-SPIDER (see below for more details).

\noindent\textbf{Choice of $\sigma^2$ and $K$:} 
Same as in MNIST: see~\cref{app: MNIST data}. 

\subsection{Breast cancer experiment}
\label{app: breast cancer}

We run an experiment on Wisconsin Breast Cancer (Diagnosis) data (WBCD) to further evaluate the performance of ISRL-DP SPIDER in binary (malignant vs. benign) classification. We partition the data set into 2 heterogeneous silos, one containing malignant labels and the other benign labels. We use a 2-layer perceptron with 5 neurons in the hidden layer. For 7 privacy levels ranging from $\epsilon = 0.75$ to $\epsilon = 18$, we compare ISRL-DP SPIDER against standard FL baselines: MB-SGD, Local SGD (a.k.a. Federated Averaging)~\citep{mcmahan2017originalFL}, ISRL-DP MB-SGD~\citep{lr21fl}, and ISRL-DP Local SGD. We fix $\delta = 1/n^2$. As Figure~\ref{fig:M2R25} shows, \textit{ISRL-DP SPIDER outperforms both ISRL-DP baselines for most tested privacy levels}. The results are based on the average error of 10 random assignment of train/test split of data for each algorithm/epsilon pair. 
WBCD data is available at~\url{https://archive.ics.uci.edu/ml/datasets} and we directly download the data from UCI repository website. The experiment is conducted on a device with 6-core Intel Core i7-8700.

\noindent{\textbf{Experimental setup}}: 
To divide the WBCD data into $N=2$ silos, we use labels (malignant vs. benign). We split the data into 2 parts, one only has malignant labels and the other only has benign data. Then, we assign each part to a client to have full heterogeneous silos. 
In all experiments, we fix $\delta=1/n^2$ (where $n=$number of training samples per client, is given in ``\textbf{Preprocessing}'') and test $\epsilon \in \{0.75, 1, 1.5, 3, 6, 12, 18\}$.

\noindent{\textbf{Preprocessing}}: We standardize the numerical data to have mean zero and unit variance. We utilize a 80/20 train/test split for data of each client.

\noindent \textbf{Gradient clipping: }
Since the Lipschitz parameter of the loss is unknown for this problem, we incorporated gradient clipping \citep{abadi16} into the algorithms. Noise was calibrated to the clip threshold $L$ to guarantee ISRL-DP (see below for more details). We also allowed the non-private algorithms to employ clipping if it was beneficial. 

\noindent{\textbf{Hyperparameter tuning}}: It is similar to hyperparameter tuning of MNIST data. However, we check $(M, R) = (4, 25)$ here. Also, the stepsize grid consists of 15 evenly spaced points between $e^{-9}$ and 1. The clipping threshold includes 4 values of 0.1, 1, 5, 10. For ISRL-DP FedProx-SPIDER, we use $q \in \{1, 2, \dots, 10\}$. Due to memory limitation, we did not check large $q$ values because it results in large batch size based on $K$ in ISRL-DP FedProx-SPIDER (see below for more details).

\noindent\textbf{Choice of $\sigma^2$ and $K$:} Same as in MNIST: see~\cref{app: MNIST data}. 

\noindent{\textbf{Generating Noise}}: Due to the low speed of NumPy package in generating multivariate random normal vectors, we use an alternative approach to generate noises. For ISRL-DP SPIDER and ISRL-DP MB-SGD algorithms, we generate the noises on MATLAB and use them in Python when we run the algorithms. Since the number of required noise vectors for ISRL-DP Local SGD is much larger ($K$ times larger) than two other ISRL-DP algorithms, saving the noises beforehand requires a lot of memory. Hence, we generate the noises of ISRL-DP Local SGD on Python by importing a MATLAB engine.

\section{Limitations}
\label{app: limitations}
A major focus of this work is on developing DP algorithms that can handle a broader, more practical range of ML/optimization problems: e.g. non-convex/non-smooth, Proximal-PL, heterogeneous silos. However, some assumptions may still be strict for certain practical applications. In particular, the requirement of an a priori bound on the Lipschitz parameter of the loss--which the vast majority of works on DP ERM and SO also rely on--may be unrealistic in cases where the underlying data distribution is unbounded and heavy-tailed. Understanding what privacy and utility guarantees are possible without this assumption is an interesting problem for future work. 

\textit{Limitations of Experiments:} Pre-processing and hyperparameter tuning were done non-privately, since the focus of this work is on DP FL.\footnote{See \citep{abadi16, liu2019private, steinkehyper} and the references therein for discussion of DP PCA and DP hyperparameter tuning.} This means that the total privacy loss of the entire experimental process is higher than the $\epsilon$ indicated, which only accounts for the privacy loss from executing the FL algorithms with given (fixed) hyperparameters and (pre-processed) data.

\section{Societal Impacts}
\label{app: impacts}
We expect a net positive impact on society from our work, given that our algorithms can prevent sensitive data leakage during FL. Nonetheless, like all technologies, it carries potential for misuse and unintended outcomes. For instance, companies may attempt to legitimize invasive data collection by arguing that the user data will solely be utilized to train a differentially private model to safeguard privacy. Furthermore, in some parameter ranges, privacy comes at the expense of lower model accuracy, which could have adverse effects in crucial applications such as medicine and environmental science.

\end{document}